\documentclass{IEEEtran}
\usepackage{cite}
\usepackage{amsmath,amssymb,amsfonts}
\usepackage{graphicx}
\usepackage{textcomp}
\def\BibTeX{{\rm B\kern-.05em{\sc i\kern-.025em b}\kern-.08em
    T\kern-.1667em\lower.7ex\hbox{E}\kern-.125emX}}

\PassOptionsToPackage{numbers, compress}{natbib}
\input{mysymbol.sty}
\input{mypackage.sty}
\input{myacronym.sty}
\usepackage{algorithmicx,algorithm}
\usepackage{algpseudocode}
\newcommand{\isarxiv}{}          % arXiv version

\ifdefined\isarxiv
  \newcommand{\arxivonly}[1]{#1}
  \newcommand{\journalonly}[1]{}
  \newcommand{\arxivjournal}[2]{#1}
\else
  \newcommand{\arxivonly}[1]{}
  \newcommand{\journalonly}[1]{#1}
  \newcommand{\arxivjournal}[2]{#2}
\fi

\newcommand{\AnalogSGD}{\texttt{Analog-\allowbreak SGD-\allowbreak WOP}}

\newcommand{\AnalogSGDAP}{\texttt{Analog-\allowbreak SGD-\allowbreak AP}}
\newcommand{\AnalogSGDSP}{\texttt{Analog-\allowbreak SGD-\allowbreak SP}}

\newtheorem{assumption}{\hspace{0pt}\bf Assumption}
\newtheorem{lemma}{\hspace{0pt}\bf Lemma}
\newtheorem{proposition}{\hspace{0pt}\bf Proposition}

\newtheorem{corollary}{\hspace{0pt}\bf Corollary}

\newtheorem{remark}{\hspace{0pt}\bf Remark}

\begin{document}

\title{On the Convergence Theory of Pipeline Gradient-based Analog In-memory Training}
\author{
Zhaoxian Wu, Quan Xiao, Tayfun Gokmen, Hsinyu Tsai, Kaoutar El Maghraoui, Tianyi Chen 
\thanks{
Z. Wu and Q. Xiao are with Cornell University, New York, NY 10044 (\{zw868, qx232\}@cornell.edu). 
T. Gokmen, H. Tsai, and K. El Maghraoui are with IBM Research, Yorktown Heights, NY 10598 (\{tgokmen, htsai, kelmaghr\}@us.ibm.com).
T. Chen is with Rensselaer Polytechnic Institute and Cornell University (tianyi.chen@cornell.edu).}  
 \thanks{
The work was done when the authors were at Rensselaer Polytechnic Institute. The work was supported by IBM through the IBM-Rensselaer Future of Computing Research Collaboration, the National Science Foundation Projects 2401297, 2532349, and 2532653, and by the Cisco Research Award. } \\
}
\maketitle
\vspace{-5em}
\begin{abstract}
    Aiming to accelerate the training of large deep neural networks (DNN) in an energy-efficient way, analog in-memory computing (AIMC) emerges as a solution with immense potential. 
    AIMC accelerator keeps model weights in memory without moving them from memory to processors during training, reducing overhead dramatically.
    Despite its efficiency, scaling up AIMC systems presents significant challenges.
    Since weight copying is expensive and inaccurate, data parallelism is less efficient on AIMC accelerators. 
    It necessitates the exploration of pipeline parallelism, particularly asynchronous pipeline parallelism, which utilizes all available accelerators during the training process. 
    This paper examines the convergence theory of stochastic gradient descent on AIMC hardware with an asynchronous pipeline (\AnalogSGDAP).
    Although there is empirical exploration of AIMC accelerators, the theoretical understanding of how analog hardware imperfections in weight updates affect the training of multi-layer DNN models remains underexplored. Furthermore, the asynchronous pipeline parallelism results in stale weights issues, which render the update signals no longer valid gradients.
    To close the gap, this paper investigates the convergence properties of {\AnalogSGDAP} on multi-layer DNN training. We show that the {\AnalogSGDAP} converges with iteration complexity $\ccalO\lp \varepsilon^{-2}+\varepsilon^{-1}\rp$ despite the aforementioned issues, which matches the complexities of digital SGD and Analog SGD with synchronous pipeline, except the non-dominant term $\ccalO(\varepsilon^{-1})$. It implies that AIMC training benefits from asynchronous pipelining almost for free compared with the synchronous pipeline by overlapping computation.
\end{abstract}


\section{Introduction}
\begin{figure}
	\centering
	\includegraphics[width=0.55\linewidth]{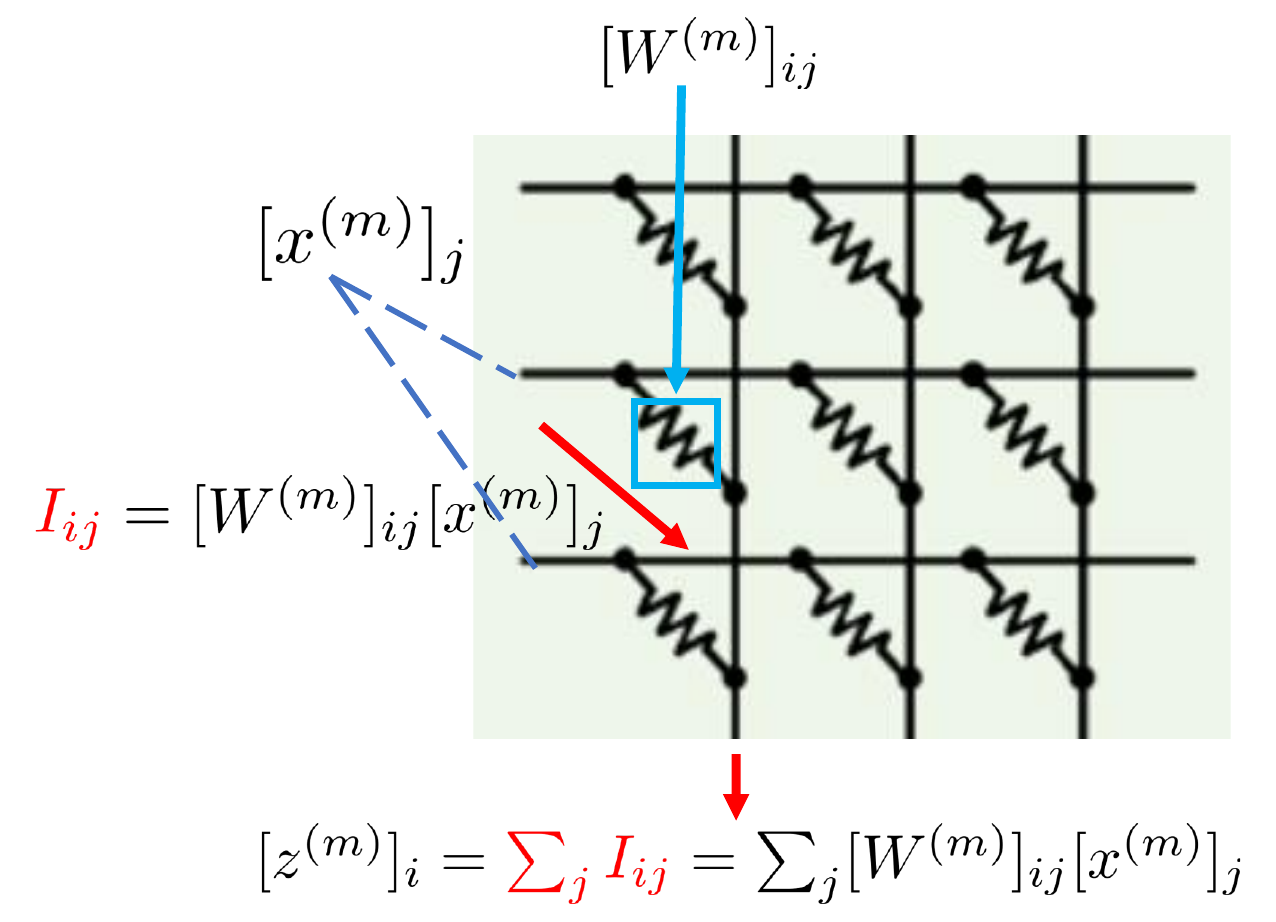}
	\vspace{0.5em}
	\caption{
		Illustration of MVM computation in AIMC accelerators. The \textbf{weight} $W^\spp{m}$ on layer $m$ is stored in a crossbar tile.
		The $(i,j)$-th element of $W^\spp{m}$ is represented by the \blue{\textbf{conductance}} of the $(i,j)$-th resistive element.
		MVM operation $z^\spp{m} = W^\spp{m}x^\spp{m}$ is performed by applying \bulletcolor{\textbf{voltage}} $[x^\spp{m}]_j$ between $j$-th and $(j+1)$-th row. By Ohm's law, the \red{\textbf{current}} is $I_{ij}=[W^\spp{m}]_{ij}[x^\spp{m}]_{j}$; and by Kirchhoff's law, the total current on the $i$-th column is $\sum_{j}I_{ij}=\sum_{j}[W^\spp{m}]_{ij}[x^\spp{m}]_{j}$. Unlike the digital counterpart, no movement of $W^\spp{m}$ is required for MVM calculation.
	}
	\label{fig:MVM}
	\vspace{-2em}
\end{figure}
Deep Neural Networks (DNNs) have facilitated remarkable progress in a wide range of applications. 
As the model sizes and dataset volumes grow exponentially, DNN training becomes time-consuming and resource-intensive. For example, training an LLAMA2 7B model requires 184 thousand GPU hours, and this increases to 1.7 million GPU hours for its 70B version \citep{touvron2023llama2}.
In response to the escalating computational demands of large-scale DNN training and inference, the pursuit of specialized hardware accelerators has emerged. 
The ubiquitous fully connected and convolutional layers, whose core operation is matrix-vector multiplication (MVM), form the fundamental backbone of both training and inference as the demand for efficient processing intensifies.
Typically, each forward and backward pass requires millions of multiply-accumulate 
(MAC) operations, necessitating specialized hardware with high performance and low energy consumption. 

\noindent
\textbf{The promise of analog computing.}
To this end, a promising solution to accelerate MVM operations is the \emph{analog in-memory computing} (AIMC) accelerator, wherein both the input and output of the MVM operation are analog signals like current or voltage. The DNN weights are stored in analog crossbar tiles consisting of resistive elements \citep{jain2019neural} where the conductance of the resistive element represents the weight magnitude. 
Unlike traditional Von Neumann accelerators such as GPUs and TPUs, which move weights from memory to the processor during computation, AIMC accelerators keep the weights stationary and perform MVM operations in memory. Without moving weights frequently, AIMC accelerators enable time- and energy-efficient MVM computations; see Fig. \ref{fig:MVM}.

\noindent
\textbf{The curse of data parallelism.}
While efficient computation is beneficial, in-memory computing imposes scalability limitations. Specifically, contemporary DNN optimizers depend on mini-batch SGD to facilitate highly parallel training. 
In the Von Neumann computation architecture, the most efficient parallel mechanism is \emph{data parallelism} \citep{li2020pytorch}, where DNN weights are replicated to different accelerators to support gradient descent with large batches \citep{goyal2017accurate,you2020Large}. 
By contrast, in the AIMC architecture, trainable weights are stored in the analog tiles as the conductances of resistive elements; thus, copying weights is expensive and prone to error. Therefore, data parallelism, the enabling factor of large-scale AI training, is unrealistic for AIMC accelerators.

Due to the physical limitations on the AIMC accelerator, we instead resort to \emph{pipeline parallelism} \citep{xu2021parallelizing}. 
Pipeline parallelism partitions a large model into a series of stages distributed across multiple accelerators, with each stage corresponding to the computation of a few layers. After an accelerator processes one datum (or batch), it passes the intermediate outputs to the next accelerator and begins processing the next datum. 
Pipeline parallelism overlaps computation across different accelerators, increasing system throughput.

\noindent
\textbf{Synchronous and asynchronous parallelism.} 
As the requirement of training large models emerges, a plethora of pipeline parallel libraries have been developed, including Gpipe \cite{huang2019gpipe,kim2020torchgpipe}, Deepspeed \citep{ren2021zero}, FSDP \citep{zhao2023pytorch}, Megetron-LM \citep{shoeybi2019megatron}, Colossal-AI \cite{li2023colossal}, just to name a few. 
In their implementation, a gradient \textit{mini-batch} (e.g., with batch size 128) is split into \textit{micro-batches} (e.g., 8 micro-batches with batch size 16) and filled into the pipeline. The model weights are updated after all the micro-batches are processed. Featuring the synchronization operation at the end of each mini-batch, this pipeline paradigm is referred to as \textit{synchronous pipeline}.
However, during the fill and drain phases, the pipeline is not fully occupied, which limits the device utilization. To overcome the issue, \textit{asynchronous pipeline} is proposed, which continuously fills the micro-batch into the pipeline and updates the weights once the gradient is available. Asynchronous pipeline increases the device utilization, while the weights for gradient computing can be stale, leading to a biased training dynamic. 

To fully realize the promise of analog training, this paper focuses on the theoretical understanding of the viability of stochastic gradient descent on AIMC hardware with asynchronous pipeline (\AnalogSGDAP). Besides the stale signal issue in the asynchronous pipeline, the update dynamics on AIMC accelerators also suffer from bias since the weight update involves adjusting conductance, which is difficult to control in theory. Both stale signals and imperfect update dynamics lead to the following natural question: 
\begin{center}
	{\em How do their combined effects provably impact training on the AIMC accelerator?  }
\end{center}
This paper theoretically demonstrates that {\AnalogSGDAP} converges at almost the same rate as SGD on a single AIMC accelerator ({\AnalogSGD}) and AIMC hardware with synchronous pipeline ({\AnalogSGDSP}), providing a theoretical foundation for parallel AIMC training.

\subsection{Main results}
This paper aims to understand the impact of bias and investigate the computational complexities and acceleration effect of {\AnalogSGDAP}.
We highlight the unique technical challenges compared with the analysis of the asynchronous pipeline on digital accelerators as follows. 
    \textbf{T1)}
    Unlike the digital training with precise weight update, the update dynamics on AIMC accelerators include an asymmetric bias. To be specific, consider an analog weight $W_k$ on a resistive crossbar array and a desired update $\Delta W_k$. The updated weight becomes
    \begin{tcolorbox}[emphblock, width=1\linewidth]
        \vspace{-1.em}
        \begin{align}
            \label{dynamic:analog-update}
            W_{k+1} = W_k + \Delta W_k
            - \frac{1}{\tau}|\Delta W_k|\odot W_k
        \end{align}
        \vspace{-1.8em}
    \end{tcolorbox}
    where $\tau$ is a hardware-specific constant, $|\cdotc|$ is element-wise absolute value, and $\odot$ is the element-wise production. Unlike digital training, the last term in \eqref{dynamic:analog-update} is unique to analog accelerators. This bias can not be eliminated within a reasonable cost and leads to challenges to the analysis. In \cref{section:response-functions}, we provide more background.
    \textbf{T2)}
    Even though there are already theoretical works on the design of the asynchronous pipeline, most of them model it as a delayed gradient method \cite{sun2017asynchronous}. However, in the asynchronous pipeline, the forward pass is performed using stale weights, while the backward pass uses the latest weights. Consequently, their inconsistent multiplication used for weight updates is not a valid gradient, and this error accumulates during training. 
    The stale signals in DNN render the update increment used in {\AnalogSGDAP} no longer a valid gradient, making its convergence behavior unclear. 
    \textbf{T3)}
    Multi-layer DNN models accumulate error through layers, leading to challenges in the analysis. Without considering a specific multi-layer architecture, prior work \cite{huo2018decoupled} treats all weights as a whole, thereby failing to capture the multi-layer DNN structure.

\begin{table*}[!t]
\centering
\begin{tabular}{c c cc}
    \hline\hline
    Pipeline Type & Sample Complexity & Comp. Density & Wall-clock Complexity \\
    \hline
    \AnalogSGD
    &  $\ccalO\lp \sigma^2\varepsilon^{-2}\rp$        & $\frac{1}{M}$  & $\ccalO\lp M\sigma^2\varepsilon^{-2}\rp$ \\ 
    {\AnalogSGDSP} [Thm. \ref{theorem:ASGD-sync-convergence-noncvx-linear}] & 
    $\ccalO\lp \sigma^2\varepsilon^{-2}\rp$ 
    & $\frac{B}{M+B-1}$ & $\ccalO\lp (M+B)B^{-1}\sigma^2\varepsilon^{-2}\rp$  \\
    {\AnalogSGDAP} [Thm. \ref{theorem:ASGD-async-convergence-noncvx-linear}] & 
    $\ccalO\lp \sigma^2\varepsilon^{-2}+\varepsilon^{-1}\rp$ & $1$ & $\ccalO\lp \sigma^2\varepsilon^{-2}+\varepsilon^{-1}\rp$\\
    \hline\hline
\end{tabular}
\vspace{-0.5em}
\caption{
    Comparison of different pipeline strategies.
    Sample complexity is the number of samples needed to achieve a training loss $\le E+\varepsilon$, where $E$ is an inevitable asymptotic error due to hardware imperfections. 
    Computing density is the average portion of active accelerators. In the table, $M$ and $B$ are the number of stages and micro-batches, respectively.
    Wall-clock complexity measures the running time for achieving training loss $\le \varepsilon$, which in math is sample complexity divided by computation density. See Theorems \ref{theorem:ASGD-sync-convergence-noncvx-linear} and \ref{theorem:ASGD-async-convergence-noncvx-linear} for iteration complexity analysis; and Corollary \ref{corollary:sample-complexity} and \ref{corollary:clock-cycle-complexity} for sample and wall-clock complexity analysis.
}
\label{table:comparison}
\vspace{-1em}
\end{table*}

Overcoming the challenges, this paper answers the question above and provides both a convergence guarantee and computational complexities. 
The key contributions are as follows.
    \textbf{C1)}
    We consider the training of a multi-layer DNN, which consists of linear layers and non-linear activation functions. We formulate the training dynamics of {\AnalogSGDAP}, which reflects the stale weights in the asynchronous pipeline and the bias from the AIMC accelerators' imperfection. We investigate the smoothness of the training objective as well as the robustness of the network against stale weights.
    \textbf{C2)}
    We analyze three types of computation complexities of the asynchronous pipeline training in AIMC accelerators: iteration, sample, and wall-clock complexities. 
    The crucial step is to construct a Lyapunov function that has sufficient descents at each iteration.
    We show that {\AnalogSGDAP} has the sample complexity $\ccalO\lp \sigma^2\varepsilon^{-2}+\varepsilon^{-1}\rp$, which is the same as that of {\AnalogSGD} and {\AnalogSGDSP} except for the higher-order term $\ccalO(\varepsilon^{-1})$.
    Unlike the synchronous pipeline, the asynchronous pipeline utilizes all available accelerators and increases the computation density from $\frac{B}{M+B-1}$ to $1$, where $B$ is the number of micro-batches and $M$ is the number of stages. For example, $M=4$ and $B=4$ lead to about 1.7$\times$ computation density. 
    The results are summarized in Table \ref{table:comparison}.
    \textbf{C3)}
    We verify the speedup of {\AnalogSGDAP} against {\AnalogSGD} and {\AnalogSGDSP} in AIMC accelerator simulator, AIHWKIT. Numerical simulation results show that despite a slightly larger sample complexity, {\AnalogSGDAP} achieves the same accuracy with fewer clock cycles due to its high computation density, which is consistent with the theoretical guarantee.
    Furthermore, empirical results show that {\AnalogSGDAP} achieves linear speedup at least within the range of 1-8 accelerators.
\subsection{Prior art.}
Prior art can be grouped into two categories.

\noindent
\textbf{Gradient-based training on AIMC accelerators.}
Analog SGD implemented by pulse update is extremely efficient \cite{gokmen2016acceleration, yao2017face, wang2019situ}, and a series of analog-friendly algorithms has been proposed to enable the gradient-based on-chip training on the AIMC hardware. 
However, the asymmetric update and noisy gradient lead to large asymptotic errors in training \citep{burr2015experimental}. 
Numerous algorithms have been proposed proposed to address these issues.
TT-v1 \citep{gokmen2020} heuristically introduces an auxiliary analog array as a gradient approximation to eliminate the asymptotic error. To demonstrate its efficiency, \cite{wu2024towards} establishes a model to characterize the dynamics of the analog training and provides a theoretical convergence guarantee. Building on that, \cite{wu2025analog} investigates the impact of response functions on the training and demonstrates how its impact can be alleviated by bilevel optimization methods.
To filter out the high-frequency noise introduced by inter-device transfer, TT-v2 \citep{gokmen2021} introduces a digital array. Based on TT-v1/v2, Chopped-TTv2 (c-TTv2) and Analog Gradient Accumulation with Dynamic reference (AGAD) \citep{rasch2024fast} are proposed to mitigate the impact of inaccurate symmetric point correction \cite{onen2022neural}. 
To overcome the limit update granularity issue, \cite{li2025memory} proposed a multi-tile training paradigm.
However, the mentioned works focus on training single AIMC accelerators. Instead, this work aims to scale up the system by parallelizing model computation.

\noindent
\textbf{Asynchronous pipeline.}
Existing works proposed methods to alleviate the stale weight issues in the asynchronous pipeline.
For instance, \cite{kosson2021pipelined} proposed weight-prediction approaches to compensate for the delay, PipeDream \citep{narayanan2019pipedream} combined the synchronous and asynchronous pipeline parallelisms to achieve greater speedup; and \cite{chen2022sapipe} and \citep{li2018pipe} propose only overlapping gradient computation and communication to ensure the delay is one. 
These algorithms are {\em not for analog pipeline training}. 
Researchers also attempted to design analog pipeline circuits to accelerate the training \cite{shafiee2016isaac, song2017pipelayer}, but they focus on the synchronous pipeline.
On the theory side, there are few studies on pipeline training for digital accelerators such as \cite{huo2018decoupled,xu2020acceleration}, but their algorithms either require weight stashing, which is {\em expensive for analog accelerators}, or slow down digital SGD through higher-order terms. 
As \cite{narayanan2019pipedream} points out, delay affects both forward and backward passes in the asynchronous pipeline DNN training, rendering the increment no longer a valid gradient and making the analysis challenging.
\section{Analog Training with Model Parallelism}
In this section, we first introduce the multi-layer DNN training model, and then introduce the analog stochastic gradient descent algorithm with the asynchronous pipeline.

\subsection{Training problem and back-propagation}
Consider a standard model training problem on a data distribution $\ccalP$ given by
\begin{equation}
    \label{problem}
    \min_{W\in\reals^{D}}  \Big\{f(W):= \mbE_{\xi\sim\ccalP} [f(W; \xi) ]\Big\}
\end{equation}
where $W$ is the trainable model weights on analog crossbar tiles, which is vectorized into a $D$-dimension vector.

\noindent
\textbf{DNN model.}
Consider the multi-layer DNN, which consists only of linear layers and activation functions.
Let $[M] := \{1, 2, \cdots, M\}$ denote an index set. 
The trainable weights are separated into $M$ different blocks, i.e., $W = [W^\spp{1}, W^\spp{2}, \cdots, W^\spp{M}]$, where $W^\spp{m}, \forall m\in[M]$ is a matrix. 
The training of DNN typically consists of three phases: forward, backward, and update.
Suppose $x$ and $y$ are the feature and label of data $\xi=(x,y)\in\ccalX\times\ccalY$, where $\ccalX$ and $\ccalY$ are feature and label spaces, respectively. The forward pass of the model for layer $m\in[M]$ is expressed by
\begin{align}
    \label{recursion:forward}
     \textsf{Forward}\qquad 
    &x^\spp{1} = x,~~~~
    z^\spp{m} = W^\spp{m}x^\spp{m},
    \\
    &x^\spp{m+1} = g^\spp{m}(z^\spp{m})
    \nonumber
\end{align}
where $g^\spp{m}(\cdotc)$ is the activation function of layer $m$, and $x^\spp{m}$ is the input feature.
At the last layer, the ouput $x^\spp{M}$ and label $y$ are fed together into the loss function $\ell(\cdotc, \cdotc) : \ccalY\times\ccalY\to\reals$ and yields 
$f(W; \xi) = \ell(x^\spp{M+1}, y)$.
For the backward pass, define the backward signal $\delta^\spp{m} := \fracpartial{f(W; \xi)}{z^\spp{m}}$, wherein the error of the last layer is first computed by $\delta^\spp{M} = \nabla_{x^\spp{M+1}}\ell(x^\spp{M+1}, y) \nabla g^\spp{M}(z^\spp{M})$.
Subsequently, each layer $m\in[M-1]$ computes the next recursion by the chain rule 
\begin{align}
    \label{recursion:backward}
    \textsf{Backward}\qquad 
    \delta^\spp{m}
    =&\ \fracpartial{f(W; \xi)}{z^\spp{m+1}}\fracpartial{z^\spp{m+1}}{x^\spp{m+1}}\fracpartial{x^\spp{m+1}}{z^\spp{m}} 
    \\
    =&\ (\delta^\spp{m+1}W^\spp{m}) \nabla g^\spp{m}(z^\spp{m}).
    \nonumber
\end{align}
In the forward pass, the activation $z^\spp{m}$ is stashed until the backward signal $\delta^\spp{m+1}$ returns. 
Typically, the computation $\nabla g^\spp{m}(z)$ is light-weight, and hence the major overhead of the backward pass lies at the MVM among $W^\spp{m}$ and $\delta^\spp{m}$. 
With the forward and backward signals, the gradient with respect to the weight $W^\spp{m}$ of layer $m$ is computed by 
\begin{align}
    \label{definition:gradient}
    \textsf{Gradient}\qquad 
    \nabla_{W^\spp{m}} f(W; \xi)
    =&\ \fracpartial{f(W; \xi)}{z^\spp{m}}\fracpartial{z^\spp{m}}{W^\spp{m}} 
    \\
    =&\ \delta^\spp{m}\otimes x^\spp{m}
    \nonumber
\end{align}
where $\otimes$ is the outer product between two vectors.

\noindent
\textbf{Rank-update}
Unlike digital accelerators, AIMC accelerators use \textit{rank-update} to change the matrix, which leverages the outer product structure as shown in \eqref{definition:gradient}. To do so, independent random voltage pulses with probability proportional to the elements of $\delta^\spp{m}$ and $x^\spp{m}$ are applied to columns and rows, respectively. Consequently, the $(i, j)$-th element $[W]_{i,j}$ receives a pulse coincident with probability $\propto [\delta^\spp{m}]_i[x^\spp{m}]_j$.
Without computing the whole matrix $\delta^\spp{m}\otimes x^\spp{m}$ explicitly, rank-update has only $O(BN)$ complexity while the update on digital accelerators has $O(BN^2)$ complexity to update an $N\times N$ matrix.
Consequently, AIMC accelerators perform the update operation in a fast and energy-efficient way \cite{gokmen2016acceleration}.

\begin{figure*}[t]
  \centering
  \includegraphics[width=0.8\linewidth]{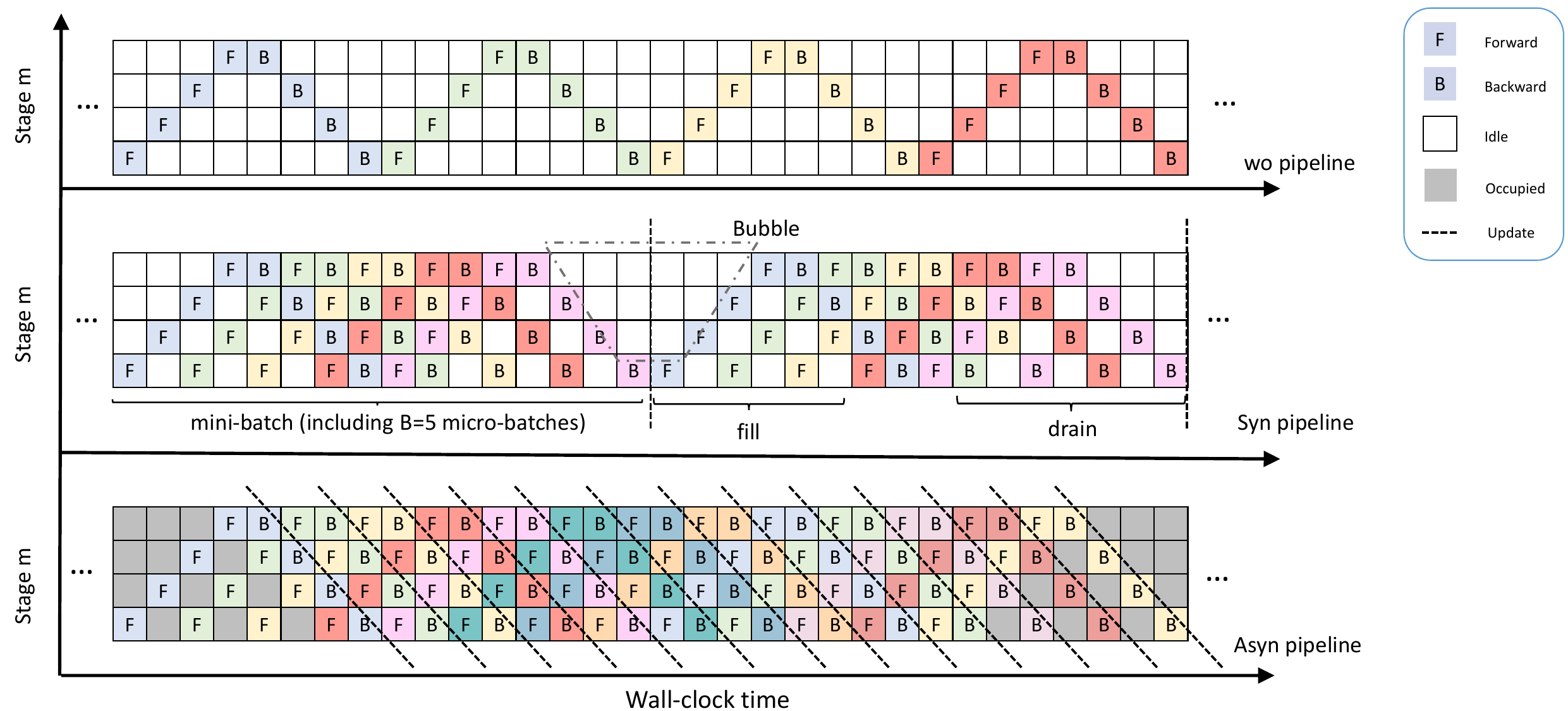}
  \caption{Illustration of pipelines with 4 accelerators ($M=4$). Each mini-batch is split into $B=4$ micro-batches. Each color represents one micro-batch, and each row from bottom to top represents stages 0 to 4. Each column corresponds to a single clock cycle in which one micro-batch is processed. The white square indicates the corresponding accelerator is idle.
  \textbf{(Top)} {\AnalogSGD}. All weights are updated after each mini-batch is computed, which is not presented in the figure. 
  \textbf{(Middle)} {\AnalogSGDSP} with micro-batch count $B=5$. The update occurs after all 5 micro-batches are processed.
  \textbf{(Bottom)} {\AnalogSGDAP}. The update occurs once the gradient of a single micro-batch is achieved.
  The grey square indicates that the corresponding accelerator is processing data that is not fully reflected in the figure.
  }
  \label{figure:syn-vs-asyn}
  \vspace{-0.4cm}
\end{figure*}

\subsection{Non-ideal update on AIMC accelerator}
\label{section:response-functions}
Despite the ultra-efficient weight update on the AIMC accelerator, as noted by \citep{gokmen2020,wu2024towards,wu2025analog}, the update on AIMC accelerators exhibits asymmetric bias. To be specific, for any desired update $\Delta W_k^\spp{m}$, the update on $W_k$ is scaled by response functions $q_+(\cdotc)$ and $q_-(\cdotc)$.
Considering the weight $W_k$ at iteration $k$ and the expected update $\Delta W\in\reals^D$, we express the asymmetric update as $W_{k+1}=U(W_k, \Delta W)$ with the update function $U: \reals\times\reals\to\reals$ defined by
\begin{align}
    \label{analog-update}
    U(w, \Delta w) := \begin{cases}
        w + \Delta w \cdot q_{+}(w),~~~\Delta w \ge 0, \\
        w + \Delta w \cdot q_{-}(w),~~~\Delta w < 0.
    \end{cases}
\end{align}
Defining the symmetric and asymmetric components as $F(w) := \frac{1}{2}(q_{-}(w)+q_{+}(w))$ and $G(w) := \frac{1}{2}(q_{-}(w)-q_{+}(w))$, we write the update in \eqref{analog-update} as $U(w, \Delta w)=w + \Delta w \cdot F(w) -|\Delta w| \cdot G(w)$. 
This paper considers a special case that $q_{+}(w) = 1 - {w}/{\tau}, q_{-}(w) = 1 + {w}/{\tau}$, which leads to the dynamics (1).
The linear forms of $q_+(w)$ and $q_-(w)$ can be viewed as a first-order approximation of more general nonlinear device response functions along the training trajectory.
In this model, $\tau$ controls how strongly the update magnitude depends on the current conductance state.
Physically, a larger $\tau$ means weaker state dependence (more symmetric updates), while a smaller $\tau$ means stronger asymmetry and stronger bias during training.
Therefore, $\tau$ is a hardware-specific constant. Refer to \citep{wu2024towards} for more details about the background of this model.

\begin{remark}
    As established in \citep{gokmen2020, wu2024towards}, the asymmetric update is one of the most significant challenges in analog training, and our theoretical framework explicitly captures this dominant source of non-ideality through the bias term in equation (1). In the theoretical analysis, we focus on the asymmetric update as the primary source of hardware imperfection to balance between faithfulness to the hardware and analytical tractability. Beyond the asymmetric update, the simulations additionally incorporate a comprehensive set of hardware imperfections to more accurately reflect real device behavior; see Section \ref{section:experiments}.

    Beyond the update model in (6), additional hardware effects (e.g., output noise) can be viewed as bounded perturbations to the idealized update dynamics. Under such perturbations, the dominant convergence trend is expected to be preserved, while constants and the asymptotic error may increase by perturbation-dependent terms. A full study that explicitly incorporates these non-idealities is left for future work.
\end{remark}

\subsection{Synchronous and asynchronous pipelines}
Consider training a large model using a mini-batch size $B_{\text{mini}}$, which is too large to fit into one accelerator and necessitates \emph{model parallelism}. 
In model parallelism, the forward and backward computations are partitioned into a sequence of \emph{stages}, which, in general, consist of a consecutive set of layers in the model. 
Each stage is mapped to one digital or analog accelerator.
Apart from the model separation, a mini-batch of data is also split into $B$ \emph{micro-batch} of size $B_{\text{micro}} := B_{\text{mini}} / B$. 
Without loss of generality, this paper assumes that one stage contains only one layer and the micro-batch size $B_{\text{micro}} = 1$. Under this assumption, $\xi_{k,b}$ denotes the $b$-th micro-batch which contains only the $b$-th data.
Consequently, the number of micro-batches is the mini-batch size, i.e., $B=B_{\text{mini}}$.
This paper treats a stage as the minimal scheduling unit and investigates the time complexity across different scheduling strategies.
To compare time costs, we define the basic time unit, \emph{clock cycle}, which is the minimum time required for all devices to process one micro-batch. 
Note that this definition is based on the assumption that per-stage computation loads are balanced across all devices, so that all stages complete each micro-batch in approximately the same amount of time, which can be achieved via balanced partitioning \citep{huang2019gpipe, narayanan2019pipedream, shoeybi2019megatron}.
Based on scheduling strategies, there are three types of model parallelism: vanilla model parallelism without pipelining, synchronous pipelining, and asynchronous pipelining.
\noindent
\textbf{Vanilla model parallelism without pipeline. }
Let $\xi_{k,b}, b\in[B]$ denote the $b$-th micro-batch (containing only one data sample, as previously mentioned) for iteration $k$. 
Vanilla model parallelism processes micro-batches sequentially. 
In the forward pass, stage $m$ achieves the output of the previews stage $x^\spp{m}_{k,b}$ as input and outputs $x^\spp{m+1}_{k,b}$ according to \eqref{recursion:forward}, while in the backward pass stage $m$ achieves the backward signal $\delta^\spp{m}_{k,b}$ from next stage and sends the error $\delta^\spp{m-1}_{k,b}$ to its previews stage according to \eqref{recursion:backward}.
The signals 
$x^\spp{m}_{k, b}$ and $\delta^\spp{m}_{k, b}$ 
are stashed when the gradient for the $b$-th micro-batch is completed. 
After achieving all $\{(x^\spp{m}_{k, b}, \delta^\spp{m}_{k, b}):b\in[B_{\text{mini}}]\}$, the weights will be updated using the gradient-based method.
Beginning from $W_{k,0}=W_k$, Analog SGD update the weight as $W_{k+1}^\spp{m} = W_{k,B}^\spp{m}$, which iterates over $b\in[B_{\text{mini}}]$ by
\begin{align}
    \label{dynamic:analog-SGD-sync}
    W_{k, b+1}^\spp{m} =&\ W_{k, b}^\spp{m} - \frac{\alpha}{B} \nabla_{W^\spp{m}} f(W_k; \xi_{k, b}) 
    \bkeq
    - \frac{\alpha}{\tau B}|\nabla_{W^\spp{m}} f(W_k; \xi_{k, b})|\odot W_{k,b}^\spp{m}.
\end{align}
Notice that for digital SGD, the last term in the {\ac{RHS}} of \eqref{dynamic:analog-SGD-sync} disappears, and hence \eqref{dynamic:analog-SGD-sync} reduces to mini-batch SGD, i.e., $W_{k+1}^\spp{m} = W_{k}^\spp{m} - \frac{\alpha}{B}\sum_{b=1}^{B} \nabla_{W^\spp{m}} f(W_k; \xi_{k, b})$. 
To fully release the potential of rank-update, AIMC accelerators follow \eqref{dynamic:analog-SGD-sync} to update its weight, which in turn complicates the convergence analysis.

\noindent
\textbf{Synchronous pipeline. } 
Vanilla model parallelism is simple to implement. However, most devices remain idle, reducing system throughput.
Existing work has already attempted to design an architecture that supports the pipeline training \cite{shafiee2016isaac, song2017pipelayer}. They faithfully implement the iteration \eqref{dynamic:analog-SGD-sync} but overlap the computing inside a mini-batch. 
Featuring synchronization at the end of a mini-batch, this stage schedule mechanism is called \emph{synchronous pipeline}.
In the vanilla non-pipeline version, stage $m$ remains idle after completing the forward computation of the $b$-th micro-batch until the backward signal returns from stage $m+1$. In the synchronous pipelining, each stage sends the output signal $x^\spp{m+1}_{k,b}$ to the next stage, while simultaneously starting to process the next forward signal $x^\spp{m}_{k,b+1}$ or backward signal $\delta^\spp{m}_{k,b-(M-m)}$, if any. 
Similar to the vanilla model parallelism, the weights will be updated after processing all micro-batches.
Consequently, for each mini-batch, most of the devices need to wait until the last micro-batch is processed. After that, each device collects the gradients of the whole mini-batch and updates weights iteratively following \eqref{dynamic:analog-SGD-sync}.
To implement that, a synchronization operation is required to wait for the completion of gradient computation before starting the next mini-batch. During the pipeline filling and draining periods, parts of the devices are idle; see the middle of Fig. \ref{figure:syn-vs-asyn}

\noindent
\textbf{Asynchronous pipeline.} 
To fully utilize all devices, a more progressive strategy, \textit{asynchronous pipeline}, is proposed. Instead of using mini-batch descent, asynchronous pipeline performs gradient descent on each micro-batch. 
Let $\xi_{k}$ denote the micro-batch for iteration $k$. 
In the forward pass, stage $m$ receives the output of the previous stage $x^\spp{m}_k$ as input and outputs $x^\spp{m+1}_k$, while in the backward pass, stage $m$ receives the backward signal $\delta^\spp{m}_k$ from next stage and sends the error $\delta^\spp{m-1}_k$ to its previous stage.
Unlike the synchronous pipeline, 
the asynchronous pipeline updates the weights once $x^\spp{m}_k$ and $\delta^\spp{m}_k$ are ready. 
Without regular synchronization and filling/draining period, no device is idle during the training, maximizing the utilization.
See Fig. \ref{figure:syn-vs-asyn} for the comparison of different parallel computing strategies.

Despite its perfect resource usage, asynchronous pipeline parallelism suffers from the stale signal issue. 
Notice that in the asynchronous pipeline setting, the forward signal $x^\spp{m}$ and backward signals $\delta^\spp{m}$ are computed based on different stale weights. 
For example, after outputting forward signal $x^\spp{0}$ for the $k$-th micro-batch, the first stage ($m=1$) receives the backward signal $\delta^\spp{0}$ after $2(M-1)$ clock cycles, wherein $M-1$ gradient updates happen.
We introduce different symbols for precise description, which have tildes on top of them to indicate the staleness.
Since there are $M-m$ gradient updates between the forward and backward passes in stage $m$, the weight used for the forward pass has a delay of $M-m$. 
In contrast, the weights are updated by the backward pass, which means the weights used for the backward pass are the latest ones. 
The forward and backward passes of asynchronous pipeline training can be summarized as  
\begin{tcolorbox}[emphblock, width=1\linewidth]
    \vspace{-1\baselineskip}
    \begin{align}
        \label{dynamic:asyn-forward}
        & \textsf{Forward} & &\ 
        \tdx^\spp{0}_k = x_k,~~
        \tdx^\spp{m+1}_k = g^\spp{m}(\tdz^\spp{m}_k),\\
        & & &
        \tdz^\spp{m}_k = W^\spp{m}_{k-(M-m)} \tdx^\spp{m}_k,~~
        \nonumber
         \\
         \label{dynamic:asyn-backward}
        & \textsf{Backward} &&\ \tddelta^\spp{m}_k =\tddelta^\spp{m+1}_kW^\spp{m+1}_k \nabla g^\spp{m}(\tdz^\spp{m}_k).
    \end{align}
\end{tcolorbox}

\noindent
\textbf{Imperfect gradient-based training on AIMC accelerators.}
After getting the backward signal $\delta^\spp{m}$, gradient-based algorithms can be used to update the weights. However, there is a bias due to the hardware imperfection \citep{wu2024towards}. Replacing $\Delta W_k$ with $-\alpha\tddelta^\spp{m}_k \otimes \tdx^\spp{m}_k$ in \eqref{dynamic:analog-update}, where $\alpha> 0$ is a positive learning rate, we reach the update dynamics of weights 

\begin{tcolorbox}[emphblock, width=1\linewidth]
    \vspace{-1\baselineskip}
    \begin{align}
        \label{dynamic:asyn-update-analog-sgd}
        \textsf{Update} \quad W_{k+1}^\spp{m} =&\ W_k^\spp{m} - \alpha\tddelta^\spp{m}_k \otimes \tdx^\spp{m}_k 
        \bkeq
        - \frac{\alpha}{\tau} |\tddelta^\spp{m}_k \otimes \tdx^\spp{m}_k| \odot W_k^\spp{m}.
    \end{align}
    \vspace{-1.4\baselineskip}
\end{tcolorbox}
\noindent
In \eqref{dynamic:asyn-update-analog-sgd}, the outer product $\tddelta^\spp{m}_k \otimes \tdx^\spp{m}_k$ is not a valid gradient with respect to any weight $W^\spp{m}$, which is referred to as partially stale gradient and makes the analysis more challenging.

The main challenge coming from the AIMC training is the \emph{weight drift} \citep{wu2024towards}. Consider the case where $W_k^\spp{m}$ is already a non-zero critical point of $f(W)$ but not the one of $f(W; \xi)$, i.e. $\mbE_{\xi\sim\ccalP}[\nabla_{W^\spp{m}} f(W_k^\spp{m}; \xi)]=0$ but $\nabla_{W^\spp{m}} f(W_k^\spp{m}; \xi)\neq 0$ for some $\xi$. After one iteration, the weight $W_k^\spp{m}$ drifts because of the sample randomness
\begin{align}
    \label{equation:analog-SGD-drift}
    \mbE[W_{k+1}^\spp{m}]-W_k^\spp{m}
    =&\ -\frac{\alpha}{\tau} \mbE[|\nabla_{W^\spp{m}} f(W_k^\spp{m}; \xi)|]\odot W_k^\spp{m} 
    \nonumber\\
    \ne&\ 0
\end{align}
which prevents the weights from converging to the critical point, leading to an error.

\arxivonly{
\begin{remark}[Communication overhead]
    In all three pipeline strategies considered in this paper, inter-accelerator communication consists of passing activation vectors $x^\spp{m}$ and backward signals $\delta^\spp{m}$ between adjacent stages. 
    The per-stage communication cost remains O(D) as M grows. Accordingly, deeper pipelines do not increase the asymptotic per-stage communication burden, although the realized wall-clock benefit still depends on bandwidth and latency remaining secondary to per-stage computation.
    Furthermore, the communication overhead can be reduced by overlapping gradient communication with computation as proposed in \cite{chen2022sapipe, li2018pipe}, thereby hiding inter-stage communication latency within the computation time of each clock cycle.
\end{remark}}
\section{Convergence \& Wall-clock Analysis}
\label{sec.convergence}
This section shows that despite stale gradient issue and the update bias in AIMC accelerators, the asynchronous pipeline converges with the same sample complexity as the counterpart without pipeline or with synchronous pipeline.
By increasing the computation density, asynchronous pipeline parallelism provides notable speedup in terms of wall-clock complexity. 
\subsection{Assumptions}
Before showing the complexity of the pipeline algorithms, we first introduce a series of assumptions on the properties of DNN structures, data distribution, as well as the AIMC accelerators.
$\|\cdotc\|$ is $\ell_2$-norm for vectors and Frobenious norm for matrices; $\|\cdotc\|_\infty$ takes the element with maximal absolute value of vector or matrix.

\begin{assumption}[Unbiasness and bounded variance]
    \label{assumption:noise}
    The stochastic gradient is unbiased and has bounded variance $\sigma^2$.
\end{assumption}
\begin{assumption}[Normalized feature]
    \label{assumption:normalized-feature}
    The feature is normalized, i.e. $\|x\|=1, \forall x\in\ccalX$.
\end{assumption}
\begin{assumption}[Activation function]
    \label{assumption:g-Lip}
    All activation functions are
    (i) Lipschitz continuous, i.e. $\|g^\spp{m}(z)-g^\spp{m}(z')\|\le L_{g,0}\|z-z'\|$;
    (ii) smooth, i.e. $\|\nabla g^\spp{m}(z)-\nabla g^\spp{m}(z')\|\le L_{g,1}\|z-z'\|$;
    (iii) centered, i.e. $g(0)=0$.
\end{assumption}
The conditions (i) and (ii) in Assumption \ref{assumption:g-Lip} hold for many activation functions like Sigmoid, Tanh, or smoothed ReLU functions, while (iii) can be achieved by adding an offset mapping before the activation function if there is a point $x$ such that $g(x)=0$.

\arxivonly{
\begin{remark}[Non-smooth activation functions]
    Assumption \ref{assumption:g-Lip} requires the activation functions to be smooth. For non-smooth activations such as ReLU, the signal smoothness (see Lemmas \ref{lemma:signal-stability-short}, \ref{lemma:objective-L-smooth-short} later) cannot be guaranteed in general. In this case, one can establish that the functions are semi-smooth \cite{zou2019improved}. In the overparameterized regime, such a semi-smoothness property can be established by showing that the training trajectory remains within a small neighborhood of initialization where the set of active neurons is nearly fixed.
    Extending the convergence analysis of {\AnalogSGDAP} to non-smooth activations within this semi-smoothness framework is an interesting direction we leave for future work.
\end{remark}
}

\begin{assumption}[Loss function]
    \label{assumption:loss-Lip}
    The loss $\ell(\cdot, \cdot)$ is:
    (i) smooth with respect to the first input, i.e., $\|\nabla_{1}\ell(y_1, y)-\nabla_{1}\ell(y_2, y)\|\le L_{\ell,1}\|y_1-y_2\|$;
    (ii) lower bounded, i.e., $\ell(y_1, y_2)\ge f^*$.
\end{assumption}
Aside from these assumptions, to avoid saturation during training, following the setting in \cite{wu2024towards}, the saturation degree is assumed to be bounded.
\begin{assumption}[Bounded Saturation]
    \label{assumption:W-bounded}
    There exists a positive constant $W_{\max,\infty}<\tau$ such that all the weight $W^\spp{m}_k$, $m\in[M]$ is $\ell_\infty$-norm bounded, i.e., 
    $\|W^\spp{m}_k\|_\infty\le W_{\max,\infty} < \tau, \forall k\in[K]$. The ratio $W_{\max,\infty}/\tau=\Theta (1)<1$ is termed saturation degree.
\end{assumption}

From the physical perspective, Assumption \ref{assumption:W-bounded} requires that the weight can be represented by a conductance state in the dynamic range $(-\tau, \tau)$. Actually, the analog update dynamic induced by the hardware naturally constrains the weight, as shown by the following theorem. 
\begin{proposition}[{\citep[Theorem 1]{wu2024towards}}]
    \label{theorem:bounded-variable}
    Given $\|W_0\|_\infty< \tau$ and for any sequence $\{\Delta W_k:k\in\naturals\}$, which satisfies $\|\Delta W_k\|_\infty\le \tau$, it holds for any $W_k$ following \eqref{dynamic:analog-update} that $\|W_k\|_\infty< \tau,\,\forall k\in\naturals$.
\end{proposition}

Proposition \ref{theorem:bounded-variable} shows that if the magnitude of the update increment $\Delta W_k$ is bounded by $\tau$, then the entire trajectory remains within the dynamic range $(-\tau,\tau)$. 
\arxivjournal{We further establish a stronger result in Lemma \ref{lemma:bounded-signal}: }{In the long version of the paper \cite{wu2025pipelineArxiv}, we further establish a stronger result: }
with a properly chosen stepsize $\alpha$, the Analog SGD update with asynchronous pipeline in \eqref{dynamic:asyn-update-analog-sgd} guarantees that the trajectory of $W_k$ stays within the dynamic range $(-\tau,\tau)$, without requiring any additional assumption on bounded update increments or bounded delay effects. 
However, we note that 
\arxivjournal{both Proposition \ref{theorem:bounded-variable} and Lemma \ref{lemma:bounded-signal} are}{Proposition \ref{theorem:bounded-variable} is} established only for the finite-time setting, so that they do not rule out the extreme limit case where $\lim_{k\to\infty}|W_k|=\tau$. 
This is because the proof is built upon a worst-case analysis that accounts only for the gradient magnitude, rather than its direction. Under additional over-parameterization assumptions for DNNs \citep{oymak2019overparameterized, liu2022loss, zou2019improved}, it is possible to rule out this extreme case. However, a formal derivation is beyond the scope of this paper, whose primary focus is on the impact of pipeline parallelism. Moreover, in practice, the ratio $W_{\max,\infty}/\tau$ is much smaller than $1$. 
We empirically monitor the saturation degree, defined as $\|W_k^{(m)}\|_\infty/\tau$, throughout the training process. 
As illustrated in Section \ref{section:experiments}, the saturation degree for all layers remains remarkably stable and stays well below 10\% across the full trajectory. Therefore, Assumption \ref{assumption:W-bounded} is sensible in practice. 

\subsection{Iteration complexity of pipeline training}\label{sec:iteration_complexity_pipeline}
Next, we present the convergence rate of analog SGD with pipelines.
In non-convex optimization tasks like the DNN training, 
it is standard to use the stationary measure as $E_K := \frac{1}{K}\sum_{k=0}^{K-1}\mbE[\|\nabla f(W_k)\|^2]$.
As a warm-up, we investigate the convergence of the synchronous pipeline, {\AnalogSGDSP}.

\begin{restatable}[Iteration complexity, \AnalogSGDSP]{theorem}{ThmASGDSyncConvergenceNoncvxLinear}
      \label{theorem:ASGD-sync-convergence-noncvx-linear}
      Under Assumption \ref{assumption:noise}--\ref{assumption:W-bounded}, if the learning rate is set as $\alpha=\sqrt{\frac{B(f(W_0) - f^*)}{\sigma^2L_fK}}$ and $K$ is sufficiently large such that $\alpha\le \frac{1}{L}$,
  it holds that
    \begin{align}
		E_K
        \le 4\sqrt{\frac{(f(W_0) -f^*)\sigma^2L_f}{BK}}
        \frac{1}{1-W_{\max,\infty}^2/\tau^2}
        + \sigma^2 S
        \nonumber
    \end{align}
    where $B$ is the number of micro-batches, and the constant $S$ is given by 
    $S := \frac{W_{\max,\infty}^2 / \tau^2}{1-W_{\max,\infty}^2/\tau^2}$.
\end{restatable}
\arxivjournal{}{See the long version of this paper \cite{wu2025pipelineArxiv} for the proof.}
{\AnalogSGDSP} converges in the rate $\ccalO(\sqrt{{\sigma^2}/({BK})})$, which matches the rate of digital SGD \cite{bottou2018optimization}.
Besides, the asymmetric bias introduces an additional term $\sigma^2 S$, which is independent of $K$ and hence referred to as \emph{asymptotic error}, which also occurs in {\AnalogSGD} \citep{wu2024towards}, and can be mitigated by a well-designed AIMC accelerator or error-correction algorithms \citep{wu2024towards,gokmen2020,gokmen2021}.  
Next, we provide the iteration complexity of \AnalogSGDAP.

\begin{restatable}[Iteration complexity, \AnalogSGDAP]{theorem}{ThmASGDAsyncConvergenceNoncvxLinear}
  \label{theorem:ASGD-async-convergence-noncvx-linear}
  Under Assumption \ref{assumption:noise}--\ref{assumption:W-bounded}, if the learning rate is set as $\alpha=\sqrt{\frac{f(W_0) - f^*}{\sigma^2L_fK}}$ and $K$ is sufficiently large such that $\alpha\le \frac{1}{L_f}$,
  it holds that
    \begin{align}\label{convergence_Ek}
		E_K
        \le &\ 
        4\sqrt{\frac{(f(W_0) -f^*)\sigma^2L_f}{K}}
        \frac{1}{1-(1+u)W_{\max,\infty}^2/\tau^2} 
        \nonumber\\
        &\ 
        + \sigma^2 S' 
        +\ccalO\lp \frac{M^3(1+1/u)}{K}\rp
    \end{align}
    where $u > 0$ is any number and $S'$ denotes the amplification factor given by $S' := \frac{(1+u)W_{\max}^2/\tau^2}{1-(1+u)W_{\max}^2/\tau^2}$.
\end{restatable}
Similar to {\AnalogSGD} or {\AnalogSGDSP}, we have \emph{asymptotic error} $\sigma^2S^\prime$, which can be mitigated by error-correction algorithms \citep{wu2024towards,gokmen2020}. Moreover, since $\ccalO\lp \frac{1+1/u}{K}\rp$ is not the dominant term, one can set a small $u$ so that {\AnalogSGDAP} achieves the same convergence order as the {\AnalogSGD} \citep{wu2024towards}  owing to our proof techniques in Section \ref{sec:proof_sketch}.

\begin{remark}
As a side product, we also provide an enhanced convergence analysis for SGD with asynchronous pipeline training on digital accelerators, detailed in 
\arxivjournal{Appendix \ref{sec:digital-pipeline-SGD}.}{the long version of this paper \cite{wu2025pipelineArxiv}.}
Our analysis, based on a novel Lyapunov function, shows that asynchronous pipeline effects appear only in higher-order terms, which does not slow down digital SGD, and the analysis is tighter than  \citep{xu2020acceleration}. 
\end{remark}

\subsection{Proof sketch of Theorem \ref{theorem:ASGD-async-convergence-noncvx-linear}}
\label{sec:proof_sketch}
In this subsection, we highlight the key steps of the proof towards Theorem \ref{theorem:ASGD-async-convergence-noncvx-linear}. 
We only summarize the key lemmas used for analysis. 
\arxivjournal{The proofs for each lemma can be found in \ref{section:proof-ASGD-async-convergence-noncvx-linear}.}{The proofs for each lemma can be found in the long version of this paper \cite{wu2025pipelineArxiv}.}

\noindent
\textbf{(Step I) Smoothness of the objective.} Before analyzing the training dynamic, we show that the objective is smooth. 
Recall that $\nabla_{W^\spp{m}} f(W; \xi) = \delta^\spp{m} \otimes x^\spp{m}$, we first show that under the assumptions of DNN structures, both $\delta^\spp{m}$ and $x^\spp{m}$ are smooth. Using slightly abusing notations, we denote the forward and backward signals computed by $W$ as $x^\spp{m}$ and $\delta^\spp{m}$, respectively; and denote the ones computed by $\tdW$ as $\tdx^\spp{m}$ and $\tddelta^\spp{m}$, respectively. 
The following lemma shows that for each layer $m$, if we regard the signals as functions of weights, they are smooth functions.
\begin{lemma}[Signal stability]
    \label{lemma:signal-stability-short}
    Denote the forward signals computed by $W$ and $\tdW$ as $x^\spp{m}$ and $\tdx^\spp{m}$, respectively, and denote the backward signals as $\delta^\spp{m}$ and $\tddelta^\spp{m}$, respectively.
    \begin{align}
        \label{inequality:x-stability-short}
        \|\tdx^\spp{m}-x^\spp{m}\|^2 \le&\ C_{x'}^\spp{m} 
        \|\tdW-W\|^2,\\ 
        \label{inequality:z-stability-short}
        \|\tdz^\spp{m}-z^\spp{m}\|^2 \le&\ C_{z'}^\spp{m} 
        \|\tdW-W\|^2,\\ 
        \|\tddelta^\spp{m}-\delta^\spp{m}\|^2 \le&\ C_{\delta'}^\spp{m}
        \|\tdW-W\|^2,
    \end{align}
    where $C_{x'}^\spp{m}$, $C_{z'}^\spp{m}$, and $C_{\delta'}^\spp{m}$ are constants that depend on the DNN architecture.
\end{lemma}
Based on this lemma, we prove the objective is smooth.
\begin{lemma}[Smoothness of the objective]
    \label{lemma:objective-L-smooth-short}
    Under Assumption \ref{assumption:normalized-feature}--\ref{assumption:W-bounded}, the objective is $L_f$-smooth with respect to $W$, i.e. $\forall \xi=(x, y)$, $\|\nabla f(W;\xi)-\nabla f(\tdW;\xi)\| \le L_f\|W-\tdW\|$, where $L_f$ is a constant depending on the DNN archetecture.
\end{lemma}
Lemma \ref{lemma:objective-L-smooth-short} enables us to adapt the tools from \cite{bottou2018optimization} to obtain sufficient descent of the function value.
However, the bias from the asymmetric update and stale signals from asynchronous pipelining accumulate layer by layer within the DNN and eventually lead to an extra term in the upper bound. Next step, we quantify the error of the stale signals.

\noindent
\textbf{(Step II) Error of stale signals.} Similar to Step I, we begin from analyzing the error of $x^\spp{m}$ and $\delta^\spp{m}$. However, the stale signal involves the weights at different iterations and introduces additional errors. 
Intuitively, the error of the forward signal $\tdx^\spp{m}_k$ comes from the stale weight $W^\spp{m}_{k-(M-m)}$ as shown by \eqref{dynamic:asyn-forward}; while the error of the backward signal arises from the error of the forward signal.
It motivated us to bound the error by the delay.
Since the error between latest $W^\spp{m'}_k$ and the stale weight $W^\spp{m'}_{k-(M-m)}$ can be decomposed as
\begin{align}
    W^\spp{m'}_k-W^\spp{m'}_{k-(M-m)} = \!\!\!\sum_{k'=k-(M-m')}^{k-1}\!\!\! (W^\spp{m'}_{k'+1}-W^\spp{m'}_{k'}),
\end{align}
the following accumulated delay of layer $m$ can be used for analyzing the upper bound of the stale signals' error
\begin{align}
    \psi_k := \sum_{m'=1}^{M}\sum_{k'=k-(M-m')}^{k-1} 
    \|\Delta^\spp{m'}_{k'}\|^2
\end{align}
where $\Delta^\spp{m}_{k} := W^\spp{m}_{k+1}-W^\spp{m}_k$ is the difference between consecutive weights at iteration $k$.

\begin{lemma}[Delay error of signals]
    \label{lemma:signal-delay-short}
    Under Assumption \ref{assumption:normalized-feature}--\ref{assumption:loss-Lip}, the difference between the latest and stale signals generated by {\AnalogSGDAP} are bounded, i.e., $\forall m\in[M+1]$
    \begin{align}
        \|\tdx^\spp{m}_k-x^\spp{m}_k\| \le&\ C_x^\spp{m} 
        \psi_k,
        \\ 
        \|\tdz^\spp{m}_k-z^\spp{m}_k\| \le&\ C_z^\spp{m} 
        \psi_k,
        \\ 
        \|\tddelta^\spp{m}_k-\delta^\spp{m}_k\| \le&\ C_{\delta}^\spp{m}
        \psi_k,
    \end{align}
    where $C_{x}^\spp{m}$, $C_{z}^\spp{m}$, and $C_{\delta}^\spp{m}$ are constants that depends on the DNN architecture.
\end{lemma}

\begin{remark}
    The inequalities in Lemma \ref{lemma:signal-stability-short} and \ref{lemma:signal-delay-short} have a similar form, but they provide upper bounds for different error: in Lemma \ref{lemma:signal-stability-short}, two signals are computed on two arbitrary weights $W$ and $\tdW$, while in Lemma \ref{lemma:signal-delay-short}, two signals are generated by {\AnalogSGDAP}.
\end{remark}

\noindent
\textbf{(Step III) Sufficient descent of the objective.} 
Building on the results in Steps I and II, we will analyze the sufficient descent of the objective function. However, due to the asymmetric bias of the AIMC accelerators, the convergence results for analog SGD are still non-trivial. 
\begin{lemma}[Sufficient descent]
    \label{lemma:ASGD-descent}
    Under Assumption \ref{assumption:noise}--\ref{assumption:W-bounded}, the objective function $f(W_k)$ satisfies
    \begin{align}\label{accumu_delay}
        &\ 
        \mathbb{E}_{\xi_k}[f(W_{k+1})] - f(W_k) \\
        \lesssim&\ - \frac{\alpha}{2} \lp 1-(1+u)\frac{\|W_k \|^2_\infty}{\tau^2} \rp\|\nabla f(W_k)\|^2 
        + \alpha^2L_f\sigma^2
        \nonumber 
        \\
        &\ 
        + \frac{\alpha\sigma^2}{2\tau^2}(1+u)\|W_k \|^2_\infty
        - \frac{1}{4\alpha}\ccalO\lp\mathbb{E}_{\xi_k}[\|W_{k+1}-W_k\|^2]\rp
        \bkeq
        + C_\psi \mathbb{E}_{\xi_k}[\psi_k]
        \nonumber
    \end{align}
    where $\lesssim$ performs non-essential approximation on the upper bound of descent.
\end{lemma}

From Lemma \ref{lemma:ASGD-descent}, we notice that the accumulated delays $\psi_k$ appear in the descent lemma, but intuitively they vanish when the algorithm converges according to its definition.

\noindent
\textbf{(Step IV) Construction of the Lyapunov function.} 
To address the key challenge that the accumulated delay error $\psi_k$ appears in the descent lemma (Lemma \ref{lemma:ASGD-descent}) as a positive term that prevents a clean telescoping argument, we augment the objective $f(W_k)$ with an auxiliary function $\Psi_k$ that satisfies two properties: (i) $\Psi_{k+1}-\Psi_{k}$ produces a term that cancels $\psi_k$ in the descent of $f(W_k)$ while other terms can be also canceled; and (ii) $\Psi_k$ itself is non-negative so that it works as an upper bound of $f(W_k)$.
Following this, construct a \textit{Lyapunov function} as
\begin{align}
    \mbV_k := f(W_k) - f^* + C_\psi \Psi_k
\end{align}
where the auxiliary function $\Psi_k$ is defined as 
\begin{align}
    \Psi_k := \sum_{m'=1}^{M}\sum_{k'=k-(M-m')}^{k-1} (k'-k+M)
    \|\Delta^\spp{m'}_{k'}\|^2
\end{align}
and $C_\psi$ is a constant.
The weighting scheme $(k' - k + M)$ in equation (23) is specifically designed to satisfy property (i). 
To see this, we view $\Psi_k$ as a weighted version of $\psi_k$, which yields a telescoping structure: it assigns larger weights to more recent iterates, reflecting that recent weight updates contribute more to the current staleness error. To be specific, the closer $k^\prime$ is to the current iteration $k$, the larger the effect of $\Delta_{k^\prime}^{(m^\prime)}$ on the current iteration; if $k^\prime$ is farther from $k$, much of its delayed effect has already been consumed, so it receives a smaller weight. Enabled by this structure, we can prove the property of $\Psi_{k+1} \leq \Psi_k-\psi_k+M\left\|W_{k+1}-W_k\right\|^2$, which cancels the $\psi_k$ term in Lemma \ref{lemma:ASGD-descent}, while the extra $\left\|W_{k+1}-W_k\right\|^2$ term can be canceled by the $\frac{1}{4\alpha}\ccalO\lp\mathbb{E}_{\xi_k}[\|W_{k+1}-W_k\|^2]\rp$ in \eqref{accumu_delay}.
        
To achieve the dynamics of $\Psi_k$, we derive that
\begin{align}
    \label{inequality:Psi-1}
    \Psi_{k+1}
    = &\ \sum_{m'=1}^{M}\sum_{k'=k+1-(M-m')}^{k}
    (k'-(k+1)+M)\|\Delta^\spp{m'}_{k'}\|^2
    \nonumber
    \\
    =&\ \sum_{m'=1}^{M}\sum_{k'=k+1-(M-m')}^{k}
    (k'-k+M)\|\Delta^\spp{m'}_{k'}\|^2
    \bkeqwn
    -\sum_{m'=1}^{M}\sum_{k'=k+1-(M-m')}^{k}
    \|\Delta^\spp{m'}_{k'}\|^2
    \nonumber
\end{align}
Let the index of the inner summation of the first term in the \ac{RHS} of \eqref{inequality:Psi-1} begin from $k-(M-m')$ instead of $k+1-(M-m')$, we have
\begin{align}
    \label{inequality:Psi-2}
    &\ \sum_{m'=1}^{M}\sum_{k'=k+1-(M-m')}^{k-1}
        (k'-k+M)\|\Delta^\spp{m'}_{k'}\|^2 \\
    =&\ \sum_{m'=1}^{M}\sum_{k'=k-(M-m')}^{k-1}
        (k'-k+M)\|\Delta^\spp{m'}_{k'}\|^2 
    \bkeq
    + M\sum_{m'=1}^{M}\|\Delta^\spp{m'}_{k}\|^2
    - \sum_{m'=1}^{M}m' \|\Delta^\spp{m'}_{k-(M-m')}\|^2
    \nonumber\\
    =&\ \Psi_k + M\|W_{k+1}-W_k\|^2
    - \sum_{m'=1}^{M}m' \|\Delta^\spp{m'}_{k-(M-m')}\|^2
    \nonumber
\end{align}
where the last equality holds by definition of $\Delta^\spp{m'}_{k}$, i.e., 
\begin{align}
    \sum_{m'=1}^{M}\|\Delta^\spp{m'}_{k}\|^2
    =&\ \sum_{m'=1}^{M}\|W^\spp{m'}_{k+1}-W^\spp{m'}_k\|^2 
    \\
    =&\ 
    \|W_{k+1}-W_k\|^2.
    \nonumber
\end{align} 
Plugging \eqref{inequality:Psi-2} into \eqref{inequality:Psi-1} yields
    $\Psi_{k+1} \le \Psi_k - \psi_k + M\|W_{k+1}-W_k\|^2$,
which gives a negative $\psi_k$. 
We then show that the Lyapunov functions have sufficient descent if $C_\psi$ is properly chosen, 
i.e.,
\begin{align}
    &\ 
    \mathbb{E}_{\xi_k}[\mbV_{k+1}] - \mbV_k 
    \\
    \lesssim&\ - \frac{\alpha}{2} \lp 1-(1+u)\frac{\|W_k \|^2_\infty}{\tau^2} \rp\|\nabla f(W_k)\|^2 
    + \alpha^2L_f\sigma^2 
    \nonumber\\
    &\ +\alpha^3(1+\frac{1}{u})(1+W_{\max,\infty}/\tau)^2C_{f}^2 M^3 \sigma^2
    \bkeq
    + \frac{\alpha\sigma^2}{2\tau^2}(1+u)\|W_k \|^2_\infty.
\end{align}
Reorganizing the inequality above yields \Cref{theorem:ASGD-async-convergence-noncvx-linear}. 

\begin{figure*}[t]
	\centering
	\vspace{-0.2cm}
		\includegraphics[width=0.35\linewidth]{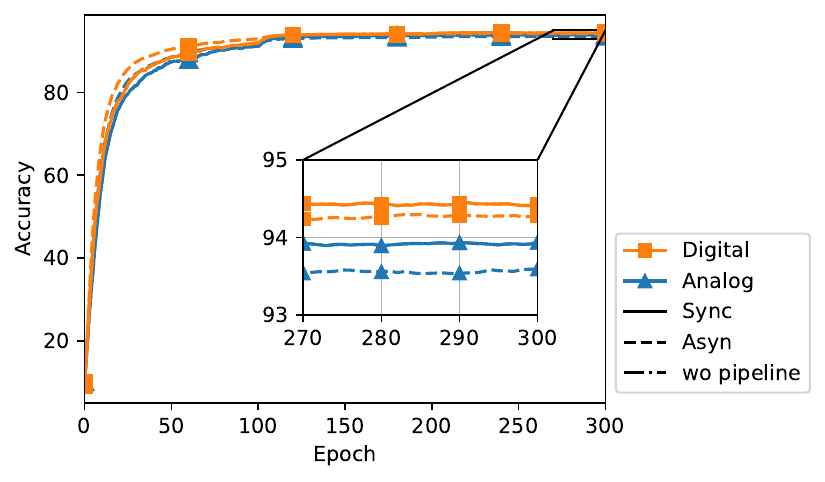}
		\includegraphics[width=0.35\linewidth]{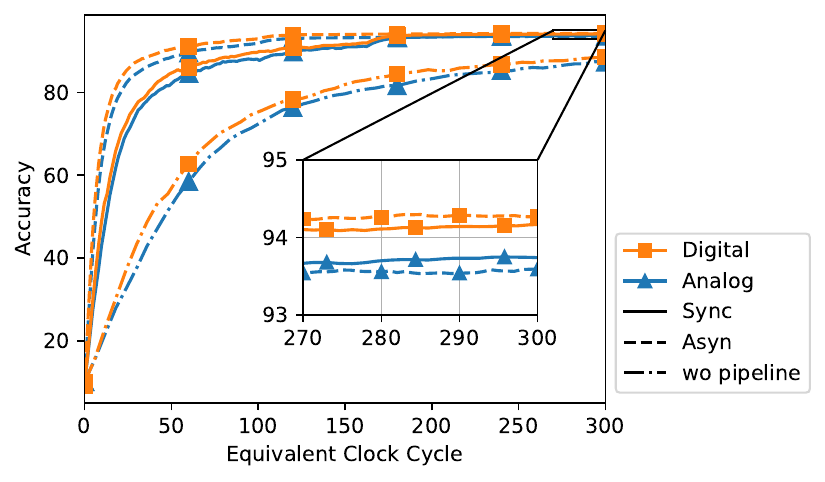}
		\includegraphics[width=0.28\linewidth]{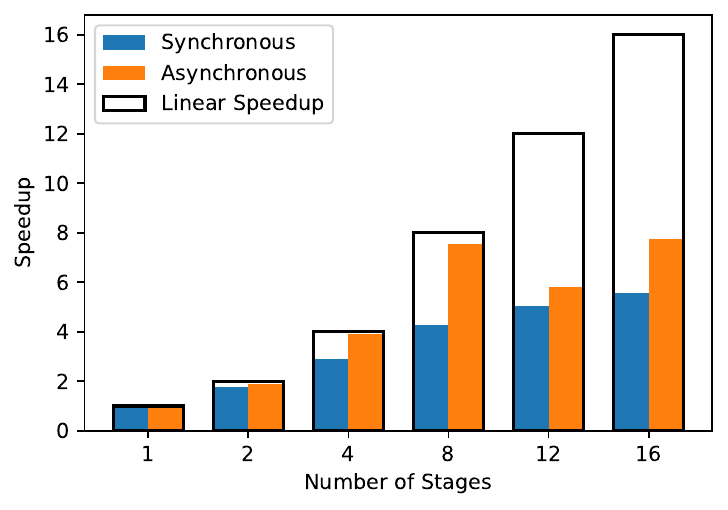}
        \vspace{-0.1cm}
	\caption{
		ResNet10 @ CIFAR10 via vanilla model parallelism without pipeline (wo pipeline), synchronous pipeline (Sync), and asynchronous pipeline (Asyn). 
		\textbf{(Left)} Accuracy-vs-Epoch reflects sample complexity.
		\textbf{(Middle)} Accuracy-vs-Clock cycle.
		\textbf{(Right)} Speedup of the proposed pipeline training methods on AIMC accelerators as the number of stages increases compared with the single machine training case (stage=1). We ignore the communication latency since the asynchronous pipeline does not introduce obvious extra latency.
	}
	\label{fig:resnet10-pipeline-sync-vs-asyn-short}
\end{figure*}

\subsection{Sampling and wall-clock complexity}
\label{sec:sample_clock}
Sample complexity measures the computational cost required to achieve the desired accuracy, whereas in model parallelism, the computation of different stages overlaps. Therefore, we are also interested in the wall-clock time it takes.
In this section, we first study the sample complexity, defined as the minimum number of gradient computations required to achieve the desired error $\varepsilon$. 
Due to asymptotic error, we define the \textit{sample complexity} in analog training as the minimum micro-batch needed for achieving 
$E_K \le \varepsilon + E$,
where $E$ is the asymptotic error, i.e. $\sigma^2 S$ for {\AnalogSGD}/{\AnalogSGDSP} and $\sigma^2 S'$ for {\AnalogSGDAP}, respectively. 
Noticing that the former requires $B$ micro-batches at each iteration while the latter requires $1$, 
we have the following conclusion. 
\begin{corollary}[Sample complexity]
    \label{corollary:sample-complexity}
    The sample complexities of {\AnalogSGD} or {\AnalogSGDSP} are $\ccalO\lp \sigma^2\varepsilon^{-2}\rp$, while that of the {\AnalogSGDAP} is $\ccalO\lp\sigma^2\varepsilon^{-2}+\varepsilon^{-1}\rp$.
\end{corollary}
This shows that three model parallelisms share the same dominant term $\ccalO(\sigma^2\varepsilon^{-2})$ in the sample complexity, while {\AnalogSGDAP} shows down by a non-dominant term $\ccalO(\varepsilon^{-1})$. 
Since the computation is overlapped, the computation density is also needed to determine the number of clock cycles required to achieve $\varepsilon$ accuracy. 

\noindent
\textbf{Computation density.}
In an ideal scenario, a micro-batch requires $2M$ clock cycles for forward and backward passes, distributed equally across $M$ accelerators, taking 2 clock cycles per micro-batch. \emph{Computation density} measures how closely an algorithm approaches this ideal. 
Let $N_{\text{c}}(t)$ be the number of micro-batches processed in $t$ clock cycles. We define computation density as the limit ratio of ideal processing time to real time: $\lim_{t\to\infty} \frac{2N_{\text{c}}(t)}{t}$, where the factor $2$ indicates that micro-batch processing involves forward and backward passes. 
{\AnalogSGD} takes $2M$ cycles per micro-batch, yielding computation density $1/M$. In {\AnalogSGDSP}, it takes $2(M+B-1)$ clock cycles to process $B$ micro-batches, resulting in a density of $B/(M+B-1)$. {\AnalogSGDAP} achieves a density of $1$, as all accelerators remain active during training. 
Similar to sample complexity, \textit{wall-clock complexity} is defined as the minimum clock cycle needed for achieving 
$E_K \le \varepsilon + E$, where $E$ is the asymptotic error. Wall-clock complexity is achieved by dividing the sample complexity by the computation density.
The following corollary summarizes the clock cycle complexity of different pipeline strategies.

\begin{corollary}[Wall-clock complexity]
    \label{corollary:clock-cycle-complexity}
    The wall-clock complexities of {\AnalogSGD}, {\AnalogSGDSP}, and {\AnalogSGDAP} are 
    $\ccalO\lp M\sigma^2\varepsilon^{-2}\rp$, 
    $\ccalO\lp (M+B-1)B^{-1}\sigma^2\varepsilon^{-2}\rp$, 
    and $\ccalO\lp \sigma^2\varepsilon^{-2}+\varepsilon^{-1}\rp$,
    respectively. 
\end{corollary}
\begin{tcolorbox}[emphblock, width=1\linewidth]
    \textbf{Takeaway. } {\AnalogSGDAP} maximizes the computation density of DNN training only at the cost of a higher-order error term compared with the synchronous pipeline.
\end{tcolorbox}

\arxivonly{
\begin{remark}[Load balancing in pipeline parallelism]
    The computation density analysis above assumes that the DNN model is evenly partitioned across $M$ pipeline stages, such that each stage incurs the same computational cost per micro-batch. In practice, uneven partitioning would introduce pipeline bubbles at the bottleneck stage, reducing the effective computation density. 
    However, this load-balancing issue is orthogonal to the convergence analysis of this paper. Theorems \ref{theorem:ASGD-sync-convergence-noncvx-linear} and \ref{theorem:ASGD-async-convergence-noncvx-linear}, as well as Corollary \ref{corollary:sample-complexity}, remain valid regardless of how the model is partitioned, as they bound the sample complexity rather than the wall-clock complexity directly. For Corollary \ref{corollary:clock-cycle-complexity}, a computation density of 1 is an achievable upper bound under balanced partitioning. If the partition is unbalanced, the density should be scaled down by a partition-specific factor, which approaches 1 as partitioning becomes more balanced. In practice, balanced partitioning can be achieved using existing pipeline-parallelism frameworks such as \citep{huang2019gpipe,narayanan2019pipedream, shoeybi2019megatron}; a detailed investigation of load-balancing strategies for AIMC pipelines is beyond the scope of this paper.
\end{remark}

\begin{remark}[Comparison with digital-SGD pipelines]
    In fact, the digital-SGD counterpart is a special case of \eqref{dynamic:analog-update} by forcing $\tau\to\infty$, which removes the asymmetric bias term.
    In particular, the analog-specific denominator inflation, $\frac{1}{1-W_{\max,\infty}^2/\tau^2}$ and $\frac{1}{1-(1+u)W_{\max,\infty}^2/\tau^2}$, and asymptotic error terms $S$ and $S'$ vanish in the digital limit.
    Consequently, the pipelining benefit established in Corollaries \ref{corollary:sample-complexity} and \ref{corollary:clock-cycle-complexity} is not unique to AIMC dynamics and can be reproduced in the digital domain. 
\end{remark}
}

\begin{table*}[t]
	\centering
    \begin{tabular}{cc|ccc|ccc}
        \hline\hline
        & & \multicolumn{3}{c|}{CIFAR10} & \multicolumn{3}{c}{CIFAR100}\\
        & & speedup & accuracy (mid) & accuracy (end)
        & speedup & accuracy (mid) & accuracy (end) \\
        \hline
        \multirow{3}{*}{Digital} & w.o. pipeline & 1.00 & 88.63 \stdv{$\pm$ 0.70} & 94.32 \stdv{$\pm$ 0.12}& 1.00 & 67.85 \stdv{$\pm$ 0.52} & 73.99 \stdv{$\pm$ 0.19} \\
        & synchronous & 3.69 & 94.17 \stdv{$\pm$ 0.15} & 94.32 \stdv{$\pm$ 0.12}& 3.69 & 73.91 \stdv{$\pm$ 0.19} & 73.99 \stdv{$\pm$ 0.19} \\
        & asynchronous & \textbf{6.18} & 94.25 \stdv{$\pm$ 0.17} & 94.25 \stdv{$\pm$ 0.16}& \textbf{8.19} & 74.64 \stdv{$\pm$ 0.48} & 74.64 \stdv{$\pm$ 0.53}  \\
        \hline
        \multirow{3}{*}{Analog} & w.o. pipeline & 1.00 & 88.04 \stdv{$\pm$ 0.58} & 93.95 \stdv{$\pm$ 0.07} & 1.00 & 61.35 \stdv{$\pm$ 0.73} & 72.29 \stdv{$\pm$ 0.30}\\
        & synchronous & 3.69 & 93.75 \stdv{$\pm$ 0.08} & 93.95 \stdv{$\pm$ 0.07} & 3.69 & 72.14 \stdv{$\pm$ 0.31} & 72.29 \stdv{$\pm$ 0.30}\\
        & asynchronous & \textbf{5.66} & 93.60 \stdv{$\pm$ 0.05} & 93.60 \stdv{$\pm$ 0.07} & \textbf{6.17} & 72.69 \stdv{$\pm$ 0.30} & 72.64 \stdv{$\pm$ 0.21}\\
        \hline\hline
    \end{tabular}
    \caption{
        ResNet10 @ CIFAR10/CIFAR100 via vanilla model parallelism without pipeline (wo pipeline), synchronous pipeline (Sync), and asynchronous pipeline (Asyn). 
        The speedup is computed by comparing the number of clock cycles required to achieve 93\% accuracy. 
        The test accuracy and standard deviation at clock cycle 300 (mid) and epoch 300 (end) are reported.
    }
    \label{table:resnet10-cifar10/100-pipeline-sync-vs-asyn}
    
	\vspace{0.2cm}
	\begin{tabular}{cc|ccc|ccc}
		\hline\hline
		& \multirow{2}{*}{$B_{\text{micro}}$}
		& \multicolumn{3}{c|}{Synchronous} 
		& \multicolumn{3}{c}{Asynchronous} \\ 
		& 
		& speedup & acc (mid) & acc (end)
		& speedup & acc (mid) & acc (end) \\
		\hline
		\multirow{3}{*}{D} 
		& $16$ & 3.69 & \multirow{4}{*}{94.17 \stdv{$\pm$ 0.15}} & \multirow{4}{*}{94.32 \stdv{$\pm$ 0.12}}
		& 6.18 & 94.25 \stdv{$\pm$ 0.17} & 94.25 \stdv{$\pm$ 0.16} \\
		& $32$ & 2.67 & 
		&
		& 6.06 & 94.27 \stdv{$\pm$ 0.18} & 94.28 \stdv{$\pm$ 0.20} \\
		& $64$ & 1.71 & 
		&
		& 6.00 & 94.09 \stdv{$\pm$ 0.13} & 94.08 \stdv{$\pm$ 0.15} \\
		& $128$ & 1.00 & 
		&
		& 5.83 & 93.68 \stdv{$\pm$ 0.39} & 93.67 \stdv{$\pm$ 0.38}  \\
		\hline
		\multirow{3}{*}{A} 
		& $16$ & 3.69 & 93.75 \stdv{$\pm$ 0.08} & 93.95 \stdv{$\pm$ 0.07} 
		& 5.66 & 93.60 \stdv{$\pm$ 0.05} & 93.60 \stdv{$\pm$ 0.07} \\
		& $32$ & 2.67 & 93.57 \stdv{$\pm$ 0.12} & 93.95 \stdv{$\pm$ 0.07} 
		& 5.88 & 93.62 \stdv{$\pm$ 0.13} & 93.61 \stdv{$\pm$ 0.11}\\
		& $64$ & 1.71 & 90.86 \stdv{$\pm$ 0.44} & 93.95 \stdv{$\pm$ 0.07} 
		& 5.88 & 93.58 \stdv{$\pm$ 0.14} & 93.60 \stdv{$\pm$ 0.17} \\
		& $128$ & 1.00 & 88.04 \stdv{$\pm$ 0.58} & 93.95 \stdv{$\pm$ 0.07} 
		& 4.89 & 93.36 \stdv{$\pm$ 0.06} & 93.37 \stdv{$\pm$ 0.06} \\
		\hline\hline
	\end{tabular}
	\caption{ResNet10 @ CIFAR10 under different micro-batch size setting. Test accuracy and standard deviation at clock cycle 300 (mid) and epoch 300 (end) are reported.
		The speedup is computed by comparing the number of clock cycles required to achieve 93\% target accuracy. 
		The accuracy of synchronous is not affected by micro-batch size.
	}
	\label{table:performance_metrics2}
	\vspace{-0.2cm}
\end{table*}
\section{Numerical Simulations}
\label{section:experiments}
In this section, we aim to numerically verify the theoretical benefits of {\AnalogSGDAP}. We test the analog methods in an open-source \ac{AIHWKIT}~\citep{Rasch2021AFA}; see \url{github.com/IBM/aihwkit}.
The pipeline parallelism is also simulated on a single GPU.
To get the \emph{equivalent clock cycle} in the figures, we map the epoch to clock cycle by dividing the computation density.

\noindent
\textbf{Simulation settings.}
In the simulations, we train ResNet models on two benchmark datasets: CIFAR10 and CIFAR100. 
The CIFAR10 dataset consists of 60,000 32$\times$32 color images across 10 classes, with each class representing a distinct category such as airplanes, cars, and birds. The CIFAR100 dataset is similar in format but includes 100 classes. 
Each model was trained over 300 epochs using a mini-batch size of $B_{\text{mini}}=128$ and micro-batch-size $B_{\text{micro}}=16$ by default. 
The initial learning rate of $\alpha=0.1$ was employed, while for Resnet34 with stage numbers $M=12, 16$, the initial learning rate is $\alpha=0.01$.
The learning rate $\alpha$ was reduced by a factor of 10 every 100 epochs. 
To further enhance the model’s ability to generalize across diverse images, we applied data augmentation techniques such as random cropping, horizontal flipping, and color jittering \cite{cubuk2018autoaugment}.

\noindent
\textbf{AIMC hardware configurations.}
The simulations use the \texttt{SoftBoundsReference} device in AIHWKIT, which directly instantiates the asymmetric update model in \eqref{analog-update} with $\tau=0.9$ by default. Beyond the asymmetric update, the simulations additionally incorporate a comprehensive set of hardware imperfections to more accurately reflect real device behavior, including output noise (0.5 \% of the quantization bin width), quantization and clipping (output range set to 20, output noise 0.1, and input and output quantization to 8 bits). 
In addition, 30\% device and cycle variation are introduced. 
Noise and bound management techniques are used in~\cite{gokmen2017cnn}.
A learnable scaling factor is applied after each analog layer and updated via digital SGD.

\noindent
\textbf{Sample complexity and wall-clock complexity.}
We train a ResNet10 on CIFAR10.
We map the first 5 stages on digital accelerators while the last one onto an AIMC (analog) accelerator. 
The results on the left of Fig. \ref{fig:resnet10-pipeline-sync-vs-asyn-short} show that the final accuracy differences between the three strategies are below 1\%,
implying all of these algorithms have the same sample complexity and align with Theorem \ref{theorem:ASGD-sync-convergence-noncvx-linear}--\ref{theorem:ASGD-async-convergence-noncvx-linear} and Corollary \ref{corollary:sample-complexity}. Moreover, as indicated by Corollary \ref{corollary:clock-cycle-complexity}, the walk-clock complexity of the {\AnalogSGDAP} pipeline is the lowest owing to the highest computation density. 
The results on the CIFAR100 training are reported in Table \ref{table:resnet10-cifar10/100-pipeline-sync-vs-asyn}, whose conclusion is the same as above.

\noindent
\textbf{Effect of stage/accelerator number. }
We also explored the impact of increasing the number of accelerators on the speedup achieved through pipeline parallelism; see Fig. 
\ref{fig:resnet10-pipeline-sync-vs-asyn-short}, Right.
In the simulations, we train a ResNet34 model \citep{wu2019wider} on the CIFAR10 dataset. The model is separated into $M=1, 2, 4, 8, 12, 16$ stages. 
Assuming the model can be evenly divided into multiple stages, and the equivalent clock cycle needed to compute the entire model is fixed.
The results reveal that the asynchronous pipeline achieves linear speedup at least within the range of $1-8$ accelerators.

As the $M$ increases, the speedup deviates from linear scaling, which can be explained by Theorem \ref{theorem:ASGD-async-convergence-noncvx-linear}. As shown by \eqref{convergence_Ek}, the iteration complexity of {\AnalogSGDAP} contains a higher-order term $O(M^3(1+1/u)/K)$, whose coefficient depends on $M$.  Although this term is higher-order and does not affect the dominant sample complexity $O(\sigma^2\varepsilon^{-2})$, it has a non-negligible practical impact when $M$ increases while iterations $K$ remain small. 
Consequently, it results in the observed deviation from linear speedup. Developing acceleration strategies to mitigate this depth-dependent overhead and recover linear speedup at larger scales ($M > 8$) is an interesting direction we leave for future work.

\noindent
\textbf{Effect of micro-batch size. }
We also report the effect of micro-batch sizes in Table \ref{table:performance_metrics2}. We set the mini-batch size to $B_{\text{mini}}=128$ and vary the micro-batch size. For synchronous pipeline, increasing the batch size improves the speedup on both digital and analog accelerators. Moreover, asynchronous pipeline training consistently yields greater speedup than the synchronous version. The trade-off is that the asynchronous pipeline sacrifices accuracy and can only reach lower accuracy when the training time is sufficiently long.

\noindent
\textbf{Effect of hardware imperfections.}
To disentangle the contributions of analog imperfections and asynchronous pipeline staleness to the convergence error, we compare the performance of digital and analog pipelines under identical configurations, as shown in Tables \ref{table:resnet10-cifar10/100-pipeline-sync-vs-asyn} and \ref{table:performance_metrics2}. The results in Table \ref{table:resnet10-cifar10/100-pipeline-sync-vs-asyn} show that (i) the accuracy gap between digital and analog is consistent ($1\%-2\%$) across all three pipeline strategies; (ii) within each hardware setting (digital or analog), the accuracy difference between synchronous and asynchronous pipelines is negligible at convergence ($<1\%$). These results indicate that pipeline staleness introduces only marginal additional error, and the gap is attributable to analog imperfections.
Table \ref{table:performance_metrics2} further shows that this conclusion holds robustly across different micro-batch sizes: the accuracy gap between digital and analog remains stable regardless of micro-batch size, while the gap between synchronous and asynchronous pipelines remains small in both settings. Taken together, these results confirm that analog imperfection is the dominant source of convergence error, while the asynchronous pipeline staleness has a secondary effect.

\noindent
\textbf{Effect of asymmetry parameter $\tau$.}
We further examine how the asymmetry parameter $\tau$ affects both convergence and acceleration; see Fig. \ref{fig:tau-convergence-speedup-cifar10-dtod-saturation-combined}a/b.
The left figure shows that when $\tau>0.5$, {\AnalogSGDAP} still achieves above 90\% test accuracy.
Furthermore, as shown in the right figure, {\AnalogSGDAP} maintains a speedup well above the synchronous pipeline baseline (shown as the dashed line at $3.69\times$) for all $\tau \ge 0.4$. This confirms that the wall-clock gains of the asynchronous pipeline are preserved across a wide range of $\tau$ values.
In contrast, as $\tau$ decreases, training quality degrades significantly, with asynchronous training affected more strongly.
\begin{figure*}[t]
    \centering
    \vspace{-0.1cm}
    \subcaptionbox{}{\includegraphics[width=0.23\linewidth]{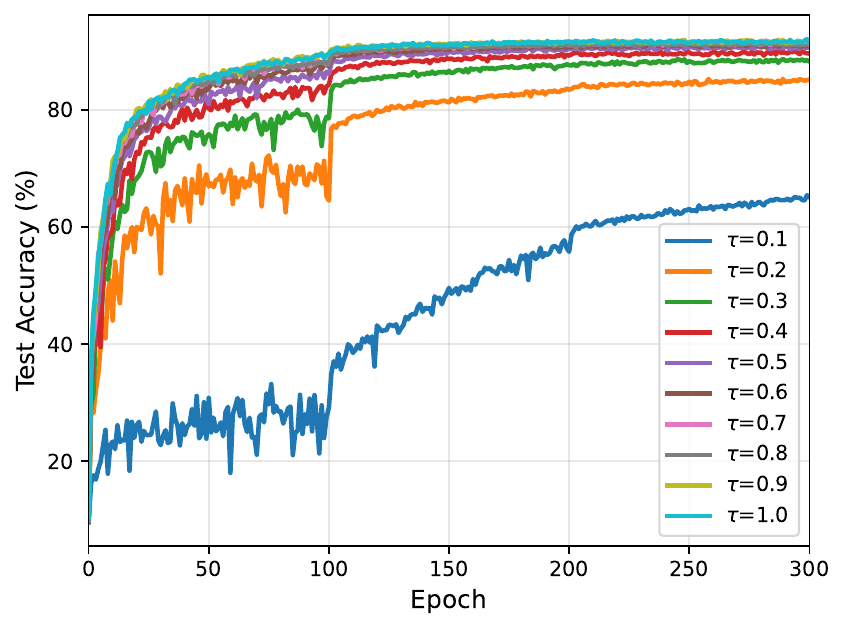}\vspace{-0.7em}}
    \subcaptionbox{}{\includegraphics[width=0.23\linewidth]{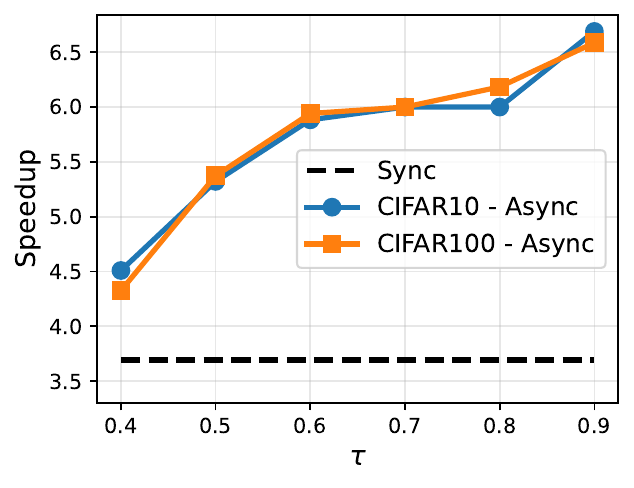}\vspace{-0.7em}}
    \subcaptionbox{}{\includegraphics[width=0.23\linewidth]{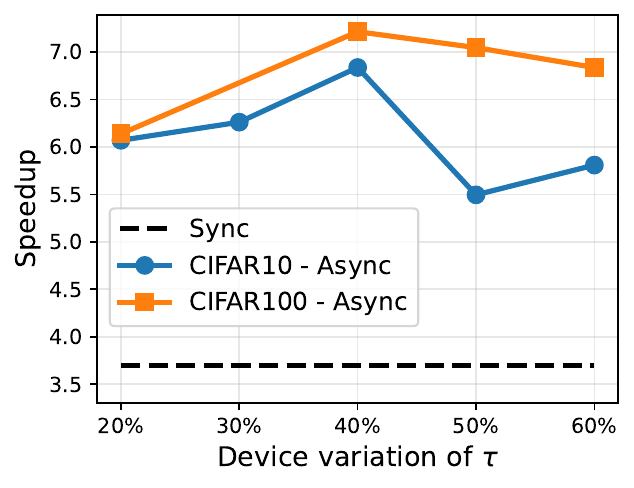}\vspace{-0.7em}}
    \subcaptionbox{}{\includegraphics[width=0.28\linewidth]{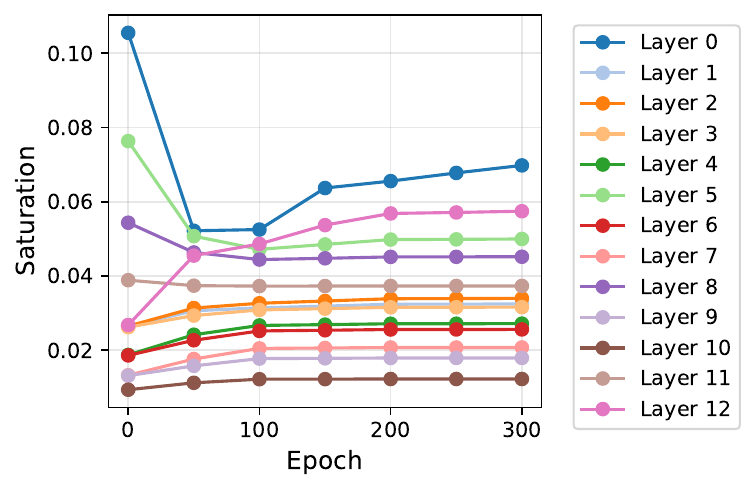}\vspace{-0.7em}}
    \vspace{-0.1cm}
    \caption{
        {\AnalogSGDAP} for CIFAR10 with $B_{\text{micro}}=16$ and delay $=0$.
        \textbf{(a)} Test accuracy versus $\tau$.
        \textbf{(b)} Speedup versus $\tau$.
        \textbf{(c)} Speedup versus device-to-device variation.
        \textbf{(d)} Saturation degree evolution over the trajectory.}
        \label{fig:tau-convergence-speedup-cifar10-dtod-saturation-combined}
\end{figure*}

\noindent
\textbf{Effect of device variation on speedup.}
We next evaluate whether device-to-device variation changes the acceleration trend; see 
Fig. \ref{fig:tau-convergence-speedup-cifar10-dtod-saturation-combined}c.
In the simulation, we set $\tau=1.0$, $B_{\text{micro}}=16$, and delay $=0$. 
Across the typical variation levels tested here, the asynchronous pipeline speedup changes only mildly, suggesting that device variation is not the primary limiter of pipeline acceleration in this regime.

\noindent
\textbf{Saturation degree evolution along training trajectory.}
We examine the saturation degree over the full training trajectory to justify the bounded saturation condition; see Fig. 
\ref{fig:tau-convergence-speedup-cifar10-dtod-saturation-combined}d. 
The saturation degree of $m$-th layer is defined as $\|W_k^\spp{m}\|_\infty / \tau$.
To ensure the formulation is well-defined, we set the device variation to 0.
The results show that the saturation remains nearly constant throughout training, with variations less than 10\%.
This observation is consistent with existing results, such as \cite{zou2019improved}, which demonstrate that the trajectory is bounded in a neighborhood of the initial point. 
It provides empirical evidence for the Assumption \ref{assumption:W-bounded} in the studied configuration; a general theoretical characterization is left for future work.

\begin{table*}[t]
		\vspace{-0.2cm}
		\centering
\end{table*}
\section{Conclusions and Limitations} 
\label{sec:conclusion_limit}
This paper theoretically investigates the pipeline algorithms for AIMC accelerators with respect to convergence and computation complexity. 
To unlock the full potential of pipeline training, asynchronous pipeline discards the synchronization at the end of one mini-batch to maximize the accelerator usage. We theoretically demonstrate that {\AnalogSGDAP} achieves maximal computational density with only a slight sacrifice in sample complexity through higher-order terms, which brings extraordinary speedup on wall-clock complexity.
Simulations validate our theoretical findings. 
A limitation of this work is the lack of real hardware validation, but this is a necessary step for every new technology and is widely used in current research regarding the algorithm development for AIMC accelerators. Future work should verify its efficiency as the technology becomes mature. 
\balance
\bibliographystyle{IEEEtran}
\bibliography{
    bibs/abrv,
    bibs/optimization,
    bibs/analog,
    bibs/textbook,
    bibs/LLM,
    bibs/training,
    bibs/publication,
    bibs/DL-theory
}

@inproceedings{oymak2019overparameterized,
  title={Overparameterized nonlinear learning: Gradient descent takes the shortest path?},
  author={Oymak, Samet and Soltanolkotabi, Mahdi},
  booktitle={International Conference on Machine Learning},
  pages={4951--4960},
  year={2019},
  organization={PMLR}
}

@article{liu2022loss,
  title={Loss landscapes and optimization in over-parameterized non-linear systems and neural networks},
  author={Liu, Chaoyue and Zhu, Libin and Belkin, Mikhail},
  journal={Applied and Computational Harmonic Analysis},
  volume={59},
  pages={85--116},
  year={2022}
}

@article{zou2019improved,
  title={An improved analysis of training over-parameterized deep neural networks},
  author={Zou, Difan and Gu, Quanquan},
  journal={Advances in neural information processing systems},
  volume={32},
  year={2019}
}

@inproceedings{xu2020acceleration,
  title={On the acceleration of deep learning model parallelism with staleness},
  author={Xu, An and Huo, Zhouyuan and Huang, Heng},
  booktitle=CVPR,
  pages={2088--2097},
  year={2020}
}

@article{wu2019wider,
  title={Wider or deeper: Revisiting the resnet model for visual recognition},
  author={Wu, Zifeng and Shen, Chunhua and Van Den Hengel, Anton},
  journal={Pattern recognition},
  volume={90},
  pages={119--133},
  year={2019}
}

@article{touvron2023llama2,
  title={Llama 2: Open foundation and fine-tuned chat models},
  author={Touvron, Hugo and Martin, Louis and Stone, Kevin and Albert, Peter and Almahairi, Amjad and Babaei, Yasmine and Bashlykov, Nikolay and Batra, Soumya and Bhargava, Prajjwal and Bhosale, Shruti and others},
  journal={arXiv preprint arXiv:2307.09288},
  year={2023}
}

@article{huang2019gpipe,
  title={{Gpipe}: Efficient training of giant neural networks using pipeline parallelism},
  author={Huang, Yanping and Cheng, Youlong and Bapna, Ankur and Firat, Orhan and Chen, Dehao and Chen, Mia and Lee, HyoukJoong and Ngiam, Jiquan and Le, Quoc V and Wu, Yonghui and others},
  journal={Advances in neural information processing systems},
  volume={32},
  year={2019}
}

@inproceedings{ren2021zero,
  title={Zero-offload: Democratizing billion-scale model training},
  author={Ren, Jie and Rajbhandari, Samyam and Aminabadi, Reza Yazdani and Ruwase, Olatunji and Yang, Shuangyan and Zhang, Minjia and Li, Dong and He, Yuxiong},
  booktitle={USENIX Annual Technical Conference},
  pages={551--564},
  year={2021}
}

@article{goyal2017accurate,
  title={Accurate, large minibatch {SGD}: Training {Imagenet} in 1 hour},
  author={Goyal, Priya and Doll{\'a}r, Piotr and Girshick, Ross and Noordhuis, Pieter and Wesolowski, Lukasz and Kyrola, Aapo and Tulloch, Andrew and Jia, Yangqing and He, Kaiming},
  journal={arXiv preprint arXiv:1706.02677},
  year={2017}
}

@article{zhao2023pytorch,
  title={Pytorch {FSDP}: experiences on scaling fully sharded data parallel},
  author={Zhao, Yanli and Gu, Andrew and Varma, Rohan and Luo, Liang and Huang, Chien-Chin and Xu, Min and Wright, Less and Shojanazeri, Hamid and Ott, Myle and Shleifer, Sam and others},
  journal={arXiv preprint arXiv:2304.11277},
  year={2023}
}

@inproceedings{you2020Large,
    title={Large Batch Optimization for Deep Learning: Training {BERT} in 76 minutes},
    author={Yang You and Jing Li and Sashank Reddi and Jonathan Hseu and Sanjiv Kumar and Srinadh Bhojanapalli and Xiaodan Song and James Demmel and Kurt Keutzer and Cho-Jui Hsieh},
    booktitle={International Conference on Learning Representations},
    year={2020}
}

@string{CVPR    = "IEEE Conference on Computer Vision and Pattern Recognition"}

@string{arxiv = "arXiv preprint arXiv: "}

@article{gokmen2016acceleration,
  title={Acceleration of deep neural network training with resistive cross-point devices: Design considerations},
  author={Gokmen, Tayfun and Vlasov, Yurii},
  journal={Frontiers in neuroscience},
  volume={10},
  pages={333},
  year={2016}
}

@ARTICLE{gokmen2021,
    AUTHOR={Gokmen, Tayfun},
    TITLE={Enabling Training of Neural Networks on Noisy Hardware},
    JOURNAL={Frontiers in Artificial Intelligence},
    VOLUME={4},
    YEAR={2021},
    pages = {1--14},
}

@article{gokmen2017cnn,
  title={Training Deep Convolutional Neural Networks with Resistive Cross-Point Devices},
  author={Gokmen, Tayfun and Onen, Murat and Haensch, Wilfried},
  journal={Frontiers in neuroscience},
  volume={11},
  pages={538},
  year={2017},
  publisher={Frontiers}
}

@article{gokmen2020,
  title={Algorithm for Training Neural Networks on Resistive Device Arrays},
  author={Gokmen, Tayfun and Haensch, Wilfried},
  journal={Frontiers in Neuroscience},
  volume={14},
  year={2020},
  publisher={Frontiers Media SA}
}

@article{jain2019neural,
  title={Neural network accelerator design with resistive crossbars: Opportunities and challenges},
  author={Jain, Shubham and others},
  journal={IBM Journal of Research and Development},
  volume=63,
  number=6,
  pages={10--1},
  year=2019,
  publisher={IBM}
}

@article{Rasch2021AFA,
    title={A Flexible and Fast {PyTorch} Toolkit for Simulating Training and Inference on Analog Crossbar Arrays},
    author={Rasch, Malte J and Moreda, Diego and Gokmen, Tayfun and Le Gallo, Manuel and Carta, Fabio and Goldberg, Cindy and El Maghraoui, Kaoutar and Sebastian, Abu and Narayanan, Vijay},
    journal={IEEE International Conference on Artificial Intelligence Circuits and Systems},
    year={2021},
    pages={1-4}
}

@article{onen2022neural,
  title={Neural Network Training with Asymmetric Crosspoint Elements},
  author={Onen, Murat and Gokmen, Tayfun and Todorov, Teodor K and Nowicki, Tomasz and Del Alamo, Jes{\'u}s A and Rozen, John and Haensch, Wilfried and Kim, Seyoung},
  journal={Frontiers in artificial intelligence},
  volume={5},
  year={2022},
  publisher={Frontiers Media SA}
}

@article{rasch2024fast,
  title={Fast and robust analog in-memory deep neural network training},
  author={Rasch, Malte J and Carta, Fabio and Fagbohungbe, Omobayode and Gokmen, Tayfun},
  journal={Nature Communications},
  volume={15},
  number={1},
  pages={7133--7147},
  year={2024},
  publisher={Nature Publishing Group UK London}
}

@article{burr2015experimental,
  title={Experimental demonstration and tolerancing of a large-scale neural network (165 000 synapses) using phase-change memory as the synaptic weight element},
  author={Burr, Geoffrey W and Shelby, Robert M and Sidler, Severin and Di Nolfo, Carmelo and Jang, Junwoo and Boybat, Irem and Shenoy, Rohit S and Narayanan, Pritish and Virwani, Kumar and Giacometti, Emanuele U and others},
  journal={IEEE Transactions on Electron Devices},
  volume={62},
  number={11},
  pages={3498--3507},
  year={2015},
  publisher={IEEE}
}

@article{shafiee2016isaac,
    title={{ISAAC}: A convolutional neural network accelerator with in-situ analog arithmetic in crossbars},
    author={Shafiee, Ali and Nag, Anirban and Muralimanohar, Naveen and Balasubramonian, Rajeev and Strachan, John Paul and Hu, Miao and Williams, R Stanley and Srikumar, Vivek},
    journal={ACM SIGARCH Computer Architecture News},
    volume={44},
    number={3},
    pages={14--26},
    year={2016},
    publisher={ACM}
}

@inproceedings{song2017pipelayer,
    title={Pipelayer: A pipelined {ReRAM}-based accelerator for deep learning},
    author={Song, Linghao and Qian, Xuehai and Li, Hai and Chen, Yiran},
    booktitle={IEEE international symposium on high performance computer architecture},
    pages={541--552},
    year={2017},
    organization={IEEE}
}

@article{yao2017face,
  title={Face classification using electronic synapses},
  author={Yao, Peng and Wu, Huaqiang and Gao, Bin and Eryilmaz, Sukru Burc and Huang, Xueyao and Zhang, Wenqiang and Zhang, Qingtian and Deng, Ning and Shi, Luping and Wong, H-S Philip and others},
  journal={Nature communications},
  volume={8},
  number={1},
  pages={15199},
  year={2017},
  publisher={Nature Publishing Group UK London}
}

@article{wang2019situ,
  title={In situ training of feed-forward and recurrent convolutional memristor networks},
  author={Wang, Zhongrui and Li, Can and Lin, Peng and Rao, Mingyi and Nie, Yongyang and Song, Wenhao and Qiu, Qinru and Li, Yunning and Yan, Peng and Strachan, John Paul and others},
  journal={Nature Machine Intelligence},
  volume={1},
  number={9},
  pages={434--442},
  year={2019},
  publisher={Nature Publishing Group UK London}
}

@article{bottou2018optimization,
    title={Optimization methods for large-scale machine learning},
    author={Bottou, L{\'e}on and Curtis, Frank E and Nocedal, Jorge},
    journal={SIAM review},
    volume={60},
    number={2},
    pages={223--311},
    year={2018},
    publisher={SIAM}
}

@article{sun2017asynchronous,
  title={Asynchronous coordinate descent under more realistic assumptions},
  author={Sun, Tao and Hannah, Robert and Yin, Wotao},
  journal={Advances in Neural Information Processing Systems},
  volume={30},
  year={2017}
}

@inproceedings{wu2024towards,
    title={Towards Exact Gradient-based Training on Analog In-memory Computing},
    author={Wu, Zhaoxian and Gokmen, Tayfun and Rasch, Malte J and Chen, Tianyi},
    booktitle={Advances in Neural Information Processing Systems},
    year={2024}
}

@inproceedings{wu2025analog,
    title={Analog In-memory Training on General Non-ideal Resistive Elements: The Impact of Response Functions},
    author={Wu, Zhaoxian and Xiao, Quan and Gokmen, Tayfun and Fagbohungbe, Omobayode and Chen, Tianyi},
    booktitle={Advances in Neural Information Processing Systems},
    year={2025}
}

@inproceedings{li2025memory,
  title={In-memory Training on Analog Devices with Limited Conductance States via Multi-tile Residual Learning},
  author={Li, Jindan and Wu, Zhaoxian and Liu, Gaowen and Gokmen, Tayfun and Chen, Tianyi},
  booktitle={International Conference on Artificial Intelligence and Statistics},
  year={2025}
}

@article{wu2025pipelineArxiv,
    title={On the Convergence Theory of Pipeline Gradient-based Analog In-memory Training},
    author={Wu, Zhaoxian and Xiao, Quan and Gokmen, Tayfun and Tsai, Hsinyu and Maghraoui, Kaoutar El and Chen, Tianyi},
    journal={arXiv preprint arXiv:2410.15155},
    year={2024}
}

@article{chen2022sapipe,
title={{SApipe}: Staleness-aware pipeline for data parallel {DNN} training},
author={Chen, Yangrui and Xie, Cong and Ma, Meng and Gu, Juncheng and Peng, Yanghua and Lin, Haibin and Wu, Chuan and Zhu, Yibo},
journal={Advances in Neural Information Processing Systems},
volume={35},
pages={17981--17993},
year={2022}
}

@inproceedings{huo2018decoupled,
title={Decoupled parallel backpropagation with convergence guarantee},
author={Huo, Zhouyuan and Gu, Bin and Huang, Heng and others},
booktitle={International Conference on Machine Learning},
pages={2098--2106},
year={2018},
organization={PMLR}
}

@article{kosson2021pipelined,
title={Pipelined backpropagation at scale: training large models without batches},
author={Kosson, Atli and Chiley, Vitaliy and Venigalla, Abhinav and Hestness, Joel and Koster, Urs},
journal={Proceedings of Machine Learning and Systems},
volume={3},
pages={479--501},
year={2021}
}

@inproceedings{narayanan2019pipedream,
title={PipeDream: generalized pipeline parallelism for {DNN} training},
author={Narayanan, Deepak and Harlap, Aaron and Phanishayee, Amar and Seshadri, Vivek and Devanur, Nikhil R and Ganger, Gregory R and Gibbons, Phillip B and Zaharia, Matei},
booktitle={Proceedings of the 27th ACM symposium on operating systems principles},
pages={1--15},
year={2019}
}

@article{li2020pytorch,
title={Pytorch distributed: Experiences on accelerating data parallel training},
author={Li, Shen and Zhao, Yanli and Varma, Rohan and Salpekar, Omkar and Noordhuis, Pieter and Li, Teng and Paszke, Adam and Smith, Jeff and Vaughan, Brian and Damania, Pritam and others},
journal={arXiv preprint arXiv:2006.15704},
year={2020}
}

@article{li2018pipe,
title={Pipe-{SGD}: A decentralized pipelined {SGD} framework for distributed deep net training},
author={Li, Youjie and Yu, Mingchao and Li, Songze and Avestimehr, Salman and Kim, Nam Sung and Schwing, Alexander},
journal={Advances in Neural Information Processing Systems},
volume={31},
year={2018}
}

@inproceedings{xu2021parallelizing,
title={Parallelizing {DNN} training on {GPU}s: Challenges and opportunities},
author={Xu, Weizheng and Zhang, Youtao and Tang, Xulong},
booktitle={Companion Proceedings of the Web Conference},
pages={174--178},
year={2021}
}

@article{shoeybi2019megatron,
title={Megatron-lm: Training multi-billion parameter language models using model parallelism},
author={Shoeybi, Mohammad and Patwary, Mostofa and Puri, Raul and LeGresley, Patrick and Casper, Jared and Catanzaro, Bryan},
journal={arXiv preprint arXiv:1909.08053},
year={2019}
}

@inproceedings{li2023colossal,
title={Colossal-{AI}: A unified deep learning system for large-scale parallel training},
author={Li, Shenggui and Liu, Hongxin and Bian, Zhengda and Fang, Jiarui and Huang, Haichen and Liu, Yuliang and Wang, Boxiang and You, Yang},
booktitle={International Conference on Parallel Processing},
pages={766--775},
year={2023}
}

@article{kim2020torchgpipe,
title={torchgpipe: On-the-fly pipeline parallelism for training giant models},
author={Kim, Chiheon and Lee, Heungsub and Jeong, Myungryong and Baek, Woonhyuk and Yoon, Boogeon and Kim, Ildoo and Lim, Sungbin and Kim, Sungwoong},
journal={arXiv preprint arXiv:2004.09910},
year={2020}
}

@article{cubuk2018autoaugment,
title={Autoaugment: Learning augmentation policies from data},
author={Cubuk, Ekin D and Zoph, Barret and Mane, Dandelion and Vasudevan, Vijay and Le, Quoc V},
journal={arXiv preprint arXiv:1805.09501},
year={2018}
}

\arxivonly{
    
%%%%%%%%%%%%%%%%%%%%%%%%%%%%%%%%%%%%%%%%%%%%%%%%%%%%%%%%%%%%%%%%%%%%%%%%%%%%%%%
%%%%%%%%%%%%%%%%%%%%%%%%%%%%%%%%%%%%%%%%%%%%%%%%%%%%%%%%%%%%%%%%%%%%%%%%%%%%%%%
% APPENDIX
%%%%%%%%%%%%%%%%%%%%%%%%%%%%%%%%%%%%%%%%%%%%%%%%%%%%%%%%%%%%%%%%%%%%%%%%%%%%%%%
%%%%%%%%%%%%%%%%%%%%%%%%%%%%%%%%%%%%%%%%%%%%%%%%%%%%%%%%%%%%%%%%%%%%%%%%%%%%%%%
\onecolumn
% \newpage
% \appendix
\appendices

\begin{center}
\Large \textbf{Supplementary} \\
\end{center}

\section{Analog SGD with Asynchronous Pipeline in Different Perspectives}
\label{section:asyn-pipeline-details}
Involving multiple accelerators and data at each time step, pipeline algorithms, especially the asynchronous pipeline, are usually hard to figure out. To facilitate the understanding, we list the Analog SGD under two different perspectives. In \emph{accelerator perspective}, we pay more attention to how accelerator $m$ processes a flow of data; see Algorithm \ref{algorithm:analog-pipeline-accelerator}. In \emph{data perspective}, we focus on how a sample $\xi_k$ flows through each accelerator; see Algorithm \ref{algorithm:analog-pipeline-data}. 
To clarify the relationship, the computational dependency is plotted in Fig. \ref{figure:asyn-pipeline-formulation}.
\begin{algorithm}[H]
  % \begin{algorithm}[t]
  \caption{Analog SGD with asynchronous pipeline (accelerator Perspective)}
  \label{algorithm:analog-pipeline-accelerator}
  % \textbf{Require: } $\GP$: Analog SGD \eqref{dynamic:analog-SGD} or Tiki-Taka (\eqref{recursion:HD-P} and \eqref{recursion:HD-W})
  \begin{algorithmic}[1]
    % \Require learning rate $\alpha$; initialization 
    % \Statex // accelerator Perspective
    \For {iteration $k = 0, 1, 2, \cdots, K$}
    \Statex // Periphery device at the beginning
    \State Sample $\xi_k = (x_k, y_k)$ and Send $x^\spp{0}_k = x_k$ to accelerator $1$
    \Statex // Accelerator $m=1, 2, \cdots, M$ in parallel
    \State Receive input $\tdx^\spp{m}_k$ from accelerator $m-1$
    \State Compute $\tdz^\spp{m}_k = W^\spp{m}_{k-(M-m)} \tdx^\spp{m}_k$
    \State Compute $\tdx^\spp{m+1}_k = g^\spp{m}(\tdz^\spp{m}_k)$ and $\nabla g^\spp{m}(\tdz^\spp{m}_k)$
    \State Store input $\tdx^\spp{m}_k$ and $\nabla g^\spp{m}(\tdz^\spp{m}_k)$
    \State Receive backward signal $\tddelta^\spp{m+1}_{k-(M-m)}$
    % \State $\delta^\spp{m}_{k-(M-m)} = (\delta^\spp{m+1}_{k-(M-m)} W^\spp{m}_{k-(M-m)}) \nabla g^\spp{m}(z^\spp{m}_{k-(M-m)})$
    \State 
    $\tddelta^\spp{m}_k = \tddelta^\spp{m+1}_kW^\spp{m+1}_k \nabla g^\spp{m}(\tdz^\spp{m}_k)$
    % $x^\spp{m}_k = \BP^\spp{m}(W^\spp{m}_{k-(M-m)}\ANALOG, \delta^\spp{m+1}_{k-(M-m)}; z^\spp{m}_{k-(M-m)})$
    % \State $W_{k-(M-m)+1}^\spp{m}\ANALOG =W_{k-(M-m)}^\spp{m}\ANALOG - \alpha\delta^\spp{m}_{k-(M-m)}\otimes x^\spp{m}_{k-(M-m)}$
    \State 
    $W_{k+1}^\spp{m} =  W_k^\spp{m} - \alpha\tddelta^\spp{m}_k \otimes \tdx^\spp{m}_k
     - \frac{\alpha}{\tau}|\tddelta^\spp{m}_k \otimes \tdx^\spp{m}_k| \odot W_k^\spp{m}$
    % $W_{k+1}^\spp{m}\ANALOG = \GP(W_k^\spp{m}\ANALOG, \delta^\spp{m}_k, x^\spp{m}_k)$
      \Statex // Periphery device at the end
    \State Compute $\tddelta^\spp{M}_k = \nabla_{\tdx^\spp{M+1}}\ell(\tdx^\spp{M+1}, y_k) \nabla g^\spp{M}(\tdz^\spp{M})$ 
    \State Send $\tddelta^\spp{M}_k$ back to accelerator $M$
    \EndFor
  \end{algorithmic}
\end{algorithm}

\begin{algorithm}[H]
  % \begin{algorithm}[t]
  \caption{Analog SGD with asynchronous pipeline (Data Perspective)}
  \label{algorithm:analog-pipeline-data}
  % \textbf{Require: } $\GP$: Analog SGD \eqref{dynamic:analog-SGD} or Tiki-Taka (\eqref{recursion:HD-P} and \eqref{recursion:HD-W})
  \begin{algorithmic}[1]
      % \Require learning rate $\alpha$; initialization 
      % \State 1
      % \Comment{Data Perspective}
      % \Statex // Data Perspective
      \For {iteration $k = 0, 1, 2, \cdots, K$}
      \State Sample $\xi_k = (x_k, y_k)$, $x^\spp{0}_k = x_k$
      \ForAll {accelerator $m=1, 2, \cdots, M$}
      \State Compute $\tdz^\spp{m}_k = W^\spp{m}_{k-(M-m)} \tdx^\spp{m}_k$
      \State Compute $\tdx^\spp{m+1}_k = g^\spp{m}(\tdz^\spp{m}_k)$ and $\nabla g^\spp{m}(\tdz^\spp{m}_k)$
      % \State Store input $x^\spp{m}_k$ and $\nabla g^\spp{m}(z^\spp{m}_k)$
      \EndFor
      % \State $\delta^\spp{M}_k = \nabla_{x^\spp{M+1}_k}\ell(x^\spp{M+1}_k, y_k) \nabla g^\spp{M}(z^\spp{M}_k)$
      \State Compute $\tddelta^\spp{M}_k = \nabla_{\tdx^\spp{M+1}}\ell(\tdx^\spp{M+1}, y_k) \nabla g^\spp{M}(\tdz^\spp{M})$ 
      \ForAll {layer $m=M, M-1, \cdots, 1$}
      % \State $\delta^\spp{m}_k = (\delta^\spp{m+1}_kW^\spp{m}_k) \nabla g^\spp{m}(z^\spp{m}_k)$
      \State $\tddelta^\spp{m}_k = \tddelta^\spp{m+1}_kW^\spp{m+1}_k \nabla g^\spp{m}(\tdz^\spp{m}_k)$
      % \State $W_{k+1}^\spp{m}\ANALOG =W_k^\spp{m}\ANALOG - \alpha\delta^\spp{m}_k\otimes x^\spp{m}_k$
      \State $W_{k+1}^\spp{m} =  W_k^\spp{m} - \alpha\tddelta^\spp{m}_k \otimes \tdx^\spp{m}_k
     - \frac{\alpha}{\tau}|\tddelta^\spp{m}_k \otimes \tdx^\spp{m}_k| \odot W_k^\spp{m}$
      \EndFor
      \EndFor
    \end{algorithmic}
\end{algorithm}

\begin{figure*}[!t]
  \centering
  \includegraphics[width=0.9\linewidth]{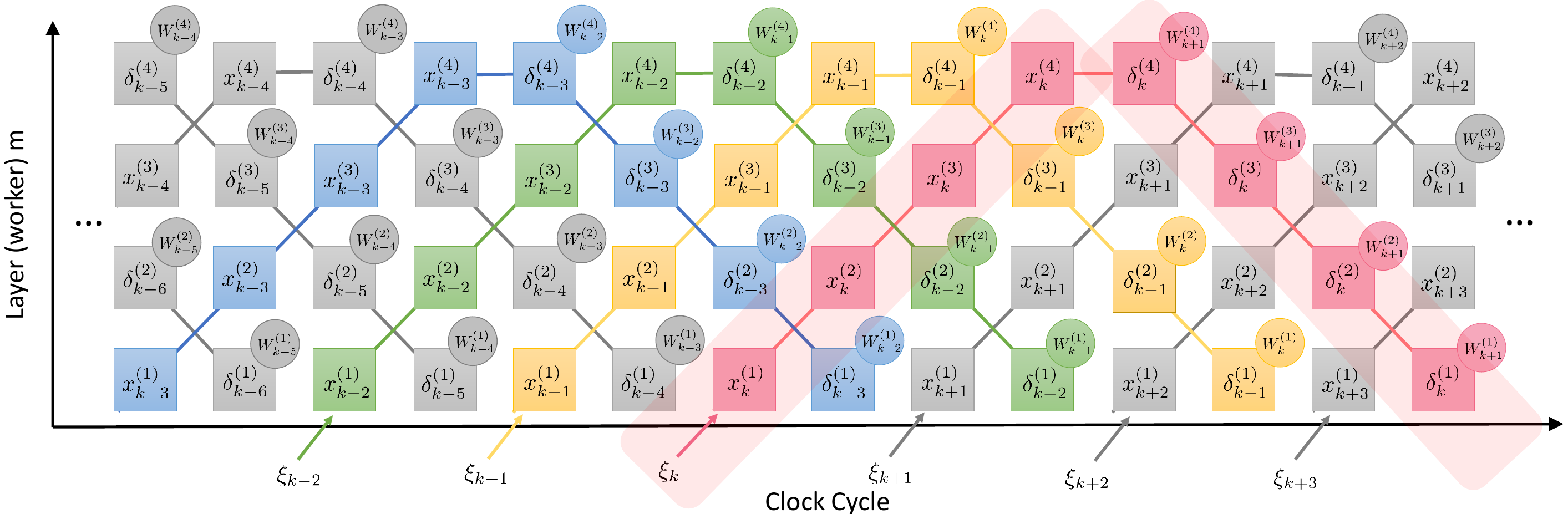}
  \caption{Illustration of the dynamics of the asynchronous pipeline. The $W_k^\spp{m}$ in the circle implies the update happens in this clock cycle, and the symbols in the squares indicate the input of each accelerator.}
  \label{figure:asyn-pipeline-formulation}
\vspace{-0.4cm}
\end{figure*}

\section{Useful Lemmas and their proof}
We first bound the error of the stale variables. 
Define $\Delta^\spp{m}_{k} := W^\spp{m}_{k+1}-W^\spp{m}_k$, which is the different between consecutive weights. 
\begin{lemma}[Delay error of signals, restatement of Lemma \ref{lemma:signal-delay-short}]
    \label{lemma:signal-delay}
    Under Assumption \ref{assumption:normalized-feature}--\ref{assumption:loss-Lip}, the difference between the latest and stale signals generated by {\AnalogSGDAP} are bounded, i.e., $\forall m\in[M+1]$
    \begin{align}
        \|\tdx^\spp{m}_k-x^\spp{m}_k\| \le&\ C_x^\spp{m} \sum_{m'=1}^{m-1}\sum_{k'=k-(M-m')}^{k-1} \|\Delta^\spp{m'}_{k'}\|,\\ 
        \|\tdz^\spp{m}_k-z^\spp{m}_k\| \le&\ C_z^\spp{m} \sum_{m'=1}^{m}\sum_{k'=k-(M-m')}^{k-1} \|\Delta^\spp{m'}_{k'}\|,\\ 
        \|\tddelta^\spp{m}_k-\delta^\spp{m}_k\| \le&\ C_{\delta}^\spp{m}\sum_{m'=1}^{M}\sum_{k'=k-(M-m')}^{k-1} \|\Delta^\spp{m'}_{k'}\|,
    \end{align}
    where the constants above are defined by
    \begin{align}
        \label{definition:C_x}
        C_x^\spp{m} :=&\ L_{g,0}(L_{g,0}W_{\max,2})^{m-1}, \\
        \label{definition:C_z}
        C_z^\spp{m} :=&\ (L_{g,0}W_{\max,2})^{m-1}, \\
        \label{definition:C_delta}
        C_{\delta}^\spp{m} :=&\ (L_{g,0}W_{\max,2})^{M-m+1}(L_{\ell, 1}L_{g, 0}^2+L_{\ell, 0}L_{g, 1})
        (L_{g,0}W_{\max,2})^{M-1} \\
        &\ +(L_{g,0}W_{\max,2})^{M-1}L_{\ell,0}L_{g,1} \frac{1-(L_{g,0}W_{\max,2})^{M-m}}{1-L_{g,0}W_{\max,2}}.
        \nonumber
    \end{align}
\end{lemma}

\begin{proof}[Proof of Lemma \ref{lemma:signal-delay-short} and Lemma \ref{lemma:signal-delay}]
    We first bound the delay of the forward signal $\tdx^\spp{m}_k$. At the first layer, there is no delay, i.e. $\tdx^\spp{1}_k = x^\spp{1}_k = x_k$. For other layers $m>1$
    \begin{align}
        \label{inequality:signal-delay-forward-1}
        &\ \|\tdx^\spp{m+1}_k-x^\spp{m+1}_k\| 
        \le L_{g,0}\|\tdz^\spp{m}_k-z^\spp{m}_k\| \\
        % ===================================
        =&\ L_{g,0}\|W^\spp{m}_{k-(M-m)}\tdx^\spp{m}_k-W^\spp{m}_kx^\spp{m}_k\|
        \nonumber \\
        % ===================================
        =&\ L_{g,0}\|(W^\spp{m}_{k-(M-m)}-W^\spp{m}_k)\tdx^\spp{m}_k+W^\spp{m}_k(\tdx^\spp{m}_k-x^\spp{m}_k)\|
        \nonumber \\
        % ===================================
        \overset{(a)}{\le}&\ L_{g,0}(\|W^\spp{m}_{k-(M-m)}-W^\spp{m}_k\|_2\|\tdx^\spp{m}_k\|+\|W^\spp{m}_k\|_2\|\tdx^\spp{m}_k-x^\spp{m}_k\|)
        \nonumber \\
        % ===================================
        \overset{(b)}{\le}&\ L_{g,0}((L_{g,0}W_{\max,2})^m\|W^\spp{m}_{k-(M-m)}-W^\spp{m}_k\|+W_{\max,2}\|\tdx^\spp{m}_k-x^\spp{m}_k\|)
        \nonumber 
    \end{align}
    where (a) follows the triangle inequality and the consistency of $\ell_2$-matrix norm, i.e. $\|Wu\|\le\|W\|_2\|u\|$ where $W$ and $u$ are arbitrary matrices and vectors, respectively; and (b) is because of Lemma \ref{lemma:bounded-signal}. 
    % \todo{l2 norm of W}
    Using triangle inequality over the first term of \ac{RHS} in \eqref{inequality:signal-delay-forward-1}, we have
    \begin{align}
        \|W^\spp{m}_{k-(M-m)}-W^\spp{m}_k\|
        \le \sum_{k'=k-(M-m)}^{k-1}\|W^\spp{m}_{k'+1}-W^\spp{m}_{k'}\|
        =\sum_{k'=k-(M-m)}^{k-1} \|\Delta^\spp{m}_{k'}\|
    \end{align}
    and consequently,
    \begin{align}
        \label{inequality:signal-delay-forward-2}
        \|\tdx^\spp{m+1}_k-x^\spp{m+1}_k\|
        \le&\ L_{g,0}(L_{g,0}W_{\max,2})^m\sum_{k'=k-(M-m)}^{k-1} \|\Delta^\spp{m}_{k'}\|
        +L_{g,0}W_{\max,2}\|\tdx^\spp{m}_k-x^\spp{m}_k\|.
        % \nonumber
    \end{align}
    Expanding inequality \eqref{inequality:signal-delay-forward-2} from $m$ to $1$ and using the fact $\tdx^\spp{1}_k = x^\spp{1}_k = x_k$ yield
    \begin{align}
        \label{inequality:signal-delay-forward-3}
        &\ 
        \|\tdx^\spp{m}_k-x^\spp{m}_k\|
        \\
        % ===================================
        \le&\ L_{g,0}(L_{g,0}W_{\max,2})^{m-1}\sum_{k'=k-(M-m+1)}^{k-1} \|\Delta^\spp{m-1}_{k'}\|
        +L_{g,0}W_{\max,2}\|\tdx^\spp{m-1}_k-x^\spp{m-1}_k\|
        \nonumber
        \\
        % ===================================
        \le&\ L_{g,0}\sum_{m'=2}^{m}(L_{g,0}W_{\max,2})^{m'-1}(L_{g,0}W_{\max,2})^{m-m'}\sum_{k'=k-(M-m'+1)}^{k-1} \|\Delta^\spp{m'-1}_{k'}\|
        \nonumber\\
        % ===================================
        \le&\ L_{g,0}(L_{g,0}W_{\max,2})^{m-1}\sum_{m'=2}^{m}\sum_{k'=k-(M-m'+1)}^{k-1} \|\Delta^\spp{m'-1}_{k'}\|.
        \nonumber\\
        % ===================================
        =&\ L_{g,0}(L_{g,0}W_{\max,2})^{m-1}\sum_{m'=1}^{m-1}\sum_{k'=k-(M-m')}^{k-1} \|\Delta^\spp{m'}_{k'}\|.
        \nonumber
    \end{align}
    where the last inequality reindexes $m'$ in the first summation.
    
    Noticing that \eqref{inequality:signal-delay-forward-1}-\eqref{inequality:signal-delay-forward-3} also implies that
    \begin{align}
        \label{inequality:signal-delay-forward-z-1}
        \|z^\spp{m}_k-\tdz^\spp{m}_k\|\le&\ (L_{g,0}W_{\max,2})^{m-1}\sum_{m'=1}^{m-1}\sum_{k'=k-(M-m')}^{k-1} \|\Delta^\spp{m'}_{k'}\|.
    \end{align}
    We then bound the delay of the backward signal $\tddelta^\spp{m}_k$ for $m\in[M-1]$
    \begin{align}
        \label{inequality:signal-delay-backward-1}
        &\ \|\tddelta^\spp{m}_k-\delta^\spp{m}_k\|
        =\|\delta^\spp{m+1}_k W^\spp{m+1}_k \nabla g^\spp{m}(z^\spp{m}_k)-\tddelta^\spp{m+1}_k W^\spp{m+1}_k \nabla g^\spp{m}(\tdz^\spp{m}_k)\|\\
        % ===================================
        \le&\ \|\delta^\spp{m+1}_k-\tddelta^\spp{m+1}_k\| \|W^\spp{m+1}_k\| \|\nabla g^\spp{m}(z^\spp{m}_k)\|
        \nonumber\\
        &\ +\|\tddelta^\spp{m+1}_k\| \|W^\spp{m+1}_k\| \|\nabla g^\spp{m}(z^\spp{m}_k)-\nabla g^\spp{m}(\tdz^\spp{m}_k)\|
        \nonumber\\
        % ===================================
        \le&\ L_{g,0}W_{\max,2}\|\delta^\spp{m+1}_k-\tddelta^\spp{m+1}_k\| 
        +W_{\max,2}(L_{g,0}W_{\max,2})^{M-m-1}L_{\ell,0}L_{g,0}L_{g,1} \|z^\spp{m}_k-\tdz^\spp{m}_k\|
        \nonumber
    \end{align}
    where the last inequality comes from Assumption \ref{assumption:g-Lip} (1)-(2) and Lemma \ref{lemma:bounded-signal}. Substituting \eqref{inequality:signal-delay-forward-z-1} into \eqref{inequality:signal-delay-backward-1}, 
    we can further bound \eqref{inequality:signal-delay-backward-1} by
    \begin{align}
        \label{inequality:signal-delay-backward-2}
        &\ 
        \|\tddelta^\spp{m}_k-\delta^\spp{m}_k\|
        \\
        % ===================================
        \le&\ L_{g,0}W_{\max,2}\|\delta^\spp{m+1}_k-\tddelta^\spp{m+1}_k\| 
        \bkeq
        +W_{\max,2}(L_{g,0}W_{\max,2})^{M-m-1}L_{\ell,0}L_{g,0}L_{g,1} (L_{g,0}W_{\max,2})^{m-1}\sum_{m'=1}^{m-1}\sum_{k'=k-(M-m')}^{k-1} \|\Delta^\spp{m'}_{k'}\|
        \nonumber\\
        % ===================================
        =&\ L_{g,0}W_{\max,2}\|\delta^\spp{m+1}_k-\tddelta^\spp{m+1}_k\| 
        +(L_{g,0}W_{\max,2})^{M-1}L_{\ell,0}L_{g,1} \sum_{m'=1}^{m-1}\sum_{k'=k-(M-m')}^{k-1} \|\Delta^\spp{m'}_{k'}\|
        \nonumber\\
        % ===================================
        \le &\ (L_{g,0}W_{\max,2})^{M-m}\|\delta^\spp{M}_k-\tddelta^\spp{M}_k\|
        +(L_{g,0}W_{\max,2})^{M-1}L_{\ell,0}L_{g,1} \times
        \nonumber\\
        &\ \sum_{m''=m}^{M-1}(L_{g,0}W_{\max,2})^{M-1-m''}\sum_{m'=1}^{m''-1}\sum_{k'=k-(M-m')}^{k-1} \|\Delta^\spp{m'}_{k'}\|.
        \nonumber
    \end{align}
    
    To bound the first term in the \ac{RHS} of \eqref{inequality:signal-delay-backward-2}, we manipulate the different $\|\delta^\spp{M}_k-\tddelta^\spp{M}_k\|$ by
    \begin{align}
        \label{inequality:signal-delay-backward-2-1}
        &\ 
        \|\delta^\spp{M}_k-\tddelta^\spp{M}_k\| 
        \\
        % =================================
        =&\ \|\nabla_{x^\spp{M+1}}\ell(x^\spp{M+1}_k, y_k) \nabla g^\spp{M}(z^\spp{M}_k) - \nabla_{x^\spp{M+1}}\ell(\tdx^\spp{M+1}_k, y_k) \nabla g^\spp{M}(\tdz^\spp{M}_k)\|
        \nonumber 
        \\
        % =================================
        \le&\ \|\nabla_{x^\spp{M+1}}\ell(x^\spp{M+1}_k, y_k)- \nabla_{x^\spp{M+1}}\ell(\tdx^\spp{M+1}_k, y_k)\|
        \|\nabla g^\spp{M}(z^\spp{M}_k)\| 
        \nonumber \\
        &\ +\|\nabla_{x^\spp{M+1}}\ell(\tdx^\spp{M+1}_k, y_k)\|
        \|\nabla g^\spp{M}(\tdz^\spp{M}_k)-\nabla g^\spp{M}(z^\spp{M}_k)\|
        \nonumber \\
        % =================================
        \overset{(a)}{\le}&\ L_{\ell, 1}L_{g, 0} \|x^\spp{M+1}_k - \tdx^\spp{M+1}_k\|
        +L_{\ell, 0}L_{g, 1}
        \|\tdz^\spp{M}_k-z^\spp{M}_k\|
        \nonumber \\
        % =================================
        \le&\ (L_{\ell, 1}L_{g, 0}^2+L_{\ell, 0}L_{g, 1})
        \|\tdz^\spp{M}_k-z^\spp{M}_k\|
        \nonumber \\
        % =================================
        \le&\ (L_{\ell, 1}L_{g, 0}^2+L_{\ell, 0}L_{g, 1})
        (L_{g,0}W_{\max,2})^{M-1}\sum_{m'=1}^{M}\sum_{k'=k-(M-m')}^{k-1} \|\Delta^\spp{m'}_{k'}\|.
        \nonumber
    \end{align}
    where (a) comes from Assumption \ref{assumption:g-Lip} and Assumption \ref{assumption:loss-Lip}.
    
    To bound the second term in the \ac{RHS} of \eqref{inequality:signal-delay-backward-2}, we change the upper bound of the second summation from $m''$ to $M$ and get
    \begin{align}
        \label{inequality:signal-delay-backward-2-2}
        &\ \sum_{m''=m}^{M-1}(L_{g,0}W_{\max,2})^{M-1-m''}\sum_{m'=1}^{m''-1}\sum_{k'=k-(M-m')}^{k-1} \|\Delta^\spp{m'}_{k'}\| \\
        % =================================
        \le&\ \sum_{m''=m}^{M-1}(L_{g,0}W_{\max,2})^{M-1-m''}\sum_{m'=1}^{M}\sum_{k'=k-(M-m')}^{k-1} \|\Delta^\spp{m'}_{k'}\| 
        \nonumber\\
        % =================================
        =&\ \frac{1-(L_{g,0}W_{\max,2})^{M-m}}{1-L_{g,0}W_{\max,2}}\sum_{m'=1}^{M}\sum_{k'=k-(M-m')}^{k-1} \|\Delta^\spp{m'}_{k'}\|.
        \nonumber
    \end{align}

    Substituting \eqref{inequality:signal-delay-backward-2-1} and \eqref{inequality:signal-delay-backward-2-2} back into \eqref{inequality:signal-delay-backward-2} we have
    \begin{align}
        \|\tddelta^\spp{m}_k-\delta^\spp{m}_k\|
        \le C_{\delta}^\spp{m}\sum_{m'=1}^{M}\sum_{k'=k-(M-m')}^{k-1} \|\Delta^\spp{m'}_{k'}\|
    \end{align}
    where the constant $C_{\delta}^\spp{m}$ is defined by
    \begin{align}
        C_{\delta}^\spp{m} :=&\ (L_{g,0}W_{\max,2})^{M-m+1}(L_{\ell, 1}L_{g, 0}^2+L_{\ell, 0}L_{g, 1})
        (L_{g,0}W_{\max,2})^{M-1} \\
        &\ +(L_{g,0}W_{\max,2})^{M-1}L_{\ell,0}L_{g,1} \frac{1-(L_{g,0}W_{\max,2})^{M-m}}{1-L_{g,0}W_{\max,2}}.
        \nonumber
    \end{align}
    This completes the proof.
\end{proof}

% \begin{lemma}[bounded gradient]
%     \label{lemma:bounded-gradient}
%     Under Assumption \ref{assumption:normalized-feature}--\ref{assumption:loss-Lip}, the gradient is bounded, i.e. $\forall W$ and $\xi$, 
%     \begin{align}
%         \|\nabla_{W^\spp{m}} f(W;\xi)\|
%         \le &\ \nabla f^\spp{m}_{\max}, \\
%         \|\nabla f(W;\xi)\|
%         \le &\ \nabla f_{\max} := \sum_{m=1}^M \nabla f^\spp{m}_{\max}.
%     \end{align}
%     % where the constant $C_{f}^\spp{m}$ is defined by
%     % \begin{align}
%     %     C_{f}^\spp{m} := (L_{g,0}W_{\max,2})^m C_{\delta}^\spp{m}
%     %     +(L_{g,0}W_{\max,2})^{M-m}L_{\ell,0}L_{g,0} C_{x}^\spp{m}
%     % \end{align}
% \end{lemma}

\begin{lemma}[Delayed error of gradient]
    \label{lemma:gradient-delay}
    % Suppose Assumptions \ref{assumption:normalized-feature}--\ref{assumption:loss-Lip} hold.
    Denote the stale gradient as 
    \begin{align}
        \tdnabla_{W^\spp{m}}^k f:= \tddelta^\spp{m}_k\otimes \tdx^\spp{m}_k.
    \end{align}
    The gradient delay is bounded by
    \begin{align}
        \|\tdnabla_{W^\spp{m}}^k f - \nabla_{W^\spp{m}} f(W_k;\xi_k)\|
        \le &\ 
        C_{f}^\spp{m}\sum_{m'=1}^{M}\sum_{k'=k-(M-m')}^{k-1} \|\Delta^\spp{m'-1}_{k'}\|.
    \end{align}
    where the constant $C_{f}^\spp{m}$ is defined by
    \begin{align}
        C_{f}^\spp{m} := (L_{g,0}W_{\max,2})^m C_{\delta}^\spp{m}
        +(L_{g,0}W_{\max,2})^{M-m}L_{\ell,0}L_{g,0} C_{x}^\spp{m}.
    \end{align}
    Consequently, the norm square of the gradient delay is bounded by
    \begin{align}
        \|\tdnabla^k f - \nabla f(W_k;\xi_k)\|^2
        \le &\ 
        C_{f}^2\sum_{m'=1}^{M}\sum_{k'=k-(M-m')}^{k-1} \|\Delta^\spp{m'}_{k'}\|^2
    \end{align}
    where the constant $C_f$ is defined by
    \begin{align}
        C_{f} := M\sqrt{\sum_{m=1}^{M} (C_f^\spp{m})^2}.
    \end{align}
\end{lemma}
\begin{proof}[Proof of Lemma \ref{lemma:gradient-delay}]
    By definition, the delay error is bounded by
    \begin{align}
        &\ \|\tdnabla_{W^\spp{m}}^k f - \nabla_{W^\spp{m}} f(W_k;\xi_k)\|
        % \\
        % ================================
        % =&\ 
        \|\tddelta^\spp{m}_k\otimes \tdx^\spp{m}_k 
        -\delta^\spp{m}_k\otimes x^\spp{m}_k\|
        % \nonumber 
        \\
        % ================================
        =&\ \|(\tddelta^\spp{m}_k-\delta^\spp{m}_k)\otimes \tdx^\spp{m}_k +\delta^\spp{m}_k\otimes(\tdx^\spp{m}_k-x^\spp{m}_k)\|
        % \nonumber \\
        % ================================
        \le \|\tddelta^\spp{m}_k-\delta^\spp{m}_k\| \|\tdx^\spp{m}_k\|+\|\delta^\spp{m}_k\|\|\tdx^\spp{m}_k-x^\spp{m}_k\|.
        \nonumber 
    \end{align}
    According to Lemma \ref{lemma:bounded-signal} and  Lemma \ref{lemma:signal-delay}, it follows that
    \begin{align}
        &\ \|\tdnabla_{W^\spp{m}}^k f - \nabla_{W^\spp{m}} f(W_k;\xi_k)\|
        \\
        % ================================
        \le &\ ((L_{g,0}W_{\max,2})^m C_{\delta}^\spp{m}
        +(L_{g,0}W_{\max,2})^{M-m}L_{\ell,0}L_{g,0} C_{x}^\spp{m})
        \sum_{m'=1}^{M}\sum_{k'=k-(M-m')}^{k-1} \|\Delta^\spp{m'}_{k'}\|
        \nonumber
        \\
        % ================================
        = &\ C_{f}^\spp{m}
        \sum_{m'=1}^{M}\sum_{k'=k-(M-m')}^{k-1} \|\Delta^\spp{m'}_{k'}\|.
        \nonumber
    \end{align}
    It completes the proof.
\end{proof}

\noindent
\textbf{Bounds of delay errors.}
% \section{Bounds of delay errors}
In {\AnalogSGDAP}, the computations in both the forward and backward passes involve stale weights, which introduce delay error in the signals.
This section presents bounds on the delay error.

\begin{lemma}[Lipschitz continuity of loss function]
Under same conditions with Lemma \ref{lemma:bounded-signal}, there exists $L_{\ell,0}$ such that for any $m\in [M]$, $\|\nabla_{x^{(m)}}\ell(x_k^{(m)},y_k)\|\leq L_{\ell,0}$ and $\|\nabla_{\tilde x^{(m)}}\ell(\tilde x_k^{(m)},y_k)\|\leq L_{\ell,0}$. 
\end{lemma}
\begin{proof}
According to norm equivalence, there exists $W_{\max,2}=c W_{\max}$ with $c$ is defined by the dimension of $W_k^{(m)}$, such that $\|W_k^{(m)}\|\leq W_{\max,2}\leq\tau$. 
For $m\in[M]$, according to the forward recursion \eqref{recursion:forward}, it holds that
    \begin{align}
        \|x^\spp{m+1}_k\| \overset{(a)}{=}&\ \|g^\spp{m}(z^\spp{m}_k)-g^\spp{m}(0)\| 
        \overset{(b)}{\le} L_{g,0}\|z^\spp{m}_k\|
        = L_{g,0}\|W^\spp{m}_k x^\spp{m}_k\| \\
        \le&\ L_{g,0}\|W^\spp{m}_k\| \|x^\spp{m}_k\|
        {\le} L_{g,0}W_{\max,2} \|x^\spp{m}_k\|
        \nonumber
        % \|\delta^\spp{m}_k\| \le \\
        % \|\tdx^\spp{m}_k\| \le \\ 
        % \|\tddelta^\spp{m}_k\| \le  
    \end{align}
    where (a) and (b) hold because of Assumption \ref{assumption:g-Lip} (1) and (2). Then by induction, we get $\|x^\spp{m}_k\| \le\ (L_{g,0}W_{\max,2})^m$ and the same bound holds for $\|x^\spp{m}_k\|$.     Since $\ell(\cdot,\cdot)$ is differentiable so that $\nabla\ell(\cdot,\cdot)$ is continuous, then there exists $L_{f,0}$ such that 
$\|\nabla_{x^\spp{m}}\ell(x^\spp{m}_k, y^k)\|\leq L_{f,0}$ and $\|\nabla_{\tilde x^\spp{m}}\ell(\tilde x^\spp{m}_k, y^k)\|\leq L_{f,0}$ hold for any $k,m\in[M]$.  
\end{proof}

\begin{lemma}[Bounded signal and iterates]
    Under Assumption \ref{assumption:normalized-feature}--\ref{assumption:loss-Lip}, selecting stepsize $\alpha\leq\frac{\tau}{(L_{g,0}\tau)^{M}L_{\ell,0}L_{g,0}}$ and initializing $\|W_{0}^{(m)}\|<\tau, \forall m\in [M+1]$, then both forward and backward signals are bounded, i.e., $\forall m\in[M+1]$
    \label{lemma:bounded-signal}
    \begin{align}
        \|x^\spp{m}_k\| \le&\ (L_{g,0}W_{\max,2})^m,&
        \|\delta^\spp{m}_k\| \le&\ (L_{g,0}W_{\max,2})^{M-m}L_{\ell,0}L_{g,0}, \\ 
        % ======================
        \|\tdx^\spp{m}_k\| \le&\ (L_{g,0}W_{\max,2})^m,&
        \|\tddelta^\spp{m}_k\| \le&\ (L_{g,0}W_{\max,2})^{M-m}L_{\ell,0}L_{g,0}.
    \end{align}
    Moreover, for any $k\leq K$ and $m\in [M+1]$, the iterates $\|W_k^{(m)}\|<\tau$ are bounded within the dynamic range, i.e. $W_{\max,2}:=\max_{k\leq K,m\in [M+1]} \{\|W_k^{(m)}\|\}<\tau$. 
\end{lemma}
\begin{proof}[Proof of Lemma \ref{lemma:bounded-signal}]\allowdisplaybreaks
    This lemma is proved by induction over $m$ and $k$. We can initialize such that $\|W_{0}^{(m)}\|<\tau$ holds for any $m\in [M+1]$. 
    
    At each iteration $k$, assume $\|W_{k}^{(m)}\|<\tau$ holds for any $m\in [M+1]$ and denote $W_{\max,2}^{(m)}=\max_m \{\|W_{k}^{(m)}\|\}$. Clearly, $W_{\max,2}^{(m)}<\tau$. Then the statement for $\|x^\spp{m}_k\|$ and $\|\tdx^\spp{m}_k\|$ at $m=1$ trivially hold because $x^\spp{0}_k=\tdx^\spp{0}_k=x_k$. For $\|\delta^\spp{M+1}_k\|$ and $\|\tddelta^\spp{M+1}_k\|$, we have
    \begin{align}
        % \label{recurssion:BP-beginning}
        \|\delta^\spp{M+1}_k\| =&\ \|\nabla_{x^\spp{M+1}}\ell(x^\spp{M+1}_k, y_k) \nabla g^\spp{M}(z^\spp{M}_k)\|
        \\
        \le&\ \|\nabla_{x^\spp{M+1}}\ell(x^\spp{M+1}_k, y_k)\|\| \nabla g^\spp{M}(z^\spp{M}_k)\|
        \le L_{\ell,0}L_{g,0}
        \nonumber \\
        % =================================
        \|\tddelta^\spp{M+1}_k\| =&\ \|\nabla_{x^\spp{M+1}}\ell(\tdx^\spp{M+1}_k, y_k) \nabla g^\spp{M}(\tdz^\spp{M}_k)\|
        \\
        \le&\ \|\nabla_{x^\spp{M+1}}\ell(\tdx^\spp{M+1}_k, y_k)\|\| \nabla g^\spp{M}(\tdz^\spp{M}_k)\|
        \le L_{\ell,0}L_{g,0}.
        \nonumber
    \end{align}
    For $m\in[M]$, according to the forward recursion \eqref{recursion:forward}, it holds that
    \begin{align}
        \|x^\spp{m+1}_k\| \overset{(a)}{=}&\ \|g^\spp{m}(z^\spp{m}_k)-g^\spp{m}(0)\| 
        \overset{(b)}{\le} L_{g,0}\|x^\spp{m}_k\|
        = L_{g,0}\|W^\spp{m}_k x^\spp{m}_k\| \\
        \le&\ L_{g,0}\|W^\spp{m}_k\| \|x^\spp{m}_k\|
        {\le} L_{g,0}W_{\max,2}^{(m)} \|x^\spp{m}_k\|
        \nonumber
        % \|\delta^\spp{m}_k\| \le \\
        % \|\tdx^\spp{m}_k\| \le \\ 
        % \|\tddelta^\spp{m}_k\| \le  
    \end{align}
    where (a) and (b) hold because of Assumption \ref{assumption:g-Lip} (1) and (2). According to the backward recursion \eqref{recursion:backward}, it holds that
    \begin{align}
        \|\delta^\spp{m}_k\|
        =&\ \|\delta^\spp{m+1}_k W^\spp{m+1}_k \nabla g^\spp{m}(z^\spp{m}_k)\|
        \\
        \le&\ \|\delta^\spp{m+1}_k\| \|W^\spp{m+1}_k\| \|\nabla g^\spp{m}(z^\spp{m}_k)\|
        \le L_{g,0}W_{\max,2}^{(m)} \|\delta^\spp{m+1}_k\|,
        \nonumber \\
        % ===================================
        \|\tddelta^\spp{m}_k\|
        =&\ \|\tddelta^\spp{m+1}_k W^\spp{m+1}_k \nabla g^\spp{m}(\tdz^\spp{m}_k)\|
        \\
        \le&\ \|\tddelta^\spp{m+1}_k\| \|W^\spp{m+1}_k\| \|\nabla g^\spp{m}(\tdz^\spp{m}_k)\|
        \le L_{g,0}W_{\max,2}^{(m)} \|\tddelta^\spp{m+1}_k\|.
        \nonumber
    \end{align}
    Therefore, by induction over $m$, we obtain, 
    \begin{align}
        \|x^\spp{m}_k\| \le&\ (L_{g,0}W_{\max,2}^{(m)})^m,&
        \|\delta^\spp{m}_k\| \le&\ (L_{g,0}W_{\max,2}^{(m)})^{M-m}L_{\ell,0}L_{g,0}, \\ 
        % ======================
        \|\tdx^\spp{m}_k\| \le&\ (L_{g,0}W_{\max,2}^{(m)})^m,&
        \|\tddelta^\spp{m}_k\| \le&\ (L_{g,0}W_{\max,2}^{(m)})^{M-m}L_{\ell,0}L_{g,0}.
    \end{align}
    On the other hand, if we select the stepsize $\alpha$ such that $\alpha\leq\frac{\tau}{(L_{g,0}\tau)^{M}L_{\ell,0}L_{g,0}}$, we have the magnitude of each update increment 
    \begin{align*}
    \|\Delta \tilde W_k^{(m)}\|=\alpha\|\tdx^\spp{m}_k\|\|\tddelta^\spp{m}_k\|\leq\tau.
    \end{align*}
    Owing to the induction nature proof for Theorem \ref{theorem:bounded-variable} (see {\citep[Theorem 1]{wu2024towards}}), we know $\|W_{k+1}^{(m)}\|<\tau$ holds for any $m\in [M+1]$. 
    
    By induction, the proof is completed. 
    % \begin{align}
    %     \tddelta^\spp{m}
    %     = \tddelta^\spp{m+1} W^\spp{m+1} \nabla g^\spp{m}(\tdz^\spp{m})
    % \end{align}
\end{proof}

\section{Smoothness of the object}
\label{section:proof-objective-L-smooth}
This section proves Lemma \ref{lemma:objective-L-smooth-short}, i.e. the object is $L$-smooth. Before that, we introduce a lemma as the step stone.
Since the definition of function smoothness involves two weights $W$, $W'$ and both of them generate a sequence of forward and backward signals, we require symbols to represent them. 
Using slightly abusing notation in this section, we denote the forward signals computed by $W$ and $\tdW$ as $x^\spp{m}$ and $\tdx^\spp{m}$, respectively, and denote the backward signals as $\delta^\spp{m}$ and $\tddelta^\spp{m}$, respectively.
% We should distinguish the 

\begin{lemma}[Signal stability, restatement of Lemma \ref{lemma:signal-stability-short}]
    \label{lemma:signal-stability}
    Denote the forward signals computed by $W$ and $\tdW$ as $x^\spp{m}$ and $\tdx^\spp{m}$, respectively, and denote the backward signals as $\delta^\spp{m}$ and $\tddelta^\spp{m}$, respectively.
    \begin{align}
        \label{inequality:x-stability}
        \|\tdx^\spp{m}-x^\spp{m}\|^2 \le&\ C_{x'}^\spp{m} 
        \|\tdW-W\|^2,\\ 
        % \sum_{m'=1}^{M}\|\tdW^\spp{m'}-W^\spp{m'}\|^2,\\ 
        \label{inequality:z-stability}
        \|\tdz^\spp{m}-z^\spp{m}\|^2 \le&\ C_{z'}^\spp{m} 
        \|\tdW-W\|^2,\\ 
        % \sum_{m'=1}^{M}\|\tdW^\spp{m'}-W^\spp{m'}\|^2,\\ 
        \|\tddelta^\spp{m}-\delta^\spp{m}\|^2 \le&\ C_{\delta'}^\spp{m}
        \|\tdW-W\|^2,
        % \sum_{m'=1}^{M}\|\tdW^\spp{m'}-W^\spp{m'}\|^2,
    \end{align}
    where the constants $C_x^\spp{m}$ and $C_z^\spp{m}$ are defined by
    \begin{align}
        \label{definition:C_x_stability}
        C_{x'}^\spp{m} :=&\ 2L_{g,0}(L_{g,0}W_{\max,2})^{2(m-1)}, \\
        \label{definition:C_z_stability}
        C_{z'}^\spp{m} :=&\ 2(L_{g,0}W_{\max,2})^{2(m-1)},
    \end{align}    
    and $C_{\delta'}^\spp{m}$ is defined by
    \begin{align}
        \label{definition:C_delta_stability}
        C_{\delta'}^\spp{m} &\ :=
        3(L_{g,0}W_{\max,2})^2(2L_{\ell, 1}^2L_{g, 0}^4+2L_{\ell, 0}^2L_{g, 1}^2)
        (L_{g,0}W_{\max,2})^{2(M-1)} \\
        &\ +3(L_{g,0}W_{\max,2})^{2(M-m)}L_{\ell,0}^2L_{g,0}^4 \sum_{m''=m}^{M-1}(L_{g,0}W_{\max,2})^{2(M-1-m'')}
        \frac{1-(L_{g,0}W_{\max,2})^{2(M-m)}}{1-L_{g,0}W_{\max,2}} 
        \nonumber\\
        &\ +3(L_{g,0}W_{\max,2})^{2(M-1)}L_{\ell,0}^2L_{g,1}^2 \sum_{m''=m}^{M-1}(L_{g,0}W_{\max,2})^{2(M-1-m'')}
        \frac{1-(L_{g,0}W_{\max,2})^{2(M-m)}}{1-L_{g,0}W_{\max,2}}.
        \nonumber
    \end{align}
\end{lemma}
\begin{proof}[Proof of Lemma \ref{lemma:signal-stability-short} and Lemma \ref{lemma:signal-stability}]
    We first bound the delay of forward signal $\tdx^\spp{m}_k$. At the first layer, there is no delay, i.e. $\tdx^\spp{1} = x^\spp{1} = x$. For other layers $m\ge1$
    \begin{align}
        \label{inequality:signal-stability-forward-1}
        &\ \|\tdx^\spp{m+1}-x^\spp{m+1}\|^2 
        \le L_{g,0}^2\|\tdz^\spp{m}-z^\spp{m}\|^2 \\
        % ===================================
        =&\ L_{g,0}^2\|\tdW^\spp{m}\tdx^\spp{m}-W^\spp{m}x^\spp{m}\|^2
        \nonumber \\
        % ===================================
        =&\ L_{g,0}^2\|(\tdW^\spp{m}-W^\spp{m})\tdx^\spp{m}+W^\spp{m}(\tdx^\spp{m}-x^\spp{m})\|^2
        \nonumber \\
        % ===================================
        \le&\ 2L_{g,0}^2(\|\tdW^\spp{m}-W^\spp{m}\|^2\|\tdx^\spp{m}\|^2+\|W^\spp{m}\|^2\|\tdx^\spp{m}-x^\spp{m}\|^2)
        \nonumber \\
        % ===================================
        \le&\ 2L_{g,0}^2((L_{g,0}W_{\max,2})^{2m}\|\tdW^\spp{m}-W^\spp{m}\|^2+W_{\max,2}^2\|\tdx^\spp{m}-x^\spp{m}\|^2).
        \nonumber 
    \end{align}
    Expanding inequality \eqref{inequality:signal-stability-forward-1} from $m$ to $1$ and using the fact $\tdx^\spp{1} = x^\spp{1} = x$ yield
    \begin{align}
        \label{inequality:signal-stability-forward-3}
        \|\tdx^\spp{m}-x^\spp{m}\|^2
        % \\
        % ===================================
        \le&\ 2L_{g,0}^2(L_{g,0}W_{\max,2})^{2(m-1)}\|\tdW^\spp{m-1}-W^\spp{m-1}\|^2
        +2L_{g,0}^2W_{\max,2}^2\|\tdx^\spp{m-1}-x^\spp{m-1}\|^2
        % \nonumber
        \\
        % ===================================
        \le&\ 2L_{g,0}^2\sum_{m'=2}^{m}(L_{g,0}W_{\max,2})^{2(m'-1)}(L_{g,0}W_{\max,2})^{2(m-m')}\|\tdW^\spp{m'-1}-W^\spp{m'-1}\|^2
        \nonumber\\
        % ===================================
        =&\ 2L_{g,0}^2(L_{g,0}W_{\max,2})^{2(m-1)}\sum_{m'=2}^{m}\|\tdW^\spp{m'-1}-W^\spp{m'-1}\|^2
        \nonumber\\
        % ===================================
        =&\ 2L_{g,0}^2(L_{g,0}W_{\max,2})^{2(m-1)}\sum_{m'=1}^{m-1}\|\tdW^\spp{m'}-W^\spp{m'}\|^2
        \nonumber
    \end{align}
    where the last equality reindexes the summation.
    By applying the relation 
    % $\sum_{m'=1}^{m}\|\tdW^\spp{m'}-W^\spp{m'}\|^2 = \|\tdW-W\|^2$, 
    \begin{align}
        \sum_{m'=1}^{m-1}\|\tdW^\spp{m'}-W^\spp{m'}\|^2 \le \sum_{m'=1}^{M}\|\tdW^\spp{m'}-W^\spp{m'}\|^2 = \|\tdW-W\|^2
    \end{align}
    we prove \eqref{inequality:x-stability}.
    Inferring in the similar way of \eqref{inequality:signal-stability-forward-1}-\eqref{inequality:signal-stability-forward-3}, we obtain \eqref{inequality:z-stability}
    \begin{align}
        \label{inequality:signal-stability-forward-z-1}
        \|\tdz^\spp{m}-z^\spp{m}\|^2\le&\ 2(L_{g,0}W_{\max,2})^{2(m-1)}
        \|\tdW-W\|^2
        % \sum_{m'=1}^{m-1}\|\tdW^\spp{m'}-W^\spp{m'}\|^2.
    \end{align}
    We then bound the difference of the backward signal $\tddelta^\spp{m}$ for $m\in[M-1]$
    \begin{align}
        \label{inequality:signal-stability-backward-1}
        &\ \|\tddelta^\spp{m}-\delta^\spp{m}\|^2
        =\|\tddelta^\spp{m+1} \tdW^\spp{m+1} \nabla g^\spp{m}(z^\spp{m})-\delta^\spp{m+1} W^\spp{m+1} \nabla g^\spp{m}(z^\spp{m})\|^2\\
        % ===================================
        \le&\ 3\|\tddelta^\spp{m+1}-\delta^\spp{m+1}\|^2 \|\tdW^\spp{m+1}\|^2 \|\nabla g^\spp{m}(z^\spp{m})\|^2
        \nonumber\\
        &\ +3\|\delta^\spp{m+1}\|^2 \|\tdW^\spp{m+1}-W^\spp{m+1}\|^2 \|\nabla g^\spp{m}(z^\spp{m})\|^2
        \nonumber\\
        &\ +3\|\delta^\spp{m+1}\|^2 \|W^\spp{m+1}\|^2 \|\nabla g^\spp{m}(z^\spp{m})-\nabla g^\spp{m}(\tdz^\spp{m})\|^2
        \nonumber\\
        % ===================================
        \le&\ 3(L_{g,0}W_{\max,2})^2\|\tddelta^\spp{m+1}-\delta^\spp{m+1}\|^2 
        +3(L_{g,0}W_{\max,2})^{2(M-m)}L_{\ell,0}^2L_{g,0}^4 \|\tdW^\spp{m+1}-W^\spp{m+1}\|^2
        \nonumber\\
        &\ +3W_{\max,2}^2(L_{g,0}W_{\max,2})^{2(M-m-1)}L_{\ell,0}^2 L_{g,0}^2 L_{g,1}^2 \|z^\spp{m}-\tdz^\spp{m}\|^2
        \nonumber
    \end{align}
    where the last inequality comes from Assumption \ref{assumption:g-Lip} (1)-(2) and Lemma \ref{lemma:bounded-signal}. Substituting \eqref{inequality:signal-stability-forward-z-1} into \eqref{inequality:signal-stability-backward-1}, 
    we can further bound \eqref{inequality:signal-stability-backward-1} by
    \begin{align}
        \label{inequality:signal-stability-backward-2}
        &\ 
        \|\tddelta^\spp{m}-\delta^\spp{m}\|^2
        \\
        % ===================================
        \le&\ 3L_{g,0}^2W_{\max,2}^2\|\tddelta^\spp{m+1}-\delta^\spp{m+1}\|^2
        +3(L_{g,0}W_{\max,2})^{2(M-m)}L_{\ell,0}^2L_{g,0}^4 \|\tdW^\spp{m+1}-W^\spp{m+1}\|^2
        \nonumber
        \\
        &\ +3W_{\max,2}^2(L_{g,0}W_{\max,2})^{2(M-m-1)}L_{\ell,0}^2L_{g,0}^2L_{g,1} ^2(L_{g,0}W_{\max,2})^{2(m-1)}\sum_{m'=1}^{m-1}\|\tdW^\spp{m'}-W^\spp{m'}\|^2
        \nonumber\\
        % ===================================
        =&\ 3L_{g,0}^2W_{\max,2}^2\|\tddelta^\spp{m+1}-\delta^\spp{m+1}\|^2 
        +3(L_{g,0}W_{\max,2})^{2(M-m)}L_{\ell,0}^2L_{g,0}^4 \|\tdW^\spp{m+1}-W^\spp{m+1}\|^2
        \nonumber\\
        &\ +3(L_{g,0}W_{\max,2})^{2(M-1)}L_{\ell,0}^2L_{g,1}^2 \sum_{m'=1}^{m-1}\|\tdW^\spp{m'}-W^\spp{m'}\|^2
        \nonumber\\
        % ===================================
        \le &\ 3(L_{g,0}W_{\max,2})^{2(M-m)}\|\delta^\spp{M}-\tddelta^\spp{M}\|^2
        \nonumber\\
        &\ + 3(L_{g,0}W_{\max,2})^{2(M-m)}L_{\ell,0}^2L_{g,0}^4 \sum_{m''=m}^{M-1}(L_{g,0}W_{\max,2})^{2(M-1-m'')}\|\tdW^\spp{m''+1}-W^\spp{m''+1}\|^2
        \nonumber\\
        &\ +3(L_{g,0}W_{\max,2})^{2(M-1)}L_{\ell,0}^2L_{g,1}^2 \sum_{m''=m}^{M-1}(L_{g,0}W_{\max,2})^{2(M-1-m'')}\sum_{m'=1}^{m''-1}\|\tdW^\spp{m'}-W^\spp{m'}\|^2.
        \nonumber
    \end{align}
    
    To bound the first term in the \ac{RHS} of \eqref{inequality:signal-stability-backward-2}, we manipulate the different $\|\delta^\spp{M}-\tddelta^\spp{M}\|$ by
    \begin{align}
        \label{inequality:signal-stability-backward-2-1}
        &\ 
        \|\delta^\spp{M}-\tddelta^\spp{M}\|^2 
        \\
        % =================================
        =&\ \|\nabla_{x^\spp{M+1}}\ell(x^\spp{M+1}, y) \nabla g^\spp{M}(z^\spp{M}) - \nabla_{x^\spp{M+1}}\ell(\tdx^\spp{M+1}, y) \nabla g^\spp{M}(\tdz^\spp{M})\|^2
        \nonumber 
        \\
        % =================================
        \le&\ 2\|\nabla_{x^\spp{M+1}}\ell(x^\spp{M+1}, y)- \nabla_{x^\spp{M+1}}\ell(\tdx^\spp{M+1}, y)\|^2
        \|\nabla g^\spp{M}(z^\spp{M})\|^2 
        \nonumber \\
        &\ +2\|\nabla_{x^\spp{M+1}}\ell(\tdx^\spp{M+1}, y)\|^2
        \|\nabla g^\spp{M}(\tdz^\spp{M})-\nabla g^\spp{M}(z^\spp{M})\|^2
        \nonumber \\
        % =================================
        \overset{(a)}{\le}&\ 2L_{\ell, 1}^2L_{g, 0}^2 \|x^\spp{M+1} - \tdx^\spp{M+1}\|^2
        +2L_{\ell, 0}^2L_{g, 1}^2
        \|\tdz^\spp{M}-z^\spp{M}\|^2
        \nonumber \\
        % =================================
        \le&\ (2L_{\ell, 1}^2L_{g, 0}^4+2L_{\ell, 0}^2L_{g, 1}^2)
        \|\tdz^\spp{M}-z^\spp{M}\|^2
        \nonumber \\
        % =================================
        \le&\ (2L_{\ell, 1}^2L_{g, 0}^4+2L_{\ell, 0}^2L_{g, 1}^2)
        (L_{g,0}W_{\max,2})^{2(M-1)}
        \sum_{m'=1}^{M-1}\|\tdW^\spp{m'}-W^\spp{m'}\|^2.
        \nonumber
    \end{align}
    where (a) comes from Assumption \ref{assumption:g-Lip} and Assumption \ref{assumption:loss-Lip}.
    
    To bound the second term in the \ac{RHS} of \eqref{inequality:signal-stability-backward-2}, we add some terms into the summation
    \begin{align}
        \label{inequality:signal-stability-backward-2-2}
        &\ 
        \sum_{m''=m}^{M-1}(L_{g,0}W_{\max,2})^{2(M-1-m'')}\|\tdW^\spp{m''+1}-W^\spp{m''+1}\|^2 
        \\
        % =================================
        \le&\ \sum_{m''=m}^{M-1}(L_{g,0}W_{\max,2})^{2(M-1-m'')}\sum_{m'=1}^{M}\|\tdW^\spp{m'}-W^\spp{m'}\|^2
        \nonumber 
        \\
        % =================================
        =&\ \frac{1-(L_{g,0}W_{\max,2})^{2(M-m)}}{1-L_{g,0}W_{\max,2}}\sum_{m'=1}^{M}\|\tdW^\spp{m'}-W^\spp{m'}\|^2.
        \nonumber
    \end{align}
    
    To bound the third term in the \ac{RHS} of \eqref{inequality:signal-stability-backward-2}, we change the upper bound of the second summation from $m''$ to $M$ and get
    \begin{align}
        \label{inequality:signal-stability-backward-2-3}
        &\ 
        \sum_{m''=m}^{M-1}(L_{g,0}W_{\max,2})^{2(M-1-m'')}\sum_{m'=1}^{m''-1}\|\tdW^\spp{m'}-W^\spp{m'}\|^2 
        \\
        % =================================
        \le&\ \sum_{m''=m}^{M-1}(L_{g,0}W_{\max,2})^{2(M-1-m'')}\sum_{m'=1}^{M}\|\tdW^\spp{m'}-W^\spp{m'}\|^2 
        \nonumber
        \\
        % =================================
        =&\ \frac{1-(L_{g,0}W_{\max,2})^{2(M-m)}}{1-L_{g,0}W_{\max,2}}\sum_{m'=1}^{M}\|\tdW^\spp{m'}-W^\spp{m'}\|^2.
        \nonumber
    \end{align}

    Substituting \eqref{inequality:signal-stability-backward-2-1}--\eqref{inequality:signal-stability-backward-2-3} back into \eqref{inequality:signal-stability-backward-2} we have
    \begin{align}
        \|\tddelta^\spp{m}-\delta^\spp{m}\|^2
        \le C_{\delta'}^\spp{m}\sum_{m'=1}^{M}\|\tdW^\spp{m'}-W^\spp{m'}\|^2.
    \end{align}
    This completes the proof.
\end{proof}

% as claimed by the following lemma.

\begin{lemma}[smoothness of the objective, restatement of Lemma \ref{lemma:objective-L-smooth-short}]
    \label{lemma:objective-L-smooth}
    Under Assumption \ref{assumption:normalized-feature}--\ref{assumption:W-bounded}, the objective is $L$-smooth with respect to $W$, i.e. $\forall \xi=(x, y)\in\ccalX\times\ccalY$, 
    \begin{align}
        \label{equation:L}
        % f(W')\le  f(W)+\la \nabla f(W), W'-W \ra   + \frac{L}{2}\|W'-W\|^2.
        \|\nabla f(W;\xi)-\nabla f(\tdW;\xi)\| \le L_f\|W-\tdW\|
        % \quad \forall W, W' \in \reals^{D}.
    \end{align}
    where the smoothness constant is defined as 
    \begin{align}
        L_f := \sqrt{2\sum_{m=1}^{M}((L_{g,0}W_{\max,2})^{2m} C_{\delta'}^\spp{m}
        +(L_{g,0}W_{\max,2})^{2(M-m)}L_{\ell,0}^2L_{g,0}^2 C_{x'}^\spp{m})}.
    \end{align}
\end{lemma}
\begin{proof}[Proof of Lemma \ref{lemma:objective-L-smooth-short} and Lemma \ref{lemma:objective-L-smooth}]
    Analogous to the proof of Lemma \eqref{lemma:signal-stability}, we use $x^\spp{m}$/$\delta^\spp{m}$ and $\tdx^\spp{m}$/$\tddelta^\spp{m}$ to denote the forward/backward signals related to $W$ and $\tdW$, respectively.
    It follows by definition of gradient \eqref{definition:gradient} that
    \begin{align}
        % &\ 
        \|\nabla f(W;\xi)-\nabla f(\tdW;\xi)\|^2
        % \\
        % ================================
        =&\ \sum_{m=1}^{M}\|\nabla_{W^\spp{m}} f(W;\xi)-\nabla_{W^\spp{m}} f(\tdW;\xi)\|^2
        \\
        % ================================
        =&\ \sum_{m=1}^{M}\|\tddelta^\spp{m}\otimes \tdx^\spp{m} 
        -\delta^\spp{m}\otimes x^\spp{m}\|^2
        \nonumber \\
        % ================================
        =&\ \sum_{m=1}^{M}\|(\tddelta^\spp{m}-\delta^\spp{m})\otimes \tdx^\spp{m} +\delta^\spp{m}\otimes(\tdx^\spp{m}-x^\spp{m})\|^2
        \nonumber \\
        % ================================
        \le &\ 2\sum_{m=1}^{M}\|\tddelta^\spp{m}-\delta^\spp{m}\|^2 \|\tdx^\spp{m}\|^2+2\|\delta^\spp{m}\|^2\|\tdx^\spp{m}-x^\spp{m}\|^2.
        \nonumber 
    \end{align}
    According to Lemma \ref{lemma:bounded-signal} and  Lemma \ref{lemma:signal-stability}, it follows that
    \begin{align}
        \label{inequality:objective-L-smooth-2}
        &\ 
        \|\nabla f(W;\xi)-\nabla f(\tdW;\xi)\|^2
        \\
        % ================================
        \le &\ 2\sum_{m=1}^{M}((L_{g,0}W_{\max,2})^{2m} C_{\delta'}^\spp{m}
        +(L_{g,0}W_{\max,2})^{2(M-m)}L_{\ell,0}^2L_{g,0}^2 C_{x'}^\spp{m})
        \|\tdW-W\|^2
        \nonumber
        \\
        % ================================
        = &\ L_f^2 \|\tdW-W\|^2.
        \nonumber
    \end{align}
    Taking the square root on both sides of \eqref{inequality:objective-L-smooth-2} completes the proof.
\end{proof}
% \textbf{Difficulties}
% \begin{enumerate}
%   \item BatchNorm
%   \item Noise amplification
% \end{enumerate}

\section{Proof of Theorem \ref{theorem:ASGD-sync-convergence-noncvx-linear}: Analog SGD with synchronous pipeline}
\label{section:proof-ASGD-sync-convergence-noncvx-linear}
This section provides the convergence guarantee of the Analog SGD with synchronous pipeline, whose iteration is
\begin{align}
    W_{k, b+1}^\spp{m} = W_{k, b}^\spp{m} - \frac{\alpha}{B} \nabla_{W^\spp{m}} f(W_k; \xi_{k, b})
    - \frac{\alpha}{\tau B}|\nabla_{W^\spp{m}} f(W_k; \xi_{k, b})|\odot W_{k,b}^\spp{m}.
\end{align}
\ThmASGDSyncConvergenceNoncvxLinear*

\begin{proof}[Proof of Theorem \ref{theorem:ASGD-sync-convergence-noncvx-linear}]
Denote the average gradient by
\begin{align}
    \barnabla f(W_k;\xi_k) := \frac{1}{B}\sum_{b=1}^{B}\nabla f(W_k;\xi_{k,b}),
\end{align}
which satisfies that $\mbE_{\{\xi_{k,b}:b\in[B]\}}[\barnabla f(W_k;\xi_k)] = \nabla f(W_k)$. Furthermore, the independence of sampling (c.f. Assumption \ref{assumption:noise}) ensures that 
\begin{align}
    \label{inequality:ASGD-sync-convergence-linear-0}
    &\ \mbE_{\{\xi_{k,b}:b\in[B]\}}[\|\barnabla f(W_k;\xi_k)-\nabla f(W_k)\|^2]
    \\
    % ========================
    =&\ \mbE_{\{\xi_{k,b}:b\in[B]\}}[\|\frac{1}{B}\sum_{b=1}^{B}\nabla f(W_k;\xi_{k,b})-\nabla f(W_k)\|^2] 
    \nonumber
    \\
    % ========================
    =&\ \frac{1}{B^2}\sum_{b=1}^{B}\mbE_{\xi_{k,b}}[\|\nabla f(W_k;\xi_{k,b})-\nabla f(W_k)\|^2]
    % \nonumber\\
    % ========================
    \le \sigma^2 / B.
    \nonumber
\end{align}
The $L$-smooth assumption of the objective (c.f. Lemma \ref{lemma:objective-L-smooth}) implies that
    \begin{align}
        \label{inequality:ASGD-sync-convergence-linear-1}
        &\  \mbE_{\{\xi_{k,b}:b\in[B]\}}[f(W_{k+1})]  \\
        % ========================
        \le&\ f(W_k)+ \mbE_{\{\xi_{k,b}:b\in[B]\}}[\la \nabla f(W_k), W_{k+1}-W_k\ra] + \frac{L_f}{2} \mbE_{\{\xi_{k,b}:b\in[B]\}}[\|W_{k+1}-W_k\|^2]
        \nonumber \\
        % ========================
        \le&\ f(W_k) - \frac{\alpha}{2} \|\nabla f(W_k)\|^2 -  (\frac{1}{2\alpha}-L_f)\mathbb{E}_{\varepsilon_k}[\|W_{k+1}-W_k+\alpha(\barnabla f(W_k;\xi_k)-\nabla f(W_k))\|^2]
        \nonumber \\
        &\ +\alpha^2L_f \mathbb{E}_{\varepsilon_k}[\|\barnabla f(W_k;\xi_k)-\nabla f(W_k)\|^2]
        + \frac{1}{2\alpha}\|W_{k+1}-W_k+\alpha(\nabla f(W_k)+\varepsilon_k)\|^2
        \nonumber
        % \\
        % % ========================
        % \le&\ f(W_k) - \frac{\alpha}{2} \|\nabla f(W_k)\|^2 
        % +\alpha^2L\sigma^2
        % + \frac{1}{2\alpha}\|W_{k+1}-W_k+\alpha \nabla f(W_k)\|^2,
        % \nonumber
    \end{align}
    where the second inequality comes from the assumption that the expectation of the noise is $0$ (Assumption \ref{assumption:noise})
    \begin{align}
        &\ 
        \mathbb{E}_{\varepsilon_k}[\la \nabla f(W_k), W_{k+1}-W_k\ra] 
        \\
        % ========================
        =&\ \mathbb{E}_{\varepsilon_k}[\la \nabla f(W_k), W_{k+1}-W_k\alpha (\barnabla f(W_k;\xi_k)-\nabla f(W_k))\ra]
        \nonumber 
        \\
        % ========================
        =&\ - \frac{\alpha}{2} \|\nabla f(W_k)\|^2 
        - \frac{1}{2\alpha}\mathbb{E}_{\varepsilon_k}[\|W_{k+1}-W_k+\alpha (\barnabla f(W_k;\xi_k)-\nabla f(W_k))\|^2]
        \nonumber \\
        % ========================
        &\ + \frac{1}{2\alpha}\mathbb{E}_{\varepsilon_k}[\|W_{k+1}-W_k+\alpha \barnabla f(W_k;\xi_k)\|^2]
        \nonumber
    \end{align}
    and the following inequality
    \begin{align}
        \frac{L_f}{2}\mathbb{E}_{\varepsilon_k}[\|W_{k+1}-W_k\|^2]  
        % \\
        % ========================
        \le&\  L_f \mbE_{\{\xi_{k,b}:b\in[B]\}}[\|W_{k+1}-W_k+\alpha(\barnabla f(W_k;\xi_k)-\nabla f(W_k))\|^2]  
        % \nonumber
        \\
        % ========================
        &\ +\alpha^2L \mbE_{\{\xi_{k,b}:b\in[B]\}}[\|\barnabla f(W_k;\xi_k)-\nabla f(W_k)\|^2].
        \nonumber
    \end{align}
    With the learning rate $\alpha\le\frac{1}{2L}$ and bounded variance (c.f. \eqref{inequality:ASGD-sync-convergence-linear-0}), \eqref{inequality:ASGD-sync-convergence-linear-1} becomes
    \begin{align}
        \label{inequality:ASGD-sync-convergence-linear-2}
        \mbE_{\{\xi_{k,b}:b\in[B]\}}[f(W_{k+1})]
        % ========================
        \le&\ f(W_k) - \frac{\alpha}{2} \|\nabla f(W_k)\|^2 
        +\frac{\alpha^2L_f\sigma^2}{B}
        \bkeqwn
        + \frac{1}{2\alpha}\mbE_{\{\xi_{k,b}:b\in[B]\}}[\|W_{k+1}-W_k+\alpha\barnabla f(W_k;\xi_k)\|^2].
        % \nonumber
    \end{align}
    % It is worth noting that the inequality \eqref{inequality:ASGD-convergence-linear-2} still holds
    According the dynamics of {\AnalogSGDSP},
    % \eqref{dynamic:analog-SGD-sync}, 
    the last term in the \ac{RHS} of \eqref{inequality:ASGD-sync-convergence-linear-2} is bounded by
    \begin{align}
        \label{inequality:ASGD-sync-convergence-linear-2-1}
        &\ \frac{1}{2\alpha} \mbE_{\{\xi_{k,b}:b\in[B]\}}[\|W_{k+1}-W_k+\alpha \barnabla f(W_k;\xi_k)\|^2]
        \\
        % ========================
        =&\ \frac{\alpha}{2\tau^2} \mbE_{\{\xi_{k,b}:b\in[B]\}}[\|\frac{1}{ B}\sum_{b=1}^B|\nabla f(W_k; \xi_{k, b})|\odot W_{k,b}\|^2]
        \nonumber
        \\
        % ========================
        \le&\ \frac{\alpha}{2\tau^2} \frac{1}{ B}\sum_{b=1}^B\mbE_{\xi_{k,b}}[\||\nabla f(W_k; \xi_{k, b})|\odot W_{k,b}\|^2]
        \nonumber
        \\
        % ========================
        \le&\ \frac{\alpha}{2\tau^2}\max_{b\in[B]}\|W_{k,b}\|^2_\infty \frac{1}{B}\sum_{b=1}^B\mbE_{\xi_{k,b}}[\|\nabla f(W_k; \xi_{k, b})\|^2 ].
        \nonumber
    \end{align}
    By the variance decomposition and Assumption \ref{assumption:noise}, the term above can be bounded by
    \begin{align}
        \label{inequality:grad-var-decomposion-microbatch}
        \mbE_{\xi_{k,b}}[\|f(W_k;\xi_{k,b})\|^2]
        =&\ \|\nabla f(W_k)\|^2+\mbE_{\xi_k}[\|\nabla f(W_k;\xi_{k,b})-\nabla f(W_k)\|^2] 
        % \\
        % ========================
        \le \|\nabla f(W_k)\|^2+\sigma^2.
        % \nonumber
    \end{align}
    Consequently, combining \eqref{inequality:ASGD-sync-convergence-linear-2}, \eqref{inequality:ASGD-sync-convergence-linear-2-1} and \eqref{inequality:grad-var-decomposion-microbatch} yields
    \begin{align}
        \label{inequality:ASGD-sync-convergence-linear-3}
        &\ 
        \mbE_{\{\xi_{k,b}:b\in[B]\}}[f(W_{k+1})] 
        \\
        % ========================
        \le&\ f(W_k) - \frac{\alpha}{2}\lp 1-\frac{\max_{b\in[B]}\|W_{k,b}\|^2_\infty}{\tau^2}\rp \|\nabla f(W_k)\|^2 
        + \frac{\alpha^2L_f\sigma^2}{B}
        + \frac{\alpha\sigma^2}{2\tau^2}\max_{b\in[B]}\|W_{k,b}\|^2_\infty .
        \nonumber
    \end{align}
    Organizing the terms, taking expectation, averaging \eqref{inequality:ASGD-sync-convergence-linear-3} over $k$ from $0$ to $K-1$, we obtain
    \begin{align}
        \label{inequality:ASGD-sync-convergence-linear-4}
        \frac{1}{K}\sum_{k=0}^{K-1} \mbE[\|\nabla f(W_k)\|^2]
        % ========================
        \le&\ \frac{2(f(W_0) - \mbE[f(W_{k+1})])}{1-W_{\max, \infty}^2/\tau^2}
        + \frac{\alpha L_f\sigma^2}{B(1-W_{\max, \infty}^2/\tau^2)} 
        \\
        & 
        + \frac{\sigma^2}{K}\sum_{k=0}^{K-1} \frac{\max_{b\in[B]}\|W_{k,b}\|^2_\infty / \tau^2}{1-\max_{b\in[B]}\|W_{k,b}\|^2_\infty/\tau^2} .
        \nonumber 
        \\
        % ========================
        \le&\ 4\sqrt{\frac{(f(W_0) -f^*)\sigma^2L_f}{BK}}
        \frac{1}{1-W_{\max,\infty}^2/\tau^2}
        + \sigma^2 S_K
        \nonumber
    \end{align}
    where the second inequality uses $f(W_{k+1})\ge f^*$ and specifies $\alpha=\sqrt{\frac{B(f(W_0) - f^*)}{\sigma^2L_fK}}$. Now we complete the proof.
\end{proof}

\section{Proof of Theorem \ref{theorem:ASGD-async-convergence-noncvx-linear}: Analog SGD with asynchronous pipeline}
\label{section:proof-ASGD-async-convergence-noncvx-linear}
This section provides the convergence guarantee of the Analog SGD under non-convex assumption on asymmetric linear devices.
\ThmASGDAsyncConvergenceNoncvxLinear*

\begin{proof}[Proof of Theorem \ref{theorem:ASGD-async-convergence-noncvx-linear}]\allowdisplaybreaks
    
    % Cooperated with Assumption \ref{assumption:noise}, the last term in the \ac{RHS} of \eqref{inequality:ASGD-convergence-linear-1} can be bounded by
    % \begin{align}
    %     \label{inequality:ASGD-convergence-linear-2}
    %     &\ \frac{1}{2\alpha}\mathbb{E}_{\varepsilon_k}[\|W_{k+1}-W_k+\alpha \nabla f(W_k)+\alpha\varepsilon_k\|^2]
    %     \\
    %     % ========================
    %     =&\ \frac{\alpha}{2\tau^2}\mathbb{E}_{\varepsilon_k}[\||\nabla f(W_k)+\varepsilon_k|\odot W_k \|^2]
    %     \nonumber \\
    %     % ========================
    %     \le&\ \frac{\alpha}{2\tau^2}\mathbb{E}_{\varepsilon_k}[\|\nabla f(W_k)+\varepsilon_k\|^2]\ \|W_k \|^2_\infty
    %     \nonumber \\
    %     % ========================
    %     \le&\ \frac{\alpha}{2\tau^2}\|\nabla f(W_k)\|^2 \|W_k \|^2_\infty
    %         +\frac{\alpha\sigma^2}{2\tau^2}\|W_k \|^2_\infty
    %     \nonumber 
    % \end{align}
    The $L$-smooth assumption of the objective (c.f. Lemma \ref{lemma:objective-L-smooth}) implies that
    \begin{align}
        \label{inequality:ASGD-convergence-linear-1}
        &\ 
        \mathbb{E}_{\xi_k}[f(W_{k+1})]  
        \\
        % ========================
        \le&\ f(W_k)+\mathbb{E}_{\xi_k}[\la \nabla f(W_k), W_{k+1}-W_k\ra] + \frac{L_f}{2}\mathbb{E}_{\xi_k}[\|W_{k+1}-W_k\|^2]
        \nonumber 
        \\
        % ========================
        \le&\ f(W_k) - \frac{\alpha}{2} \|\nabla f(W_k)\|^2 -  (\frac{1}{2\alpha}-L_f)\mathbb{E}_{\xi_k}[\|W_{k+1}-W_k+\alpha(\nabla f(W_k;\xi_k)-\nabla f(W_k))\|^2]
        \nonumber \\
        &\ +\alpha^2L_f \mathbb{E}_{\xi_k}[\|\nabla f(W_k;\xi_k)-\nabla f(W_k)\|^2]
        + \frac{1}{2\alpha}\mathbb{E}_{\xi_k}[\|W_{k+1}-W_k+\alpha \nabla f(W_k;\xi_k)\|^2]
        \nonumber
        % \\
        % % ========================
        % \le&\ f(W_k) - \frac{\alpha}{2} \|\nabla f(W_k)\|^2 
        % +\alpha^2L\sigma^2
        % + \frac{1}{2\alpha}\|W_{k+1}-W_k+\alpha \nabla f(W_k)\|^2,
        % \nonumber
    \end{align}
    where the second inequality comes from the assumption that the expectation of the noise is $0$ (Assumption \ref{assumption:noise})
    \begin{align}
        &\ 
        \mathbb{E}_{\xi_k}[\la \nabla f(W_k), W_{k+1}-W_k\ra] 
        \\
        % ========================
        =&\ \mathbb{E}_{\xi_k}[\la \nabla f(W_k), W_{k+1}-W_k+\alpha (\nabla f(W_k;\xi_k)-\nabla f(W_k))\ra]
        \nonumber 
        \\
        % ========================
        =&\ - \frac{\alpha}{2} \|\nabla f(W_k)\|^2 
        - \frac{1}{2\alpha}\mathbb{E}_{\xi_k}[\|W_{k+1}-W_k+\alpha (\nabla f(W_k;\xi_k)-\nabla f(W_k))\|^2]
        \bkeq
        + \frac{1}{2\alpha}\mathbb{E}_{\xi_k}[\|W_{k+1}-W_k+\alpha \nabla f(W_k;\xi_k)\|^2]
        \nonumber
    \end{align}
    and the following inequality
    \begin{align}
        \frac{L_f}{2}\mathbb{E}_{\xi_k}[\|W_{k+1}-W_k\|^2]
        \le&\  L_f\mathbb{E}_{\xi_k}[\|W_{k+1}-W_k+\alpha(\nabla f(W_k;\xi_k)-\nabla f(W_k))\|^2]  \\
        % ========================
        &\ +\alpha^2L_f\mathbb{E}_{\xi_k}[\|\nabla f(W_k;\xi_k)-\nabla f(W_k)\|^2].
        \nonumber
    \end{align}
    With the learning rate $\alpha\le\frac{1}{4L_f}$ and bounded variance of noise (Assumption \ref{assumption:noise}), \eqref{inequality:ASGD-convergence-linear-1} becomes
    \begin{align}
        \label{inequality:ASGD-convergence-linear-2}
        \mathbb{E}_{\xi_k}[f(W_{k+1})]
        \le&\ f(W_k) - \frac{\alpha}{2} \|\nabla f(W_k)\|^2 
        +\alpha^2L_f\sigma^2
        + \frac{1}{2\alpha}\mathbb{E}_{\xi_k}[\|W_{k+1}-W_k+\alpha \nabla f(W_k;\xi_k)\|^2]
        \nonumber \\
        % ========================
        &\ - \frac{1}{4\alpha}\mathbb{E}_{\xi_k}[\|W_{k+1}-W_k+\alpha(\nabla f(W_k;\xi_k)-\nabla f(W_k))\|^2].
    \end{align}
    % It is worth noting that the inequality \eqref{inequality:ASGD-convergence-linear-2} still holds
    By the dynamics of Analog SGD, the last in the \ac{RHS} of \eqref{inequality:ASGD-convergence-linear-2} is bounded by
    \begin{align}
        \label{inequality:ASGD-asyn-bias-1}
        &\ \frac{1}{2\alpha}\mathbb{E}_{\xi_k}[\|W_{k+1}-W_k+\alpha \nabla f(W_k;\xi_k)\|^2]
        \\
        % ========================
        =&\ \frac{\alpha}{2}\mathbb{E}_{\xi_k}[\|\tdnabla^k f - \nabla f(W_k;\xi_k)+\frac{1}{\tau}|\tdnabla^k f|\odot W_k \|^2]
        \nonumber \\
        % ========================
        =&\ \frac{\alpha}{2}\mathbb{E}_{\xi_k}[\|\tdnabla^k f - \nabla f(W_k;\xi_k)+\frac{1}{\tau}\tdnabla^k f \odot \sign(\tdnabla^k f)\odot W_k \|^2]
        \nonumber \\
        % ========================
        =&\ \frac{\alpha}{2}\mathbb{E}_{\xi_k}[\|\tdnabla^k f - \nabla f(W_k;\xi_k)
        +\frac{1}{\tau}(\tdnabla^k f-\nabla f(W_k;\xi_k)) \odot \sign(\tdnabla^k f)\odot W_k 
        \nonumber \\
        &\ \hspace{10em}
        + \frac{1}{\tau}\nabla f(W_k;\xi_k) \odot \sign(\tdnabla^k f)\odot W_k \|^2]
        \nonumber \\
        % ========================
        =&\ \frac{\alpha}{2}\mathbb{E}_{\xi_k}[\|(\bbone+\frac{1}{\tau} \odot \sign(\tdnabla^k f)\odot W_k )(\tdnabla^k f - \nabla f(W_k;\xi_k))
        \nonumber \\
        &\ \hspace{10em}
        + \frac{1}{\tau}\nabla f(W_k;\xi_k) \odot \sign(\tdnabla^k f)\odot W_k \|^2]
        \nonumber
    \end{align}
    where $\bbone$ is the all-one vector.
    Using inequality $\|x+y\|^2\le (1+\frac{1}{u}) \|x\|^2+(1+u)\|y\|^2$ we obtain 
    \begin{align}
        \label{inequality:ASGD-asyn-bias-2}
        &\ \frac{1}{2\alpha}\mathbb{E}_{\xi_k}[\|W_{k+1}-W_k+\alpha \nabla f(W_k;\xi_k)\|^2]
        \\
        % ========================
        \le&\ \frac{\alpha}{2}(1+\frac{1}{u})\mathbb{E}_{\xi_k}[\|(\bbone+\frac{1}{\tau} \odot \sign(\tdnabla^k f)\odot W_k )\odot(\tdnabla^k f - \nabla f(W_k;\xi_k)) \|^2]
        \nonumber \\
        &\ + \frac{\alpha}{2\tau^2}(1+u)\mathbb{E}_{\xi_k}[\|\nabla f(W_k;\xi_k) \odot \sign(\tdnabla^k f)\odot W_k \|^2]
        % \nonumber \\
        % &\ + \frac{\alpha}{2}\mathbb{E}_{\xi_k}\lB\la
        %     (\bbone+\frac{1}{\tau} \odot \sign(\tdnabla^k f)\odot W_k )(\tdnabla^k f - \nabla f(W_k;\xi_k)),
        %     \nabla f(W_k;\xi_k) \odot \sign(\tdnabla^k f)\odot W_k
        % \ra\rB
        % \nonumber \\
        % % ========================
        % \le&\ \lp\frac{\alpha \tau}{\tau -\sqrt{W_{\max,\infty}}} + \frac{\alpha}{\tau \sqrt{W_{\max,\infty}}}\rp
        % \mathbb{E}_{\xi_k}[\|\tdnabla^k f - \nabla f(W_k;\xi_k)\|^2] 
        % +\frac{\alpha}{\tau \sqrt{W_{\max,\infty}}}\mathbb{E}_{\xi_k}[|\tdnabla^k f|\odot W_k \|^2]
        % \nonumber \\
        % % ========================
        % \le&\ \frac{\alpha C_{f}}{2}\sum_{m'=1}^{M}\sum_{k'=k-(M-m'+1)}^{k-1} \|\Delta^\spp{m'-1}_{k'}\|
        \nonumber
    \end{align}
    The first term in the \ac{RHS} of \eqref{inequality:ASGD-asyn-bias-2} is bounded by
    \begin{align}
        \label{inequality:ASGD-asyn-bias-2-1}
        &\ \frac{\alpha}{2}(1+\frac{1}{u})\mathbb{E}_{\xi_k}[\|(\bbone+\frac{1}{\tau} \odot \sign(\tdnabla^k f)\odot W_k )\odot(\tdnabla^k f - \nabla f(W_k;\xi_k)) \|^2]
        \\
        % ========================
        \le &\ \frac{\alpha}{2}(1+\frac{1}{u})\mathbb{E}_{\xi_k}[
            \|\bbone+\frac{1}{\tau} \odot \sign(\tdnabla^k f)\odot W_k \|^2_\infty
            \|\tdnabla^k f - \nabla f(W_k;\xi_k) \|^2]
        \nonumber\\
        % ========================
        \le &\ \frac{\alpha}{2}(1+\frac{1}{u})
        (1+W_{\max,\infty}/\tau)^2\mathbb{E}_{\xi_k}[
            \|\tdnabla^k f - \nabla f(W_k;\xi_k) \|^2]
        \nonumber\\
        % ========================
        \le &\ \frac{\alpha}{2}(1+\frac{1}{u})(1+W_{\max,\infty}/\tau)^2M^2
            C_{f}^2 
            \mathbb{E}_{\xi_k}\lB\sum_{m'=1}^{M}\sum_{k'=k-(M-m')}^{k-1} \|\Delta^\spp{m'}_{k'}\|^2\rB.
        \nonumber
    \end{align}
    where the last inequality comes from Lemma \ref{lemma:gradient-delay}.
    
    The second term in the \ac{RHS} of \eqref{inequality:ASGD-asyn-bias-2} is bounded by
    % The third term in the \ac{RHS} of \eqref{inequality:ASGD-asyn-bias-2} is bounded by
    % \begin{align}
    %     &\ \frac{\alpha}{2}\mathbb{E}_{\xi_k}\lB\la
    %         (\bbone+\frac{1}{\tau} \odot \sign(\tdnabla^k f)\odot W_k )(\tdnabla^k f - \nabla f(W_k;\xi_k)),
    %         \nabla f(W_k;\xi_k) \odot \sign(\tdnabla^k f)\odot W_k
    %     \ra\rB
    %     \nonumber \\
    %     \le&\ \frac{\alpha}{2}\mathbb{E}_{\xi_k}\lB\la
    %         (\bbone+\frac{1}{\tau} \odot \sign(\tdnabla^k f)\odot W_k )(\tdnabla^k f - \nabla f(W_k;\xi_k)),
    %         \nabla f(W_k;\xi_k) \odot \sign(\tdnabla^k f)\odot W_k
    %     \ra\rB
    % \end{align}
    \begin{align}
        \label{inequality:ASGD-asyn-bias-2-2}
        &\ \frac{\alpha}{2\tau^2}(1+u)\mathbb{E}_{\xi_k}[\|\nabla f(W_k;\xi_k) \odot \sign(\tdnabla^k f)\odot W_k \|^2]
        \\
        % ========================
        \le&\ \frac{\alpha}{2\tau^2}(1+u)\mathbb{E}_{\xi_k}[\|\nabla f(W_k;\xi_k)\|^2]\ 
        % W_{\max,\infty}^2
        \|W_k \|^2_\infty
        \nonumber \\
        % ========================
        \le&\ \frac{\alpha}{2\tau^2}(1+u)\|\nabla f(W_k)\|^2 \|W_k \|^2_\infty
            +\frac{\alpha\sigma^2}{2\tau^2}(1+u)\|W_k \|^2_\infty
        \nonumber 
    \end{align}
    where the last inequality comes from Assumption \ref{assumption:noise}
    \begin{align}
        \label{inequality:grad-var-decomposion}
        \mbE_{\xi_k}[\|\nabla f(W_k;\xi_k)\|^2]
        =&\ \|\nabla f(W_k)\|^2+\mbE_{\xi_k}[\|\nabla f(W_k;\xi_k)-\nabla f(W_k)\|^2] 
        \\
        % ========================
        \le 
        &\ 
        \|\nabla f(W_k)\|^2+\sigma^2.
        \nonumber
    \end{align}
    Plugging \eqref{inequality:ASGD-asyn-bias-2}, \eqref{inequality:ASGD-asyn-bias-2-1} and \eqref{inequality:ASGD-asyn-bias-2-2} back into \eqref{inequality:ASGD-asyn-bias-1}, we have
    \begin{align}
        \label{inequality:ASGD-asyn-bias-3}
        &\ \frac{1}{2\alpha}\mathbb{E}_{\xi_k}[\|W_{k+1}-W_k+\alpha \nabla f(W_k;\xi_k)\|^2]
        \\
        % ========================
        =&\ \frac{\alpha}{2}(1+\frac{1}{u})(1+W_{\max,\infty}/\tau)^2
            C_{f}^2 M^2 \sum_{m'=1}^{M}\sum_{k'=k-(M-m')}^{k-1} \mathbb{E}_{\xi_k}[\|\Delta^\spp{m'}_{k'}\|^2]
        \nonumber\\
        &\ +\frac{\alpha}{2\tau^2}(1+u)\|\nabla f(W_k)\|^2 \|W_k \|^2_\infty
        +\frac{\alpha\sigma^2}{2\tau^2}(1+u)\|W_k \|^2_\infty.
        \nonumber
    \end{align}
    Substituting \eqref{inequality:ASGD-asyn-bias-3} into \eqref{inequality:ASGD-convergence-linear-2} yields
    \begin{align}
        \label{inequality:ASGD-convergence-linear-3}
        &\ 
        \mathbb{E}_{\xi_k}[f(W_{k+1})]
        \\
        % ========================
        \le&\  f(W_k) - \frac{\alpha}{2} \|\nabla f(W_k)\|^2 
        + \alpha^2L_f\sigma^2
        + \frac{\alpha}{2\tau^2}(1+u)\|\nabla f(W_k)\|^2 \|W_k \|^2_\infty
        + \frac{\alpha\sigma^2}{2\tau^2}(1+u)\|W_k \|^2_\infty
        \nonumber 
        \\
        % ========================
        &\ - \frac{1}{4\alpha}\mathbb{E}_{\xi_k}[\|W_{k+1}-W_k+\alpha(\nabla f(W_k;\xi_k)-\nabla f(W_k))\|^2]
        \nonumber \\
        &\ +\frac{\alpha}{2}(1+\frac{1}{u})(1+W_{\max,\infty}/\tau)^2
        C_{f}^2 M^2\sum_{m'=1}^{M}\sum_{k'=k-(M-m')}^{k-1} \mathbb{E}_{\xi_k}[\|\Delta^\spp{m'}_{k'}\|^2].
        \nonumber\\
        % ========================
        =&\  f(W_k) - \frac{\alpha}{2} \lp 1-(1+u)\frac{\|W_k \|^2_\infty}{\tau^2} \rp\|\nabla f(W_k)\|^2 
        + \alpha^2L_f\sigma^2
        + \frac{\alpha\sigma^2}{2\tau^2}(1+u)\|W_k \|^2_\infty
        \nonumber \\
        &\ - \frac{1}{4\alpha}\mathbb{E}_{\xi_k}[\|W_{k+1}-W_k+\alpha(\nabla f(W_k;\xi_k)-\nabla f(W_k))\|^2]
        \nonumber \\
        &\ +\frac{\alpha}{2}(1+\frac{1}{u})(1+W_{\max,\infty}/\tau)^2
        C_{f}^2M^2
        \mathbb{E}_{\xi_k}\Big[
        \underbrace{\sum_{m'=1}^{M}\sum_{k'=k-(M-m')}^{k-1} \|\Delta^\spp{m'}_{k'}\|^2
        }_{=: \psi_k}
        \Big].
        \nonumber
    \end{align}
    % Specifying the constant $u=\sqrt{\alpha L_f} \le 1$ and using $1+\frac{1}{u}=\frac{1}{u}(1+u)\le \frac{2}{u}$, we have
    % \begin{align}
    %     \label{inequality:ASGD-convergence-linear-3}
    %     &\  \mathbb{E}_{\varepsilon_k}[f(W_{k+1})]\\
    %     % ========================
    %     =&\  f(W_k) - \frac{\alpha}{2} \lp 1-(1+\sqrt{\alpha L_f})\frac{\|W_k \|^2_\infty}{\tau^2} \rp\|\nabla f(W_k)\|^2 
    %     + \alpha^2L_f\sigma^2
    %     + \frac{\alpha\sigma^2}{2\tau^2}(1+\sqrt{\alpha L_f})\|W_k \|^2_\infty
    %     \nonumber\\
    %     &\ +\frac{\sqrt{\alpha}(1+W_{\max,\infty}/\tau)^2C_{f}^2}{\sqrt{L_f}}
    %     \sum_{m'=1}^{M}\sum_{k'=k-(M-m')}^{k-1} \|\Delta^\spp{m'}_{k'}\|^2.
    %     \nonumber
    % \end{align}
    To deal with $\psi_k$ in the \ac{RHS} of \eqref{inequality:ASGD-convergence-linear-3}, we define an auxiliary function
    \begin{align}
        \Psi_k := \sum_{m'=1}^{M}\sum_{k'=k-(M-m')}^{k-1}
        (k'-k+M)\|\Delta^\spp{m'}_{k'}\|^2
    \end{align}
    which has the following recursion from $k$ to $k+1$
    \begin{align}
        \Psi_{k+1} = &\ \sum_{m'=1}^{M}\sum_{k'=k+1-(M-m')}^{k}
        (k'-(k+1)+M)\|\Delta^\spp{m'}_{k'}\|^2
        \\
        % ========================
        =&\ \sum_{m'=1}^{M}\sum_{k'=k+1-(M-m')}^{k}
        (k'-k+M)\|\Delta^\spp{m'}_{k'}\|^2
        -\sum_{m'=1}^{M}\sum_{k'=k+1-(M-m')}^{k}
        \|\Delta^\spp{m'}_{k'}\|^2
        \nonumber\\
        % ========================
        =&\ \sum_{m'=1}^{M}\sum_{k'=k-(M-m')}^{k-1}
        (k'-k+M)\|\Delta^\spp{m'}_{k'}\|^2
        -\sum_{m'=1}^{M}\sum_{k'=k+1-(M-m')}^{k}
        \|\Delta^\spp{m'}_{k'}\|^2
        \nonumber\\
        &\ + M\sum_{m'=1}^{M}\|\Delta^\spp{m'}_{k}\|^2
        - \sum_{m'=1}^{M}m' \|\Delta^\spp{m'}_{k-(M-m')}\|^2
        \nonumber\\
        % ========================
        \le&\ \Psi_k - \psi_k
        + M\|W_{k+1}-W_k\|^2
        % + M\sum_{m'=1}^{M}\|\Delta^\spp{m'}_{k}\|^2.
        \nonumber
    \end{align}
    where the last inequality holds by definition of $\Delta^\spp{m'}_{k}$, i.e., 
    \begin{align}
        \sum_{m'=1}^{M}\|\Delta^\spp{m'}_{k}\|^2
        =\sum_{m'=1}^{M}\|W^\spp{m'}_{k+1}-W^\spp{m'}_k\|^2
        =\|W_{k+1}-W_k\|^2.
    \end{align} 
    Construct a Lyapunov function by
    \begin{align}
        \label{lyapunov}
        \mbV_k := f(W_k) - f^* + \frac{\alpha}{2}(1+\frac{1}{u})(1+W_{\max,\infty}/\tau)^2C_{f}^2M^2 \Psi_k.
    \end{align}
    According to \eqref{inequality:ASGD-convergence-linear-3}, we have
    \begin{align}
        \label{inequality:ASGD-convergence-linear-4}
        \mbE_{\xi_k}[\mbV_{k+1}] 
        =&\ f(W_{k+1})-f^*
        + \frac{\alpha}{2}(1+\frac{1}{u})(1+W_{\max,\infty}/\tau)^2C_{f}^2M^2 \Psi_{k+1}
        % + \frac{\alpha C_{f}M^2}{2}\sum_{m'=1}^{M}\sum_{k'=k+1-(M-m'+1)}^{k}
        % (k'-{k+1}+M+1)\mathbb{E}_{\xi_k}[\|\Delta^\spp{m'-1}_{k'}\|^2]
        \\
        % ========================
        \le&\ f(W_k) - \frac{\alpha}{2} \lp 1-(1+u)\frac{\|W_k \|^2_\infty}{\tau^2} \rp\|\nabla f(W_k)\|^2 
        + \alpha^2L_f\sigma^2
        + \frac{\alpha\sigma^2}{2\tau^2}(1+u)\|W_k \|^2_\infty
        \nonumber\\
        &\ - \frac{1}{4\alpha}\mathbb{E}_{\xi_k}[\|W_{k+1}-W_k+\alpha(\nabla f(W_k;\xi_k)-\nabla f(W_k))\|^2]
        \nonumber \\
        % &\ - \frac{\alpha}{2}(1+\frac{1}{u})(1+W_{\max,\infty}/\tau)^2C_{f}^2M^2\sum_{m'=1}^{M}\sum_{k'=k-(M-m'+1)}^{k}
        % \mathbb{E}_{\xi_k}[\|\Delta^\spp{m'}_{k'}\|^2]
        % \nonumber\\
        &\ + \frac{\alpha}{2}(1+\frac{1}{u})(1+W_{\max,\infty}/\tau)^2C_{f}^2M^3\mathbb{E}_{\xi_k}[\|W_{k+1}-W_k\|^2].
        % - + \frac{\alpha C_{f}}{2}\sum_{m'=1}^{M} m'\|\Delta^\spp{m'-1}_{k'}\|
        \nonumber
    \end{align}
    % By the definition of $\Delta^\spp{m}_{k}$ and the dynamics of Analog SGD,
    % $\Delta^\spp{m}_{k} = W^\spp{m}_{k+1}-W^\spp{m}_k$,
    We bound the last term of the \ac{RHS} of \eqref{inequality:ASGD-convergence-linear-4} by
    \begin{align}
        &\ 
        \|W_{k+1}-W_k\|^2
        \\
        % ========================
        \le&\ 2\|W_{k+1}-W_k+\alpha (\nabla f(W_k;\xi_k)-\nabla f(W_k))\|^2 + 2\alpha^2\|\nabla f(W_k;\xi_k)-\nabla f(W_k)\|^2
        \nonumber
        \\
        % ========================
        \le&\ 2\|W_{k+1}-W_k+\alpha (\nabla f(W_k;\xi_k)-\nabla f(W_k))\|^2 +2\alpha^2\sigma^2.
        \nonumber
    \end{align}
    % By Lemma \ref{lemma:bounded-gradient} and inequality $\|\bbone -  W_k^\spp{m}/\tau\|^2_\infty \le (1+W_{\max,\infty}/\tau)^2$, it holds that
    % \begin{align}
    %     \label{inequality:ASGD-convergence-linear-4-1}
    %     \sum_{m'=1}^{M}\|\Delta^\spp{m'}_{k'}\|^2 \le \alpha^2 (1+W_{\max,\infty}/\tau)^2 \nabla f_{\max}^2.
    % \end{align}
    Organizing the terms of \eqref{inequality:ASGD-convergence-linear-4},
    % and combining \eqref{inequality:ASGD-convergence-linear-4-1}, 
    we obtain
    \begin{align}
        \label{inequality:ASGD-convergence-linear-5}
        &\ \frac{\alpha}{2} \lp 1-(1+u)\frac{\|W_k \|^2_\infty}{\tau^2} \rp\|\nabla f(W_k)\|^2 
        \\
        % ========================
        % \le&\ \mbV_k - \mbE_{\xi_{k+1}}[\mbV_{k+1}] 
        % + \alpha^2L_f\sigma^2
        % + \frac{\alpha\sigma^2}{2\tau^2}(1+\sqrt{\alpha L_f})\|W_k \|^2_\infty
        % + \frac{\alpha}{2}(1+\frac{1}{u})(1+W_{\max,\infty}/\tau)^2C_{f}^2\sum_{m'=1}^{M}\|\Delta^\spp{m'-1}_{k'}\|
        % \nonumber\\
        % % ========================
        \le &\ \mbV_k - \mbE_{\xi_k}[\mbV_{k+1}] 
        + \alpha^2L_f\sigma^2
        + \frac{\alpha\sigma^2}{2\tau^2}(1+u)\|W_k \|^2_\infty
        \nonumber\\
        &\ - \lp\frac{1}{4\alpha}-\alpha(1+\frac{1}{u})(1+W_{\max,\infty}/\tau)^2C_{f}^2M^2\rp\mathbb{E}_{\xi_k}[\|W_{k+1}-W_k+\alpha(\nabla f(W_k;\xi_k)-\nabla f(W_k))\|^2]
        \nonumber \\
        &\ +\alpha^3(1+\frac{1}{u})(1+W_{\max,\infty}/\tau)^2C_{f}^2 M^3 \sigma^2.
        \nonumber
    \end{align}
    Taking expectation, averaging \eqref{inequality:ASGD-convergence-linear-5} over $k$ from $0$ to $K-1$, and choosing the stepsize  $\alpha\leq\frac{1}{2\sqrt{1+1/u}(1+W_{\max,\infty}/\tau)C_f M}$ yield  
    \begin{align}
        % \label{inequality:ASGD-convergence-linear-6}
        &\ 
        \frac{1}{K}\sum_{k=0}^{K-1}\mbE\lB\|\nabla f(W_k)\|^2\rB
        \\
        % ========================
        % \le&\ \mbV_k - \mbE_{\xi_{k+1}}[\mbV_{k+1}] 
        % + \alpha^2L_f\sigma^2
        % + \frac{\alpha\sigma^2}{2\tau^2}(1+\sqrt{\alpha L_f})\|W_k \|^2_\infty
        % + \frac{\alpha}{2}(1+\frac{1}{u})(1+W_{\max,\infty}/\tau)^2C_{f}^2\sum_{m'=1}^{M}\|\Delta^\spp{m'-1}_{k'}\|
        % \nonumber\\
        % % ========================
        \le &\ 
        \frac{2(\mbV_0 - \mbE[\mbV_{k+1}])}{\alpha K(1-(1+u)W_{\max,\infty}^2/\tau^2)}
        +\frac{2\alpha L_f\sigma^2}{1-(1+u)W_{\max,\infty}^2/\tau^2}
        \bkeq
        + \sigma^2 \frac{1}{K}\sum_{k=0}^{K-1}
        \frac{(1+u)\|W_k\|^2_\infty/\tau}{1-(1+u)\|W_k \|^2_\infty/\tau^2}
        +\frac{\alpha^2(1+1/u)(1+W_{\max,\infty}/\tau)^4C_{f}^2M^3 \sigma^2}{(1-(1+u)W_{\max,\infty}^2/\tau^2)}.
        \nonumber
    \end{align}
    % By the fact that
    % % By Taylor expansion,
    % \begin{align}
    %     \frac{(1+u)\|W_k \|^2_\infty/\tau}{1-(1+u)\|W_k \|^2_\infty/\tau^2}
    %     = \frac{\|W_k \|^2_\infty/\tau}{1-\|W_k \|^2_\infty/\tau^2}
    %     + \frac{u\|W_k \|^2_\infty/\tau}{(1-\|W_k \|^2_\infty/\tau^2)^2}
    %     + o(u).
    % \end{align}
    Using the fact that  $V_{k}\ge 0$ for any $k\in\naturals$ and $V_0=f(W_0)-f^*$, and choosing the learning rate $\alpha=\sqrt{\frac{f(W_0) - f^*}{\sigma^2L_fK}}$, we have
    \begin{align}
        \label{inequality:ASGD-convergence-linear-6}
        &\ 
        \frac{1}{K}\sum_{k=0}^{K-1}\mbE\lB\|\nabla f(W_k)\|^2\rB
        \\
        % ========================
        % \le&\ \mbV_k - \mbE_{\xi_{k+1}}[\mbV_{k+1}] 
        % + \alpha^2L_f\sigma^2
        % + \frac{\alpha\sigma^2}{2\tau^2}(1+\sqrt{\alpha L_f})\|W_k \|^2_\infty
        % + \frac{\alpha}{2}(1+\frac{1}{u})(1+W_{\max,\infty}/\tau)^2C_{f}^2\sum_{m'=1}^{M}\|\Delta^\spp{m'-1}_{k'}\|
        % \nonumber\\
        % % ========================
        \le &\ 
        4\sqrt{\frac{(f(W_0) -f^*)\sigma^2L_f}{K}}
        \frac{1}{1-(1+u)W_{\max,\infty}^2/\tau^2}
        + \sigma^2 S_K'
        % + \sigma^2 \frac{1}{K}\sum_{k=0}^{K-1}        \frac{(1+u)\|W_k\|^2_\infty/\tau}{1-(1+u)\|W_k \|^2_\infty/\tau^2}
        +\ccalO\lp \frac{1+1/u}{K}\rp.
        \nonumber
    \end{align}
    where $S_K'$ denotes the amplification factor given by
    \begin{align}
        S_K' := \frac{1}{K}\sum_{k=0}^{K}\frac{(1+u)\|W_k\|_\infty^2/\tau^2}{1-(1+u)\|W_k\|_\infty^2/\tau^2} \le \frac{(1+u)W_{\max}^2/\tau^2}{1-(1+u)W_{\max}^2/\tau^2}.
    \end{align}
    The proof is completed now.
\end{proof}

\section{Convergence of Digital Asynchronous Pipeline SGD}\label{sec:digital-pipeline-SGD}
\begin{restatable}[Iteration complexity, asynchronous pipeline]{theorem}{ThmDSGDAsyncConvergenceNoncvxLinear}
  \label{theorem:SGD-async-convergence-noncvx-linear}
  Under Assumption \ref{assumption:noise}--\ref{assumption:W-bounded}, if the learning rate is set as $\alpha=\sqrt{\frac{f(W_0) - f^*}{\sigma^2L_fK}}$ and $K$ is sufficiently large such that $\alpha\leq\sqrt{\frac{1}{4C_f M^3}}$,
      % \begin{align}
      %     \alpha = \sqrt{\frac{f(W_0) - f(W_k)}{\sigma^2LK}}
      % \end{align}
  % $\alpha=\frac{\tau}{L(\|W^*\|_\infty+\tau)}$, 
  it holds that 
  % Assume it holds that 
  % \begin{align}
  % 	\la |\nabla f(W_k)| \odot W_k-|\nabla f(W^\infty)| \odot W^\infty, W_k-W^\infty\ra 
  % 	\ge&\ -\mu'\|W_k-W^\infty\|^2, \\
  % 	\la |\nabla f(W_k)| \odot W_k-|\nabla f(W^\infty)| \odot W^\infty, \nabla f(W_k)-\nabla f(W^\infty)\ra 
  % 	\le&\ L\|W_k-W^\infty\|^2.
  % \end{align}
  % The optimal distance $\|W_k-W^\infty\|^2$ has the upper-bound that
  % \begin{align}
  % 	\|W_k-W^\infty\|^2 \le (1-\mu+\mu')_k \|W_0-W^\infty\|^2.
  % \end{align}
  % \begin{align}
  %   \label{inequality:ASGD-convergence-noncvx-linear}
  %   \frac{1}{K}\sum_{k=0}^{K-1}\mbE[\|\nabla f(W_k)\|^2]
  %       % E_K
  %   % \min_{k\in [K]}\mbE[\|\nabla f(W_k)\|^2]
  %   \le&\ 4\sqrt{\frac{
  %           (f(W_0) - f^*)
  %           % F_0
  %           \sigma^2L}{K}}\frac{1}{1-W_{\max}^2/\tau^2}
  %       +4\sigma^2S_K
  %     % +\frac{4\sigma^2W_{\max}^2/\tau^2}{1-W_{\max}^2/\tau^2}.
  % \end{align}
    \begin{align}
        % \label{inequality:ASGD-convergence-linear-6}
        \ \frac{1}{K}\sum_{k=0}^{K-1}\mbE\lB\|\nabla f(W_k)\|^2\rB
        % ========================
        % \le&\ \mbV_k - \mbE_{\xi_{k+1}}[\mbV_{k+1}] 
        % + \alpha^2L_f\sigma^2
        % + \frac{\alpha\sigma^2}{2\tau^2}(1+\sqrt{\alpha L_f})\|W_k \|^2_\infty
        % + \frac{\alpha}{2}(1+\frac{1}{u})(1+W_{\max,\infty}/\tau)^2C_{f}^2\sum_{m'=1}^{M}\|\Delta^\spp{m'-1}_{k'}\|
        % \nonumber\\
        % % ========================
        \le \ 4\sqrt{\frac{(f(W_0)-f^*)L_f\sigma^2}{K}}
        + \frac{2C_f M^3 (f(W_0)-f^*)}{K L_f }. 
    \end{align}
\end{restatable} 

\begin{proof}\allowdisplaybreaks
    % \textbf{Digital SGD.}
    The validation of inequality \eqref{inequality:ASGD-convergence-linear-1} does not rely on any training algorithm, and hence it still holds here. To bound the last term in the \ac{RHS} of \eqref{inequality:ASGD-convergence-linear-1}, we have
    \begin{align}
        \label{inequality:ASGD-convergence-linear-digital-SGD}
        &\ \frac{1}{2\alpha}\mathbb{E}_{\xi_k}[\|W_{k+1}-W_k+\alpha \nabla f(W_k;\xi_k)\|^2]\nonumber
        = \frac{\alpha}{2}\mathbb{E}_{\xi_k}[\|\tdnabla^k f - \nabla f(W_k;\xi_k)\|^2]
        \\
        \le&\ \frac{\alpha C_{f}}{2}\mathbb{E}_{\xi_k}\left[\sum_{m'=1}^{M}\sum_{k'=k-(M-m'+1)}^{k-1} \|\Delta^\spp{m'}_{k'}\|\right]^2
        % \nonumber\\
        % ========================
        \le \frac{\alpha C_{f}M^2}{2}\sum_{m'=1}^{M}\sum_{k'=k-(M-m'+1)}^{k-1} \mathbb{E}_{\xi_k}[\|\Delta^\spp{m'}_{k'}\|^2]
    \end{align}
where the first inequality is according to Lemma \ref{lemma:gradient-delay}. 

According to the descent lemma in \eqref{inequality:ASGD-convergence-linear-2}, we have 
\begin{align}
    \label{eq:f-descent}
    &\ \mathbb{E}_{\xi_k}[f(W_{k+1})]
    \\
    \le&\ f(W_k) - \frac{\alpha}{2} \|\nabla f(W_k)\|^2 
    +\alpha^2L_f\sigma^2
    + \frac{1}{2\alpha}\mathbb{E}_{\xi_k}[\|W_{k+1}-W_k+\alpha \nabla f(W_k;\xi_k)\|^2]
    \nonumber \\
    % ========================
    &\ - \frac{1}{4\alpha}\mathbb{E}_{\xi_k}[\|W_{k+1}-W_k+\alpha(\nabla f(W_k;\xi_k)-\nabla f(W_k))\|^2]
    \nonumber 
    \\
    \le&\ f(W_k) - \frac{\alpha}{2} \|\nabla f(W_k)\|^2 
    +\alpha^2L_f\sigma^2
    + \frac{\alpha C_{f}M^2}{2}\sum_{m'=1}^{M}\sum_{k'=k-(M-m'+1)}^{k-1} \mathbb{E}_{\xi_k}[\|\Delta^\spp{m'}_{k'}\|^2]
    \nonumber \\
    % ========================
    &\ - \frac{1}{4\alpha}\mathbb{E}_{\xi_k}[\|W_{k+1}-W_k+\alpha(\nabla f(W_k;\xi_k)-\nabla f(W_k))\|^2]
    \nonumber
\end{align}
 where the second inequality is earned by plugging in \eqref{inequality:ASGD-convergence-linear-digital-SGD}. 

    % where the last inequality comes from Assumption \ref{assumption:noise}
    % \begin{align}
    %     \label{inequality:grad-var-decomposion}
    %     \mbE_{\varepsilon_{k+1}}[\|\nabla f(W_{k+1})+\varepsilon_{k+1}\|^2]
    %     =&\ \|\nabla f(W_{k+1})\|^2+\mbE_{\varepsilon_{k+1}}[\|\varepsilon_{k+1}\|^2] \\
    %     % ========================
    %     \le&\ \|\nabla f(W_{k+1})\|^2+\sigma^2.
    % \nonumber
    % \end{align}
    Construct a Lyapunov function as 
    \begin{align}
        \mbV_k := f(W_k) 
        + \frac{\alpha C_{f}M^2}{2}\sum_{m'=1}^{M}\sum_{k'=k-(M-m'+1)}^{k-1}
        (k'-k+M-m'+2)\|\Delta^\spp{m'}_{k'}\|^2.
    \end{align}
    According to \eqref{eq:f-descent}, we have
    \begin{align}
        &\ 
        \mbE_{\xi_{k}}[\mbV_{k+1}] -\mbV_{k}
        \nonumber\\
        % ========================
        % ========================
        \le&\ - \frac{\alpha}{2} \|\nabla f(W_k)\|^2 
        +\alpha^2L_f\sigma^2
        + \frac{\alpha C_{f}M^2}{2}\sum_{m'=1}^{M}\sum_{k'=k-(M-m'+1)}^{k-1} \mbE_{\xi_{k}}[\|\Delta^\spp{m'}_{k'}\|^2]
        \nonumber
        \\
        &\ - \frac{1}{4\alpha}\mathbb{E}_{\xi_k}[\|W_{k+1}-W_k+\alpha(\nabla f(W_k;\xi_k)-\nabla f(W_k))\|^2]\nonumber\\
        &\ + \frac{\alpha C_{f}M^2}{2} \sum_{m'=1}^{M}\sum_{k'=k-(M-m'+1)}^{k}
        (k'-k+M-m'+1)\mbE_{\xi_{k}}[\|\Delta^\spp{m'}_{k'}\|^2]\nonumber
        \\
        &\ -\frac{\alpha C_{f}M^2}{2}\sum_{m'=1}^{M}\sum_{k'=k-(M-m'+1)}^{k-1}
        (k'-k+M-m'+2)\mathbb{E}_{\xi_k}[\|\Delta^\spp{m'}_{k'}\|^2]\nonumber\\
        \le&\ - \frac{\alpha}{2} \|\nabla f(W_k)\|^2 
        +\alpha^2L_f\sigma^2
        + \frac{\alpha C_{f}M^2}{2}\sum_{m'=1}^{M}\sum_{k'=k-(M-m'+1)}^{k-1} \mbE_{\xi_{k}}[\|\Delta^\spp{m'}_{k'}\|^2]
        \nonumber\\
        &\ - \frac{1}{4\alpha}\mathbb{E}_{\xi_k}[\|W_{k+1}-W_k+\alpha(\nabla f(W_k;\xi_k)-\nabla f(W_k))\|^2]\nonumber\\
        &\ -\frac{\alpha C_{f}M^2}{2}\sum_{m'=1}^{M}\sum_{k'=k-(M-m'+1)}^{k-1}
        \mbE_{\xi_{k}}[\|\Delta^\spp{m'}_{k'}\|^2]
        + \frac{\alpha C_{f}M^2}{2}(M-m^\prime+1)\sum_{m'=1}^{M}\mbE_{\xi_{k}}[\|\Delta^\spp{m'}_{k}\|^2].
        % - + \frac{\alpha C_{f}}{2}\sum_{m'=1}^{M} m'\|\Delta^\spp{m'}_{k'}\|^2
        \nonumber\\
        \le&\ - \frac{\alpha}{2} \|\nabla f(W_k)\|^2 
        +\alpha^2L_f\sigma^2
        + \frac{\alpha C_{f}M^2}{2}\sum_{m'=1}^{M}\sum_{k'=k-(M-m'+1)}^{k-1} \mbE_{\xi_{k}}[\|\Delta^\spp{m'}_{k'}\|^2]
        \nonumber\\
        &\ - \frac{1}{4\alpha}\mathbb{E}_{\xi_k}[\|W_{k+1}-W_k+\alpha(\nabla f(W_k;\xi_k)-\nabla f(W_k))\|^2]\nonumber\\
        &\ -\frac{\alpha C_{f}M^2}{2}\sum_{m'=1}^{M}\sum_{k'=k-(M-m'+1)}^{k-1}
        \mbE_{\xi_{k}}[\|\Delta^\spp{m'}_{k'}\|^2] + \frac{\alpha C_{f}M^3}{2}\sum_{m'=1}^{M}\mbE_{\xi_{k}}[\|\Delta^\spp{m'}_{k}\|^2]
        \nonumber\\
        \le&\ - \frac{\alpha}{2} \|\nabla f(W_k)\|^2 
        +\alpha^2L_f\sigma^2
        + \frac{\alpha C_{f}M^2}{2}\sum_{m'=1}^{M}\sum_{k'=k-(M-m'+1)}^{k-1} \mbE_{\xi_{k}}[\|\Delta^\spp{m'}_{k'}\|^2]
        \nonumber\\
        &\ -\frac{\alpha C_{f}M^2}{2}\sum_{m'=1}^{M}\sum_{k'=k-(M-m'+1)}^{k-1}
        \mbE_{\xi_{k}}[\|\Delta^\spp{m'}_{k'}\|^2]+\alpha^3 C_f M^3 \sigma^2\nonumber\\
        &\ -\left(\frac{1}{4\alpha}- \alpha C_{f}M^3\right)\mbE_{\xi_{k}}[\|W_{k+1}-W_{k}+\alpha(\nabla f(W_k;\xi_k)-\nabla f(W_k))\|^2]
        \nonumber\\
        \le&\ - \frac{\alpha}{2} \|\nabla f(W_k)\|^2 
        +\alpha^2L_f\sigma^2+\alpha^3 C_f M^3 \sigma^2\nonumber
    \end{align}
where the fourth inequality holds because 
\begin{align}
    \label{eq:important1}
    \sum_{m'=1}^{M}\mbE_{\xi_{k}}[\|\Delta^\spp{m'}_{k}\|^2]
    &= 
    \sum_{m'=1}^{M}\mbE_{\xi_{k}}[\|W_{k+1}^{(m^\prime)}-W_{k}^{(m^\prime)}\|^2]
    = \mbE_{\xi_{k}}[\|W_{k+1}-W_{k}\|^2]
    \\
    &\leq 2\mbE_{\xi_{k}}[\|W_{k+1}-W_{k}+\alpha(\nabla f(W_k;\xi_k)-\nabla f(W_k))\|^2]\nonumber\\
    &~~~~+2\alpha^2\mbE_{\xi_{k}}[\|\nabla f(W_k;\xi_k)-\nabla f(W_k)\|^2]\nonumber\\
    &\leq 2\mbE_{\xi_{k}}[\|W_{k+1}-W_{k}+\alpha(\nabla f(W_k;\xi_k)-\nabla f(W_k))\|^2]+2\alpha^2\sigma^2. 
    \nonumber
\end{align}
and the last inequality is earned by letting  $\alpha\leq\sqrt{\frac{1}{4C_f M^3}}$. Reorganizing the terms yields
    \begin{align}
        \frac{\alpha}{2} \|\nabla f(W_k)\|^2 
        % ========================
        \le &\ \mbV_k - \mbE_{\xi_{k}}[\mbV_{k+1}] 
        +\alpha^2L_f\sigma^2+\alpha^3 C_f M^3 \sigma^2. 
        \nonumber
    \end{align}
Therefore, by telescoping and taking the expectation, we get 
\begin{align}
    \frac{1}{K}\sum_{k=0}^{K-1}\mbE[\|\nabla f(W_k)\|^2] 
    \le&\ \frac{2(\mbV_0 - \mbE[\mbV_{K}])}{\alpha K}
    +2\alpha L_f\sigma^2+2\alpha^2 C_f M^3 \sigma^2
    % \nonumber 
    \\
    % ========================
    \le \frac{2(f(W_0)-f^*)}{\alpha K}
    +2\alpha L_f\sigma^2
    + 2\alpha^2 C_{f} M^3\sigma^2.  
    \nonumber
\end{align}
Letting $\alpha\leq\sqrt{\frac{f(W_0)-f^*}{L_f\sigma^2 K}}$, we obtain that    
    \begin{align}
        \frac{1}{K}\sum_{k=0}^{K-1}\mbE[\|\nabla f(W_k)\|^2]
        \le &\ 4\sqrt{\frac{(f(W_0)-f^*)L_f\sigma^2}{K}}
        + \frac{2C_f M^3 (f(W_0)-f^*)}{K L_f }.
        \nonumber
    \end{align}

\end{proof}
}

\end{document}